\documentclass{article}
\PassOptionsToPackage{numbers,sort&compress}{natbib}
\usepackage[preprint]{neurips_2026}

\usepackage{needspace}
\usepackage{wrapfig}
\usepackage[utf8]{inputenc}
\usepackage[T1]{fontenc}
\usepackage{microtype}

\usepackage{comment}
\usepackage{xparse}

\usepackage[dvipsnames]{xcolor}

\definecolor{paleBlueA}{RGB}{15,132,246}
\definecolor{paleBlueB}{RGB}{5,125,245}
\definecolor{paleBlueC}{RGB}{220,240,255}
\definecolor{pillarblue}{RGB}{20,40,120}
\definecolor{softblue}{RGB}{60,120,180}
\definecolor{royalblue}{RGB}{65,105,225}
\definecolor{brickred}{RGB}{178,34,34}
\definecolor{forestgreen}{RGB}{34,139,34}
\definecolor{royalpurple}{RGB}{128,0,128}
\definecolor{slate}{RGB}{112,128,144}

\usepackage{graphicx}
\usepackage{subcaption}
\usepackage{booktabs}
\usepackage[inline]{enumitem}

\usepackage{amsmath,amssymb,amsfonts,mathtools,amsthm}
\usepackage{thmtools}
\usepackage{upgreek}
\allowdisplaybreaks

\usepackage{algorithm}
\usepackage{algpseudocode}
\usepackage{float}

\usepackage{caption}

\usepackage[most]{tcolorbox}
\usepackage[textsize=tiny]{todonotes}
\tcbuselibrary{theorems,breakable,skins}

\usepackage{tikz}
\usetikzlibrary{arrows.meta, positioning, decorations.pathmorphing, shadows.blur, calc, backgrounds}

\usepackage[numbers,sort&compress]{natbib}

\usepackage[
  colorlinks=true,
  linkcolor=paleBlueA,
  citecolor=paleBlueB,
  urlcolor=blue,
  breaklinks=true
]{hyperref}

\hypersetup{
  linktocpage=true,
  pdfstartview=FitH,
  pdfpagemode=UseOutlines,
  pageanchor=true,
  plainpages=false,
  bookmarksnumbered=true,
  bookmarksopen=false,
  bookmarksopenlevel=1,
  hypertexnames=false,
  pdfhighlight=/O
}

\usepackage[capitalize,noabbrev]{cleveref}

\theoremstyle{plain}
\newtheorem{theorem}{Theorem}[section]

\newtheorem{lemma}{Lemma}[section]
\newtheorem{corollary}{Corollary}[section]

\theoremstyle{definition}
\newtheorem{definition}{Definition}[section]
\newtheorem{assumption}{Assumption}[section]

\theoremstyle{remark}
\newtheorem{remark}{Remark}[section]

\newtcolorbox{yellownote}{
  colback=yellow!12,
  colframe=gray!50,
  boxrule=0.6pt,
  arc=1pt,
  left=6pt,
  right=6pt,
  top=6pt,
  bottom=6pt
}

\definecolor{PBBlueBg}{RGB}{232,246,255}
\definecolor{PBBlueFrame}{RGB}{70,150,200}

\newtcbtheorem[number within=section]{PBAppThmBox}{Theorem}{
  enhanced,
  breakable,
  colback=PBBlueBg,
  colframe=PBBlueFrame,
  boxrule=0.8pt,
  arc=2mm,
  left=1.2mm,
  right=1.2mm,
  top=1mm,
  bottom=1mm,
  fonttitle=\bfseries,
  coltitle=black,
  attach title to upper
}{pbthm}

\newtcbtheorem[number within=section]{PBAppLemmaBox}{Lemma}{
  enhanced,
  breakable,
  colback=PBBlueBg,
  colframe=PBBlueFrame,
  boxrule=0.8pt,
  arc=2mm,
  left=1.2mm,
  right=1.2mm,
  top=1mm,
  bottom=1mm,
  fonttitle=\bfseries,
  coltitle=black,
  attach title to upper
}{pblem}

\newtcbtheorem[number within=section]{PBAppCorollaryBox}{Corollary}{
  enhanced,
  breakable,
  colback=PBBlueBg,
  colframe=PBBlueFrame,
  boxrule=0.8pt,
  arc=2mm,
  left=1.2mm,
  right=1.2mm,
  top=1mm,
  bottom=1mm,
  fonttitle=\bfseries,
  coltitle=black,
  attach title to upper
}{pbcor}
\ExplSyntaxOn

\NewDocumentEnvironment{PBAppTheorem}{ o m m +b }
 {
  \par\medskip
  \IfNoValueTF{#1}
    {\subsection{Proof\ of\ Theorem~\ref{pbthm:#2}}}
    {\subsection{Proof\ of\ Theorem~\ref{pbthm:#2} (\texorpdfstring{#1}{#1})}}
  \label{proof:#2}
  \par\medskip
  \noindent\textit{\small Intuition. #3}\par\smallskip
  \IfNoValueTF{#1}
    {\def\PBApp@opt{}}
    {\def\PBApp@opt{#1}}
  \begin{PBAppThmBox}{\PBApp@opt . \phantom{1}}{#2}
  #4
  \end{PBAppThmBox}
 }
 {}

\NewDocumentEnvironment{PBAppLemma}{ o m m +b }
 {
  \par\medskip
  \IfNoValueTF{#1}
    {\subsection{Proof\ of\ Lemma~\ref{pblem:#2}}}
    {\subsection{Proof\ of\ Lemma~\ref{pblem:#2}\  (\texorpdfstring{#1}{#1})}}
  \label{proof-lem:#2}
  \par\medskip
  \noindent\textit{\small Intuition. #3}\par\smallskip
  \IfNoValueTF{#1}
    {\def\PBApp@opt{}}
    {\def\PBApp@opt{#1}}
  \begin{PBAppLemmaBox}{\PBApp@opt . \phantom{1}}{#2}
  #4
  \end{PBAppLemmaBox}
 }
 {}

\ExplSyntaxOff

\newcommand{\authorcell}[3]{%
  \begin{minipage}[t]{0.30\textwidth}\centering
  {\normalsize\bfseries #1\par}
  \vspace{0.2em}
  {\small #2\par}
  \if\relax\detokenize{#3}\relax\else
    \vspace{0.2em}
    {\small\href{mailto:#3}{\texttt{#3}}\par}
  \fi
  \end{minipage}
}

\usepackage[framemethod=tikz]{mdframed}
\newmdenv[
  linewidth=0.7pt,
  roundcorner=2pt,
  innerleftmargin=6pt,
  innerrightmargin=6pt,
  innertopmargin=6pt,
  innerbottommargin=6pt,
  skipabove=6pt,
  skipbelow=6pt
]{theorembox}
\newmdenv[
  linewidth=0.5pt,
  roundcorner=1pt,
  innerleftmargin=2pt,
  innerrightmargin=2pt,
  innertopmargin=0pt,
  innerbottommargin=2pt,
  skipabove=0pt,
  skipbelow=0pt
]{algobox}

\usepackage{tabularx}
\usepackage{xspace}
\usepackage{array}
\newcolumntype{Y}{>{\raggedright\arraybackslash}X}

\newcommand{\PB}{\textsc{Prudent-Banker}}
\newcommand{\Banker}{\textsc{Banker-OMD}}
\newcommand{\CUCB}{\textsc{Conservative-UCB}}
\newcommand{\SafeEXP}{\textsc{Safe-EXP3-IX}}

\title{Prudent-Banker: No Extra Fees for Baseline Safety in Adversarial Bandits With and Without Delays}
\author{%
  Ting Hu\thanks{Equal contribution.} \\
  Department of Economics \\
  University of Wisconsin--Madison \\
  \texttt{ting.hu@wisc.edu} 
  \and
  {\bf Luanda Cai\footnotemark[1]} \\
  Department of Finance \\
  University of Wisconsin--Madison \\
  \texttt{luanda.cai@wisc.edu} 
  \And
  Emmanouil-Vasileios Vlatakis-Gkaragkounis \\
  Department of Computer Sciences \\
  University of Wisconsin--Madison \\
  \texttt{vlatakis@wisc.edu} \\
}

\begin{document}
\maketitle

\begin{abstract}
We study adversarial multi-armed bandits with and without delayed feedback under a \emph{safety-aware} goal: achieving minimax-optimal worst-case
regret while keeping nearly \emph{constant} regret relative to a designated ``safe'' baseline policy. Existing approaches can balance this trade-off with immediate feedback for smooth comparators, but arbitrary delays can mistime transitions between \emph{conservatism} and \emph{exploration}, endangering the safety guarantee.
To bridge this gap, we propose \textsc{\small Prudent-Banker}, a novel algorithm that combines a delay-adapted variant of Online Mirror Descent with a modified phased-aggression mechanism. Its key technical contribution is a delay-calibrated restart threshold that rigorously accounts for the worst-case distortion induced by unobserved feedback  and reliably detects comparator suboptimality. We also establish new lower bounds for safety-constrained adversarial delayed bandits, showing that the regret guarantees of \textsc{\small Prudent-Banker} are unimprovable, up to logarithmic factors, under the baseline-safety requirement.

To the best of our knowledge, \textsc{\small Prudent-Banker} is the first
algorithm to achieve the optimal safety--robustness trade-off: pseudo-regret
$\widetilde{O}(\sqrt{T}+\sqrt{D})$ together with $\widetilde{O}(1)$ regret
against the safe comparator, both with and without delays. Experiments across
diverse delay distributions show that, unlike standard delay-robust baselines,
\textsc{\small Prudent-Banker} effectively balances safety and learning.
\end{abstract}

\section{Introduction}
Sequential decision-making is shaped by a recurring dilemma: \emph{explore} to acquire new information, or \emph{exploit} what is already known to perform well. The multi-armed bandit (MAB) framework, introduced by \citet{robbins1952some}, has been the canonical mathematical model for this trade-off and has become a foundational tool across Operations Research, Economics, and Computer Science.

In this setting, a learner repeatedly chooses one action, an ``arm'',  from a fixed set. After choosing its arm, the learner incurs a loss (or receives a reward) that depends on the environment; The goal is to learn from this feedback over time so as to minimize \emph{regret}, i.e., the gap between the learner’s cumulative loss and the cumulative loss of the best single arm in hindsight.

However, classical MAB theory often assumes immediate feedback, but modern systems routinely violate this premise. In web applications, latency and user inactivity delay feedback loops \citep{parvin2020just}; in video streaming, QoE signals (e.g., buffering and engagement) are observed only after delivery \citep{changuel2012online}; in robotics, actuation and sensing lag behind commands \citep{mahmood2018setting}; and in multi-agent settings, communication is inherently asynchronous \citep{chen2020delay}. In high-stakes domains such as online advertising \citep{LiChuLangfordSchapire2010} and clinical decision-making \citep{WasonTrippa2014}, delays are intrinsic: transactions or patient outcomes may appear long after an action is taken. These realities motivate bandit algorithms whose performance guarantees remain meaningful even when feedback is substantially delayed.

While a growing body of work addresses adversarial bandits with delayed feedback \cite{zimmert2020optimal,gyorgy2021adapting,howson2021delayed,jin2022near,dai2022follow,huang2023banker}, the majority of existing guarantees focus on a single metric: \emph{minimax regret minimization}. While this worst-case lens is theoretically robust, it is insufficient for high-stakes deployments. In practice, a learner rarely enters an environment \emph{tabula rasa}; they are often equipped with a vetted baseline policy—such as a standard medical treatment, a compliant control rule, or a conservative investment strategy—that is trusted to avoid catastrophic outcomes. 
In light of this discussion, we arrive at the main question driving this work:
\vspace{-0.5em}
\begin{center}
    \begin{minipage}[c]{0.85\linewidth}
        \centering
        \textbf{\textit{Can we be worst-case robust while paying essentially no extra cost over a trusted baseline, even with arbitrary feedback delays?}}
    \end{minipage}%
    \hspace{0.5em} 
    \begin{minipage}[c]{0.05\linewidth}
        $(\star)$
    \end{minipage}
\end{center}\vspace{-0.5em}
In such settings, delayed feedback significantly amplifies the risk of experimentation. If an exploratory action is harmful, the evidence of failure may only surface after a long lag. During this blind spot, the learner incurs a ``hidden debt", repeatedly committing to a suboptimal strategy before the feedback arrives to trigger a correction. This danger motivates exactly the following dual mandate:
\begin{center}
    \vspace{-0.75em}
    \textbf{\small Safe Online Learning in Delays (SOLID) }
    \fbox{%
        \begin{minipage}{\dimexpr\linewidth-2\fboxsep-2\fboxrule\relax} 
            \vspace{2pt}
            \begin{itemize}[leftmargin=-1pt, itemsep=4pt, topsep=2pt, parsep=0pt] 
                \item[] \textbf{Goal \#1: Global Robustness (Minimax Optimality).}\\
                The algorithm must ensure regret $\tilde{\mathcal{O}}(\sqrt{T}+\sqrt{D})$ against the optimal arm in hindsight, even if the environment is fully adversarial and delays are arbitrary.
        
                \item[] \textbf{Goal \#2: Risk-Averse Behavior at No Extra Cost.}\\
                The algorithm must incur nearly {constant} regret $\tilde{\mathcal{O}}(1)$ relative to the designated safe policy. This ensures that the ``price of learning'' is bounded, and the method effectively reverts to the baseline when experimentation is risky.
            \end{itemize}
        \end{minipage}%
    }
\end{center}
\vspace{-0.75em}
\subsection{How Delay Breaks Safety: Three Roadblocks}
\vspace{-1em}
Achieving the above goals have been already subtle in bandits.
We next highlight three obstacles that make this twofold objective fundamentally harder under delays, and guide our algorithmic solution:
\vspace{-1em}
\paragraph{A) The Price of Low Comparator Regret in Bandits.}
Achieving this ``best-of-both-worlds''\footnote{In the vast bandit literature, ``\textit{best-of-both-worlds} (BOBW)'' most commonly denotes algorithms that attain instance-optimal $\tilde{O}(\log T)$ regret in stochastic environments while maintaining $\tilde{O}(\sqrt{T})$ regret in adversarial ones. Here, we use \emph{BOBW} in a different sense: simultaneous guarantees (i) against a designated comparator and (ii) in the worst case. To our knowledge, this regret-based interpretation was first explicitly termed as ``best-of-both-worlds'' by recent work of \citet{MullerEtAl2025}.} trade-off is non-trivial even under immediate bandit feedback.
Specifically, comparator-adaptive learning with bandit feedback is subject to a sharp impossibility result: \citet{lattimore2015pareto} proved that any algorithm guaranteeing $O(1)$ regret relative to a specific deterministic arm must necessarily suffer linear worst-case regret, $\Omega(T)$, against other arms.

This lower bound stems from an inherent \emph{exploration--safety conflict}: to maintain constant regret against a deterministic strategy, i.e., a specific arm, an algorithm is forced to select that action almost exclusively, thereby starving itself of the information required to detect if another arm has become superior.
Consequently, strong safety relative to a deterministic comparator is incompatible with sublinear adversarial regret. 
To circumvent this barrier in a minimal way, \citet{MullerEtAl2025} relax the requirement on the \emph{safe comparator}: we assume the safe baseline is an \emph{exploratory policy} that assigns strictly positive probability to every action.

\vspace{-0.25em}\paragraph{B) Why Banker, and Why Safety Is Harder.}
Lags in the environment impose additional theoretical burdens. Standard minimax-optimal algorithms for online learning and bandits—including \emph{Follow-the-Regularized-Leader} (FTRL) and \emph{Online Mirror Descent} (OMD) \citep{shalev2025online,warmuth1997continuous,beck2003mirror} and their bandit instantiations such as \textsc{Exp3} \citep{auer2002nonstochastic} and \textsc{ZO} \citep{flaxman2004online,abernethy2009competing}—rely on potential-based analyses whose core mechanism is a telescoping cancellation across consecutive rounds. Delays disrupt this real-time structure: the learner updates using stale information, while the losses needed for the next cancellation step may still be hidden. This mismatch has historically led to delay-specific techniques and ad-hoc corrections \citep{joulani2013online,thune2019nonstochastic,flaspohler2021online,gyorgy2021adapting,zimmert2020optimal}.

To address the safety objective on top of delays, a natural first attempt would be to combine such a delay-adapted method with an existing mechanism for baseline safety. This integration is non-trivial. As discussed in \citet[Appendix~A]{MullerEtAl2025}, standard safety approaches---including shifted variants \citep{sani2014exploiting} of \textsc{Prod-MWU} \citep{cesa2007improved,gaillard2014second}, originating from the seminal work of \citet{even-dar-2008}, as well as parameter-free adaptations \citep{orabona2019modern,van2020comparator}—do not directly yield strong baseline-safety guarantees in the bandit setting. The remaining viable strategy in bandit feedback is \emph{phased aggression} recently explored by \cite{MullerEtAl2025} in the non-delay safe online bandit learning.

This paradigm maintains a controlled mixture between the safe baseline and an exploratory no-regret learner, becoming more aggressive only after the data explicitly certifies that the baseline is suboptimal. Yet such technique does not align naturally with generic FTRL frameworks for delayed bandits, whose updates are driven by implicit regularization rather than explicit exploration. The construction of \citet{MullerEtAl2025} suggests the right path forward—an OMD-based approach, in which explicit exploration can be grafted onto the learning dynamics in a principled way. Thus, to make phased aggression compatible with delays, we need a delay-robust OMD backbone whose modularity exposes and controls exploration. \texttt{Banker-OMD} \citep{huang2023banker} of Huang et. al is precisely such a substrate. 

By restoring potential-based reasoning under arbitrary and unknown delays, it cleanly separates no-regret dynamics, bandit feedback, and delay handling, directly addressing our first requirement of near-minimax adversarial regret under delays.
What it does \emph{not} address is our second objective—near-zero regret against a safe baseline. The missing ingredient is not delay robustness alone, but delay-robust \emph{safe switching}: under delayed feedback, the hypothesis tests driving phased aggression rely on stale observations, and unseen losses create a hidden debt that can make exploration appear safer, or the baseline less suboptimal, than it truly is.
This leads to the central technical question:
\vspace{-0.5em}
\begin{center}
\textbf{\emph{How can we safely switch from conservatism to exploration when \ delays create a hidden debt \\that can mask the safe policy's suboptimality?}}
\end{center}
\vspace{-1.15em}

\paragraph{C) Late Switches Between Conservatism and Aggression.}
The core mechanism of phased aggression relies on a restart condition that detects when the safe comparator $x^c$ performs poorly.
In the no-delay setting, \citet{even-dar-2008,MullerEtAl2025}, achieved this by comparing a \emph{phase-gap statistic} to a threshold of $ \mathfrak{R}$: exceeding this value certifies strictly negative regret against $x^c$ within the phase, enabling geometric cancellation and fixed comparator regret. With delayed feedback, this certification becomes treacherous.  At time $T$, the
algorithm can only compute the phase gap from feedback that has already
arrived:
{\setlength{\abovedisplayskip}{3pt}
\setlength{\belowdisplayskip}{3pt}\[
G_T=\max_{x\in\text{Strategies}}\left\langle \sum_{t \in \text{observed by } T} \text{LossEstimation}[t], x^c-x \right\rangle. 
\]}
Crucially, in the phase that triggers a restart, the losses whose feedback has
not yet arrived contribute an unknown correction term. This creates a
``hidden debt'': the baseline may appear less suboptimal than it truly is, or
past exploration may appear cheaper than it was. In our analysis, this
unobserved distortion can be as large as $O(\sqrt{D}/\delta)$, potentially
offsetting the evidence certified by the standard threshold $\mathfrak{R}$.
Thus, relying on the vanilla restart test risks mistimed transitions between
conservatism and aggression, and can break the baseline-safety guarantee.

\subsection{Our Contributions}\label{sec:contributions}
The aforementioned discussion suggests that a simplistic ``plug-and-play'' approach—simply attaching a safety wrapper to a delay-tolerant bandit algorithm—is insufficient. The interaction between delayed feedback and safety test creates an \emph{information deficit} that necessitates a new synergistic design.

To address this, one might first ask: \emph{How can we adapt the restart threshold $\mathfrak{R}$ to account for missing information?} A key theoretical observation driving our work is that augmenting the standard threshold with a delay-dependent slack $\xi \approx O(\sqrt{D})$ effectively neutralizes the risk of ``ghost losses'' from unobserved feedback.
However, a critical obstacle arises: the total delay $D$ is unknown in advance. Relying on a worst-case upper bound (where feedback might arrive only at the very end, implying $D \approx T^2$) would necessitate a prohibitive slack of $\xi \approx O(T)$. While safe, this pessimistic parameterization would yield a vacuous linear regret bound $O(T)$, forfeiting the benefits of learning in any realistic environment where delays are moderate (i.e., $D \ll T^2$).

To resolve this dilemma, we introduce \textsc{Prudent-Banker}, an algorithm
that adapts to the unknown delay landscape without sacrificing safety or
performance. We prove matching upper and lower bounds for the
safety-constrained adversarial bandit problem, showing that the resulting
trade-off is optimal up to logarithmic factors, both with and without delayed
feedback. To the best of our knowledge, this gives the first affirmative
resolution of \textsc{SOLID} in the adversarial delayed-feedback setting,
improving the immediate-feedback framework of \citet{MullerEtAl2025}.

\begin{theorem}[Informal Main Result]\label{thm:informal:main}
\textsc{Prudent-Banker} is \textsc{SOLID} and achieves the optimal
safety--robustness instance-dependent constants up to logarithmic factors.
\end{theorem}

\section{Preliminaries \& Problem Setup}
\label{sec:prelims}
\vspace{-0.5em}
\paragraph{Basic Notation.} 
Let $[A]\coloneqq\{1,2,\dots,A\}$ denote the set of arms. We denote by $\Delta_A$ the probability simplex over $[A]$. For vectors $x,y \in \mathbb{R}^A$, we denote the inner product by $\langle x, y \rangle$. We use $\mathbf{e}_a$ for the $a$-th standard basis vector and $\tilde{O}(\cdot)$ to hide polylogarithmic factors in $T$ and $A$.
\paragraph{The Delayed Adversarial Bandit Protocol.}
We consider adversarial bandits with delayed feedback over horizon $T$ against
an oblivious adversary. At each round $t\in [T]$, the adversary chooses
$\ell_t\in[0,1]^A$ and delay $d_t\in\mathbb N_0$; the learner chooses
$x_t\in\Delta_A$, samples $A_t\sim x_t$, incurs unobserved loss
$\ell_{t,A_t}$, and observes $(A_t,\ell_{t,A_t})$ at the end of round
$t+d_t$.
\paragraph{Feedback Sets and Delay Counters.}
Let $\mathcal{O}_t \coloneqq \{ \tau \le t : \tau + d_\tau = t \}$ be the feedback arriving at the end of round $t$, and let $D \coloneqq \sum_{t=1}^T d_t$ be the \emph{total delay}.  Following \citet{huang2023banker}, let $\mathfrak d_t\coloneqq\sum_{\tau<t}\mathbf 1\{t\le \tau+d_\tau\}$ be the number of outstanding observations at $t$ and $\mathfrak D_t\coloneqq\sum_{r\le t}\mathfrak d_r$, with $\mathfrak D_T\le D$. 
For a phase $k(s)$ within stage $s$ of \textsc{Prudent-Banker} with start time $\operatorname{start}_{k(s)}$, we define the phase-local delay  parameters up to time $t$ as 
$D_t^{k(s)}\coloneqq\sum_{r=\operatorname{start}_{k(s)}}^t d_r$, $\mathfrak d_t^{k(s)}\coloneqq\sum_{\tau=\operatorname{start}_{k(s)}}^{t-1}\mathbf 1\{t\le \tau+d_\tau\}$, and $\mathfrak D_t^{k(s)}\coloneqq\sum_{u=\operatorname{start}_{k(s)}}^t\mathfrak d_u^{k(s)}$.

\paragraph{Regularization and Banker-OMD Setup.}
Our active learner uses OMD with a strictly convex, continuously differentiable
regularizer $\Psi:\mathcal D\to\mathbb R$ on an open convex set
$\mathcal D\supseteq\Delta_A$. We identify $\Psi$ with its restriction to
$\Delta_A$ by setting $\Psi(x)=+\infty$ for $x\notin\Delta_A$, s.t.
$\Psi^*(y)\coloneqq\sup_{x\in\mathbb R^A}\{\langle y,x\rangle-\Psi(x)\}$
implicitly enforces the simplex constraint. 
Finally, we denote the associated \emph{Bregman divergence} as:
{\setlength{\abovedisplayskip}{3pt}
\setlength{\belowdisplayskip}{1pt}\[D_\Psi(x,y)\coloneqq\Psi(x)-\Psi(y)-\langle\nabla\Psi(y),x-y\rangle.\]}
We characterize the geometry (``width'' and ``curvature'') of the regularizer via the following constants:
\vspace{-1.5em}
\begin{definition}[$(C_1, C_2)$-Regularity]
\label{def:regularity}
The regularizer $\Psi$ and initial point $x_0$ satisfy:
\begin{enumerate}[label=(\roman*), leftmargin=1pt,itemsep=0pt, topsep=-1pt, parsep=0pt]
    \item[] \textbf{$\star$ Bounded Diameter:} $D_\Psi(y, x_0) \le C_1$ for all $y \in \Delta_A$.
    \item[] \textbf{$\star$ Local Norm Stability:} $\mathbb{E}\left[\sigma_t D_{\Psi}\left(x_t, \tilde{z}_t\right)\right] \leq {C_2}/{\sigma_t} \text{ for all } t \in [T]$.
\end{enumerate}
\vspace{-0.75em}
\end{definition}
\emph{Remark. For negative entropy, $(C_1,C_2)=(\ln A,1/\delta)$; for $1/2$-Tsallis entropy,
$(C_1,C_2)=(2(\sqrt A-1),2/\delta)$. Details are defered in the appendix.}
\paragraph{Performance Metrics.} We measure our performance using standard adversarial \emph{pseudo-regret}:
{\setlength{\abovedisplayskip}{1pt}
\setlength{\belowdisplayskip}{1pt}\[
    \mathcal{R}_T \coloneqq \max_{x \in \Delta_A} \mathbb{E}\left[\sum_{t=1}^T \langle \ell_t, x_t - x \rangle\right]\tag{\text{Adversarial Regret}}
\]}The learner is also equipped with a designated \emph{safe baseline policy} $x^c \in \Delta_A$.
In light of the impossibility result of \citet{lattimore2015pareto} and the relaxation advocated by \citet{MullerEtAl2025}, we focus on \emph{smoothed} comparators and assume that the baseline is \emph{uniformly exploratory}, i.e., it assigns nonzero probability to every action:
\begin{assumption}[Exploratory Baseline]
\label{ass:exploratory}
There exists a constant $\delta \in (0, 1/A]$ such that $x^c(a) \ge \delta$ for all arms $a \in [A]$.
\end{assumption}
Our secondary objective is to minimize the \emph{comparator regret} with respect to this baseline:
{\setlength{\abovedisplayskip}{0pt}
\setlength{\belowdisplayskip}{0pt}\[\mathcal{R}(x^c) \;\coloneqq\; \mathbb{E}\left[ \sum_{t=1}^T \langle \ell_t, x_t-x^c\rangle\right] \tag{\text{Comparator Regret}}
\]}
\vspace{-1em}
\section{The \textsc{Prudent-Banker} Algorithm}
\vspace{-1em}
\label{sec:algorithm}
{We show that \textsc{Prudent-Banker} is to guarantee $\mathcal{R}_T = \tilde{O}(\sqrt{C_1 C_2 T} + \sqrt{C_1 C_2 D})$ worst-case, while ensuring $\mathcal{R}(x^c) = O(\log D)$ relative to the safe baseline.}For the negative-entropy regularizer, this yields the concrete bound
\[
\mathcal R_T=
\widetilde O\!\left(\delta^{-1/2}(\sqrt{T}+\sqrt{D})\right),
\qquad
\mathcal R(x^c)=O(\log D).
\]
At a high level, \textsc{Prudent-Banker} runs in \emph{epochs} and combines
three interlocking components: (i) an \emph{active no-regret learner} based on
\texttt{Banker-OMD}; (ii) a \emph{mixing schedule} implementing phased
aggression; and (iii) a \emph{safety guardrail} based on delay-calibrated
restarts.\subsection{Algorithm Description}

\begin{wrapfigure}{r}{0.6\textwidth}
\vspace{-2\baselineskip}
\begin{algobox}
\vspace{-1em}
\begin{minipage}{\linewidth}
\begin{algorithm}[H]
\caption{\small \textsc{Prudent-Banker}}
\label{alg:prudent_banker}
\footnotesize
\begin{algorithmic}[1]
\State \textbf{Instantiate Banker:} $\mathcal{B} \gets \texttt{Banker-OMD}(\Psi,S)$
\State \hspace{\algorithmicindent} \{regularizer $\Psi$ and step-size routine $S$\}

\State \textbf{Input:} Initial \& comparator strategies $(x_0,x^c)$, $x_0 = \mathbf{1}/A$, horizon $T$, comparator margin $\delta \in (0,1)$.
\State \textbf{Parameters:} Constants $C_1,C_2>0$ for thresholding.
\State \textbf{Initialize:} $\mathfrak{C}\gets 10\sqrt{C_1^{\Psi}C_2^{\Psi}}$; global time $t\gets 1$.
\State \hspace{\algorithmicindent} \textbf{Delay-Stage} $s\gets 1$, Delay Estimate $\widehat D_s\gets 1$.
\State \hspace{\algorithmicindent} \textbf{Aggression-Phase} $k\gets 1$, Start of instance $\tau_s\gets 1$.
\State \textbf{Define functions:}
\State \hspace{\algorithmicindent} $\widehat R(D)\gets \mathfrak{C}\bigl(\sqrt{T}+\sqrt{D\log(D+1)}\bigr)$
\State \hspace{\algorithmicindent} $\widehat \xi(D)\gets \frac{\sqrt{8D+1}-1}{\delta},\quad B(D)\gets 2\widehat R(D)+\widehat \xi(D)$
\State \hspace{\algorithmicindent} Aggression $\alpha_k\gets \min\{1,\,1/\widehat R(\widehat D_s)\}$.
\State \hspace{\algorithmicindent} Initialize \texttt{Banker-OMD} state $x_{\tau_s}=x_0$.

\While{$t\le T$}
    \State {\color{blue}\{Step 1\}}
    \State {\color{blue}\{Adaptive Delay Check Stage Logic\}}
    \State Observe delays for arriving feedback. Let $D_t^{(s)}$ be total delay in current stage.
    \If{$D_t^{(s)}>\widehat D_s$}
        \State $s\gets s+1$, $\widehat D_s\gets 2\cdot D_t^{(s)}$
        \State {\color{blue}\{Doubling Trick on Delay\}}
        \State \textbf{Hard Restart:} $\tau_s\gets t+1$, $k\gets 1$, Reset $\mathcal{B}$ to $x_0$.
        \State Recalculate $\alpha_k$ with new $\widehat D_s$.
    \EndIf

    \State {\color{blue}\{Step 2: Base Learner (Banker-OMD)\}}
    \State Query learner $\mathcal{B}$ for base prediction $\hat x_t$.
    \State \emph{(Internal):} $\mathcal{B}$ computes step-size $\sigma_t$ based on $S$ and available feedback credits.

    \State {\color{blue}\{Step 3: Safe Mixture\}}
    \State Sample $A_t\sim x_t \leftarrow \alpha_k\hat x_t+(1-\alpha_k)x^c$.
    \State Incur loss $\ell_t(x_t)$ (unobserved), receive feedback for past rounds $u\in\mathcal F_t$.
    \State Update $\mathcal{B}$ with pairs $\{(u,x_u(A_u),\ell_u)\}$ for all $u\in\mathcal F_t$ and construct the importance-weighted estimator
    \State \hspace{\algorithmicindent} $\hat\ell_u\gets \frac{\ell_u(A_u)}{x_u(A_u)}\,e_{A_u}$.

    \State {\color{blue}\{Step 4: Phased Safety Check\}}
    \State Compute empirical gap against baseline on observed data:
    \State \hspace{\algorithmicindent} $\widehat\Delta_k(t)=\sum_{u=\tau_s}^{t}\mathbb I\{u\in\textsc{Arrived}\}\langle \hat\ell_u,x^c-\hat x_u\rangle$

    \If{$\widehat\Delta_k(t)>B(\widehat D_s)$ \textbf{and} $\alpha_k<1$}
        \State $k\gets k+1$
        \State {\color{blue}\{Enter new Aggression Phase\}}
        \State Increase aggression: $\alpha_k\gets \min\{2^{k-1}/\widehat R(\widehat D_s),1\}$
        \State \textbf{Soft Restart:} $\tau_s\gets t+1$, Reset $\mathcal{B}$ to $1/A$.
    \EndIf
\EndWhile
\end{algorithmic}
\end{algorithm}
\end{minipage}
\end{algobox}
\vspace{-4\baselineskip}
\end{wrapfigure}
The pseudocode for \textsc{Prudent-Banker} is provided in Algorithm~\ref{alg:prudent_banker}. For clarity, we present the core logic here; the formal definition, including the precise configuration of the \texttt{Banker-OMD} step-size schedules and constants, is deferred to the supplement (See Appendix~\ref{app:sec:alg}).

The algorithm orchestrates three hierarchical levels of adaptation:(1) \emph{Stages} to adapt to the unknown total delay $D$, (2) a \emph{Banker-OMD} core to handle local feedback latency, and 
(3) \emph{Phases} to adapt the aggression $\alpha$ for safety. More analytically, the following observations are in order.

\paragraph{1. Delay-Stage Logic (Unknown $D$).}
The outermost layer handles the unknown total delay. We define our threshold function $B(D)$ based on the regret $\widehat{R}(D)$ and a slack function $\xi(D)$ (Line 10). The algorithm maintains a delay estimate $\widehat{D}_s$, initially optimistic. If the accumulated delay within a stage exceeds this estimate (Line 17), a \emph{Stage} transition is triggered ($s \to s+1$). This involves doubling the estimate $\widehat{D}_{s+1} \approx 2 D_t^{(s)}$ and performing a \textbf{Hard Restart} (Line 20), effectively resetting the instance to adapt to the newly discovered delay scale.

\paragraph{2. Banker-OMD Core (Feedback Latency).}
At the heart of the process is the \texttt{Banker-OMD} learner $\mathcal{B}$. It manages local latency by maintaining an account of step-size credits, computing gradients using importance-weighted estimators $\hat{\ell}_u$ constructed only upon feedback arrival (Lines 28-29). This ensures that the base strategy $\hat{x}_t$ maintains $\tilde{O}(\sqrt{T} + \sqrt{D})$ regret regardless of the arrival order.

\paragraph{3. Aggression-Phase Logic (Safety).}
The safety mechanism is controlled by the aggression parameter $\alpha_k$. We continuously monitor the empirical gap $\widehat{\Delta}_k(t)$ between the baseline $x^c$ and the learner $\hat{x}$ on the \emph{observed} data (Lines 32-34).
Crucially, the restart threshold $B(\widehat{D}_s) = 2\widehat{R}(\widehat{D}_s) + \xi(\widehat{D}_s)$ includes a delay-dependent slack $\xi(\widehat{D}_s)$ scaling with $\sqrt{\widehat{D}_s}/\delta$. This term acts as a buffer against the hidden performance gap caused by missing feedback.
If the gap exceeds this threshold (Line 34), we conclude that the baseline is suboptimal. The algorithm then enters a new \emph{Phase} ($k \to k+1$), increases aggression to $\alpha_k \propto 2^{k-1}/\widehat{R}$ (Line 37), and performs a \textbf{Soft Restart}, resetting the learner to the uniform distribution to lock in safety gains while exploiting the optimal arm.

\begin{remark}[The roles of the two restarts]
\it A \textbf{Soft Restart}
implements phase aggression: once the observed data certifies that the safe
comparator is suboptimal, the algorithm starts a new phase with a larger weight
on the active learner. A \textbf{Hard Restart} implements delay adaptation: ideally,
\textsc{Prudent-Banker} would be calibrated using the Banker-OMD regret budget,
but this budget depends on the unknown total delay $D$, including future unseen
feedback. Since this quantity is not available online, the algorithm maintains a
delay guess $\widehat D_s$ and restarts whenever the certified delay scale
exceeds it. Thus, \textbf{Soft Restart} adapts the
baseline--exploration mixture, whereas \textbf{Hard Restart} adapt the regret budget to
the unknown delay landscape.
\end{remark}
\section{Main Results and Analysis}
\label{sec:main_results}
\subsection{Proof Sketch: Safety under Latency}
\vspace{-0.5em}
We now present our main theoretical guarantees. The analysis relies on a hierarchical decomposition of the regret, treating delay adaptation and safety constraints as coupled objectives. We begin with a high-level roadmap of the proof architecture. Our proof strategy unfolds in three logical steps:

\textbf{1. Decomposition via Stages and Phases.}
We first bound the number of delay-adaptive stages using \cref{lem:delay-mass}, ensuring that the doubling trick on $\widehat{D}_s$ incurs only logarithmic overhead. Within each stage, the regret is analyzed through the lens of \emph{phases}. We distinguish between \emph{Normal Phases} (where the algorithm maintains the baseline mixture) and \emph{Extinction Phases} (where the algorithm detects suboptimality and restarts).

\textbf{2. The Latency-Robust Base Learner.}
A core component is establishing that the base learner remains robust to arbitrary delays. \cref{lem:banker} proves that \texttt{Banker-OMD}, by managing a virtual account of step-sizes, achieves the optimal $\tilde{O}(\sqrt{T} + \sqrt{D})$ regret bound regardless of the arrival order of feedback.

\textbf{3. Bridging the Gap: The Cost of Missing Information.}
The most delicate aspect of the proof—distinguishing our work from classical phased aggression—is handling the \emph{pending feedback}. In standard settings, the empirical gap is exact. Here, the safety check is performed on stale data. We control this via \cref{lem:missing}, which bounds the potential deviation caused by unobserved losses. This bound motivates the slack term $\xi(\widehat{D}_s)$ in our threshold. Finally, \cref{lem:normal} and \cref{lem:extinction} demonstrate that the safe comparator regret is either $2^{k(s)-1}$ (in normal phases) or strictly negative ($-2^{k(s) -1}$ in extinction phases), allowing for a telescopic summation that cancels out the exploration cost, yielding the result of \cref{thm:main}.
\vspace{-0.25em}
\subsection{Our Algorithmic Guarantee}
\vspace{-0.25em}
Our first result shows that \textsc{Prudent-Banker} simultaneously achieves
near-minimax adversarial robustness and bounded regret against the safe
baseline, even under arbitrary delays.
\begin{theorembox}
\begin{theorem}[Global Safety and Robustness]
\label{thm:main}
For any delay sequence with total delay $D$, let $\{x_t\}_{t=1}^T$ be the
strategies generated by \textsc{Prudent-Banker} with step-size schedule
\[
\sigma_t=
\sqrt{\frac{C_2}{C_1}}
\left(
\frac{1}{\sqrt{t-\operatorname{start}_{k(s)}+1}}
+
\mathfrak d_t^{k(s)}
\sqrt{
\frac{\ln(\mathfrak D_t^{k(s)}+1)}
{\mathfrak D_t^{k(s)}}
}
\right)^{-1}.
\]
Then
\( \qquad
\mathcal R_T
\le
\widetilde O\!\left(\sqrt{C_1C_2}(\sqrt T+\sqrt D)\right) \text{ and }
\mathcal R_T(x^c)
\le O(1+\log D) =
\widetilde O(1).
\)
\end{theorem}
\end{theorembox}
\vspace{-0.75em}
 As an immediate consequence, in the
non-delayed case our bound closes the gap between the safety-constrained upper bound $O(\delta^{-1}\sqrt{T}))$
and lower bound $O(\sqrt{\delta^{-1}T}))$ left by \citet{MullerEtAl2025}.
\vspace{-1em}
\begin{corollary}[Optimal non-delayed safety-constrained regret]
\label{cor:nondelayed-optimal}
In the non-delayed setting $D=0$, instantiating \textsc{Prudent-Banker} with
negative-entropy regularization where $C_1C_2=\log A/\delta$ gives
\[
\mathcal R_T
\le
\widetilde O\!\left(\sqrt{{\delta^{-1}{T\log A}}}\right),
\qquad
\mathcal R_T(x^c)
\le O(1).
\]
Consequently, \textsc{Prudent-Banker} matches the safety-constrained lower
bound of order $\Omega(\sqrt{T/\delta})$ up to logarithmic factors.
\end{corollary}

\paragraph{Deciphering the arm dependence.}
Without the safety constraint, the minimax adversarial regret is bounded below by $\Omega(\sqrt{AT}+\sqrt{D\log A})$. It is therefore natural to ask how our guarantee compares to this unconstrained benchmark. If we take the safe comparator to be uniform (yielding $\delta=1/A$) and instantiate \textsc{\small Prudent-Banker} with negative entropy, Theorem~\ref{thm:main} gives an upper bound of $\widetilde{O}(\sqrt{AT\log A}+\sqrt{AD})$. Thus, relative to the unconstrained benchmark, our dependence on $T$ is optimal up to logarithmic factors, while the dependence on $D$ incurs an extra $\sqrt{A}$ penalty. This delay-side suboptimality is not an artifact of our baseline-safety mechanism, but rather an inherent limitation inherited directly from the \texttt{Banker-OMD} backbone.

However, this apparent inflation in regret is not an artifact of our analysis, but rather the fundamental price of safety. As we discuss in Section~\ref{sec:lower-bound}, our novel lower bounds prove that this resulting dependence is optimal for the safety-constrained problem, up to logarithmic factors.

\subsection{Key Lemmas}
\vspace{-0.5em}

\noindent We now detail the technical lemmas that constitute the building blocks of the proof.
\vspace{-0.5em}
\begin{lemma}[Conservation of Delayed Mass]
\label{lem:delay-mass}
Let $\mathfrak{D}_t^{k(s)}$ denote the cumulative number of missing feedback items inside phase $k(s)$ up to time $t$. Then, for all $\text{start}_{k(s)} \le t\le \text{end}_{k(s)}$:
\[
    \mathfrak{D}_t^{k(s)} \leq \mathfrak{D}_{\operatorname{end}_{k(s)}}^{k(s)} \leq D_{\operatorname{end}_{k(s)}}^{k(s)}
    \]
\end{lemma}

\begin{remark}
This structural result allows us to decouple the instantaneous missing mass from the time index. It implies that the ``penalty'' for missing information in any regret bound is bounded by the total delay complexity $D_{\operatorname{end}_{k(s)}}^{k(s)}$ inside phase $k(s)$, rather than the length of phase $k(s)$. This validates the usage of $D_{\operatorname{end}_{k(s)}}^{k(s)}$ as the natural complexity measure for the delay-adaptive step-sizes.
\end{remark}

\begin{lemma}[Bound on Estimated Total Delay]
\label{lem:delay-estimator}
    For every phase $k(s)$ inside stage $s$, the total-delay capacity satisfies:
    \[
    \widehat{D}_{s} \leq \widehat{D}_{s+1}<2 D_{\operatorname{end}_{k(s)}}^{k(s)}
    \]
\end{lemma}

\begin{remark}
    This ensures that the number of stages remains logarithmic ($S \approx O(\log D)$), allowing us to sum per-stage regrets efficiently without incurring a linear penalty. For each phase $k(s)$ inside stage $s$, they share the same $\widehat D_s$. This indicates we can interpret $\widehat D_s = \widehat D_{k(s)}$. 
\end{remark}

\begin{lemma}[Banker-OMD Guarantee]
\label{lem:banker}
Consider the \texttt{Banker-OMD} iterates $\{\hat{x}_t\}$ using adaptive the step-sizes of Theorem~\ref{thm:main}
Then, the regret against any comparator $x \in \Delta_A$ satisfies:
\[
\max _{x \in \Delta_A} \mathbb{E}\left[\sum_{t=\mathrm{start}_{k(s)}}^{\mathrm{end}_{k(s)}} \langle \hat{\ell}_t, \hat{x}_t - x \rangle\right]
\;\le\;
\tilde{O}\!\left(\sqrt{C_1C_2}\,(\sqrt{T_s}+\sqrt{D_{\operatorname{end}_{k(s)}}^{k(s)}})\right)
\]
where $T_s = \text{end}_{k(s)} - \text{start}_{k(s)} + 1$.
\end{lemma}

\begin{remark}
\cref{lem:banker} establishes the ``best-of-both-worlds'' capability of the base learner. By dynamically adjusting the learning rate $\sigma_t$ based on the arrived data density, the algorithm essentially pauses learning during high-latency periods and accelerates when feedback becomes available, ensuring optimal worst-case performance.
\end{remark}

\begin{lemma}[Bound on Missing Feedback]
\label{lem:missing}
Let $\mathcal{M}(T_0)$ denote the set of rounds $u \le T_0$ for which feedback has not arrived by time $T_0$ inside phase $k(s)$. Then, for any comparator $x \in \Delta_A$:
\[
    \sum_{u\in\mathcal{M}(T_0)} \langle \hat{\ell}_u,\, x^c-x\rangle
    \;\le\;
    O(\sqrt{D_{\operatorname{end}_{k(s)}}^{k(s)}}) .
\]
\end{lemma}

\begin{remark}
This lemma quantifies the risk of acting on stale information. Intuitively, missing feedback obscures the true performance gap. \cref{lem:missing} bounds the worst-case error introduced by these ``phantom'' losses. This bound is pivotal for the safety certification, as it dictates the magnitude of the slack term $\xi$ required in the restart condition to prevent false positives.
\end{remark}

\begin{lemma}[Normal Phase Analysis]
\label{lem:normal}
Consider any phase $k(s)$ that does not trigger a restart (i.e., $\alpha_{k(s)} < 1$). For all $x \in \Delta_A$, the regret on arrived data is bounded by the threshold:
{\setlength{\abovedisplayskip}{0pt}
\setlength{\belowdisplayskip}{0pt}
\(\phantom{11111111111111111}
\mathbb{E}\!\left[\sum_{t=\mathrm{start}_{k(s)}}^{\mathrm{end}_{k(s)}}
    \langle \hat{\ell}_t, x_t - x\rangle\right]
    \;\le\;
    4 \widehat R (D^{k(s)}_{\text{end}_{k(s)}}) + O(\sqrt{D^{k(s)}_{\text{end}_{k(s)}}}).
\)}\\
Furthermore, the regret against the safe baseline satisfies:
\(
\mathbb{E}\!\left[\sum_{t=\mathrm{start}_{k(s)}}^{\mathrm{end}_{k(s)}}
\langle \hat{\ell}_t, x_t-x^c\rangle\right] \;\le\; 2^{k(s)-1}.
\)
\end{lemma}

\begin{remark} Above we ensure that as long as the algorithm remains in a specific phase, the deviation from the  baseline is controlled by the current aggression level $\alpha_k$. The bound reflects that the empirical evidence (the gap statistic) has not yet provided sufficient proof to justify abandoning the current mixture.
\end{remark}

\begin{lemma}[Extinction Phase Analysis]
\label{lem:extinction}
If phase $k(s)$ is terminated by a restart condition (indicating the baseline is suboptimal), then:
\(
\mathbb{E}\Bigg[\sum_{t=\mathrm{start}_{k(s)}}^{\mathrm{end}_{k(s)}}\langle \hat{\ell}_t,\, x_t-x^c\rangle\Bigg]\;\le\;-2^{k(s)-1}.
\)
\end{lemma}

\begin{remark}
\cref{lem:extinction} is the mechanism that funds the exploration. It asserts that any phase ending in a restart must have accumulated strictly \emph{negative} regret against $x^c$. This negative regret serves as a "credit," offsetting the potential positive regret of future phases. This guarantees that the algorithm only escalates its aggression when it has already "banked" enough safety margin to do so risk-free.
\end{remark}

{\it
\textbf{Summary.}
\cref{thm:main} unifies the phase-level analyses to provide a global guarantee. For the pseudo-regret $\mathcal{R}_T$, the bound is dominated by the final aggressive phases where $\alpha \approx 1$, matched by the guarantee of the base learner (\cref{lem:banker}). Since $\alpha_k$ doubles geometrically, there are only logarithmically many phases, preserving any $\tilde{O}(\cdot)$ rate.
Crucially, for the safety constraint $\mathcal{R}_T(x^c)$, the geometric restarts ensure that the positive regret accumulated in the final phase is offset by the negative regret credit accumulated in previous {\color{red}``Pivot''} phases (\cref{lem:extinction}), resulting in a telescopic cancellation that bounds the deviation by a constant.
}

\subsection{Lower Bound: The Price of Safety Under Delay}
\label{sec:lower-bound}

Before concluding with the experiments, we ask the converse question:
\begin{center}
\emph{What is the best regret any safe algorithm can hope for under delayed bandit feedback?}
\end{center} We prove a lower bound showing that the $\sqrt{ \delta^{-1} T}+\sqrt{ \delta^{-1} D}$ rate achieved by \textsc{Prudent-Banker} is unimprovable in the safety-constrained regime---matching the standard delayed-bandit lower bounds while exposing that  $\delta^{-1/2}$ overhead  is the intrinsic cost of the safety constraint.

Our lower-bound provides a novel interplay between the two environments (safe,unsafe) \cite{MullerEtAl2025} and a construction  that follows the blocking methodology used in delayed
bandit lower bounds, particularly the batching idea of
\citet{zimmert2020optimal}. The key observation is that a delayed bandit
instance can be organized into blocks so that feedback from actions played
inside one block is revealed only after the block has ended. We first prove a
safety-constrained lower bound for the corresponding batched bandit game. We
then show that any delayed bandit algorithm can be simulated by a batched
algorithm on the induced block structure. This reduction transfers the batched
safe lower bound back to the delayed setting and yields a delay-side lower
bound of order $\delta^{-1/2}\sqrt D$. For concision, we defer the proof to \cref{app:lower-bound}.
\begin{theorembox}
\begin{theorem}[Delay-side price of safety; informal]
\label{thm:delay-side-safety-lb-informal}
There exist delay sequences with total delay $D$ and a safe comparator
$x^c$ satisfying $\min_i x_i^c\ge\delta$ such that every delayed bandit
algorithm satisfies, for some adversarial loss sequence,
\[
    \mathcal R_T(x^c)\le \log D
    \quad\Longrightarrow\quad
    \mathcal R_T
    \ge
    \widetilde\Omega\!\left(\delta^{-1/2}\sqrt D\right).
\]
\end{theorem}
\end{theorembox}

\paragraph{Putting time and delay together.}
The no-delay lower bound of \citet[Theorem~3.2]{MullerEtAl2025} already shows
that baseline safety alone forces regret
\(
    \Omega(\sqrt{\delta^{-1}T})
\)
up to lower-order terms. Combining this obstruction with
\cref{thm:delay-side-safety-lb-informal}, we obtain delay and loss sequences
for which any baseline-safe algorithm must incur
\[
    \mathcal R_T
    \ge
    \Omega\!\left(
        \sqrt{\delta^{-1}T}
        +
        \sqrt{\delta^{-1}D}
    \right)
    -
    O\!\left(
        \delta^{-3/4}T^{1/4}
        +
        \delta^{-3/4}D^{1/4}
    \right).
\]
Thus the guarantee of \textsc{Prudent-Banker} in \cref{thm:main} is optimal up
to logarithmic factors and lower-order terms, both in the delayed and non-delayed regimes.

\vspace{-0.5em}
\section{Experiments}
\label{sec:experiments}
\vspace{-0.5em}
We conclude our work by empirically corroborating our theoretical claims to demonstrate that
\PB{} delivers on \textsc{SOLID} in practice: it preserves baseline safety
under delay while remaining competitive with delay-aware bandit learners.
We pit \PB{} against two safety-aware baselines run with delayed feedback:
\CUCB{}~\citep{wu2016conservative} and \SafeEXP{}, which couples the
conservative wrapper of \citet{wu2016conservative} with
EXP3-IX~\citep{kocak2014efficient,neu2015explore}. All three methods process
feedback only upon arrival; the distinguishing feature of \PB{} is its
\emph{delay-calibrated} restart rule, while the baselines treat delay
implicitly through lazy updates.

\vspace{-0.5em}
\paragraph{Setup.}
We evaluate on a non-stationary synthetic bandit environment with horizon
$T=50{,}000$ and $A=100$ arms. Losses are generated in $S=500$ blocks, with
block-specific means and variances so that the best arm changes over time. To
stress-test safety, the comparator $x^c$ is a full-support distribution
concentrated on the hindsight best fixed arm $i^\star$, assigning mass
$1-(A-1)\delta$ to $i^\star$ and $\delta=10^{-3}$ to every other arm. This
oracle choice is deliberately favorable to the baselines and isolates the
safety--learning trade-off. We compare performance under three delay regimes:
no delay, geometric delays, and heavy-tailed Pareto delays, using the same loss
tables across conditions.\footnote{ Full implementation details, parameter choices, and
additional ablations are deferred to \cref{app:experiments}.}

\begin{figure}[h]
  \vspace{-0.75em}
  \centering
  \includegraphics[width=0.75\linewidth]{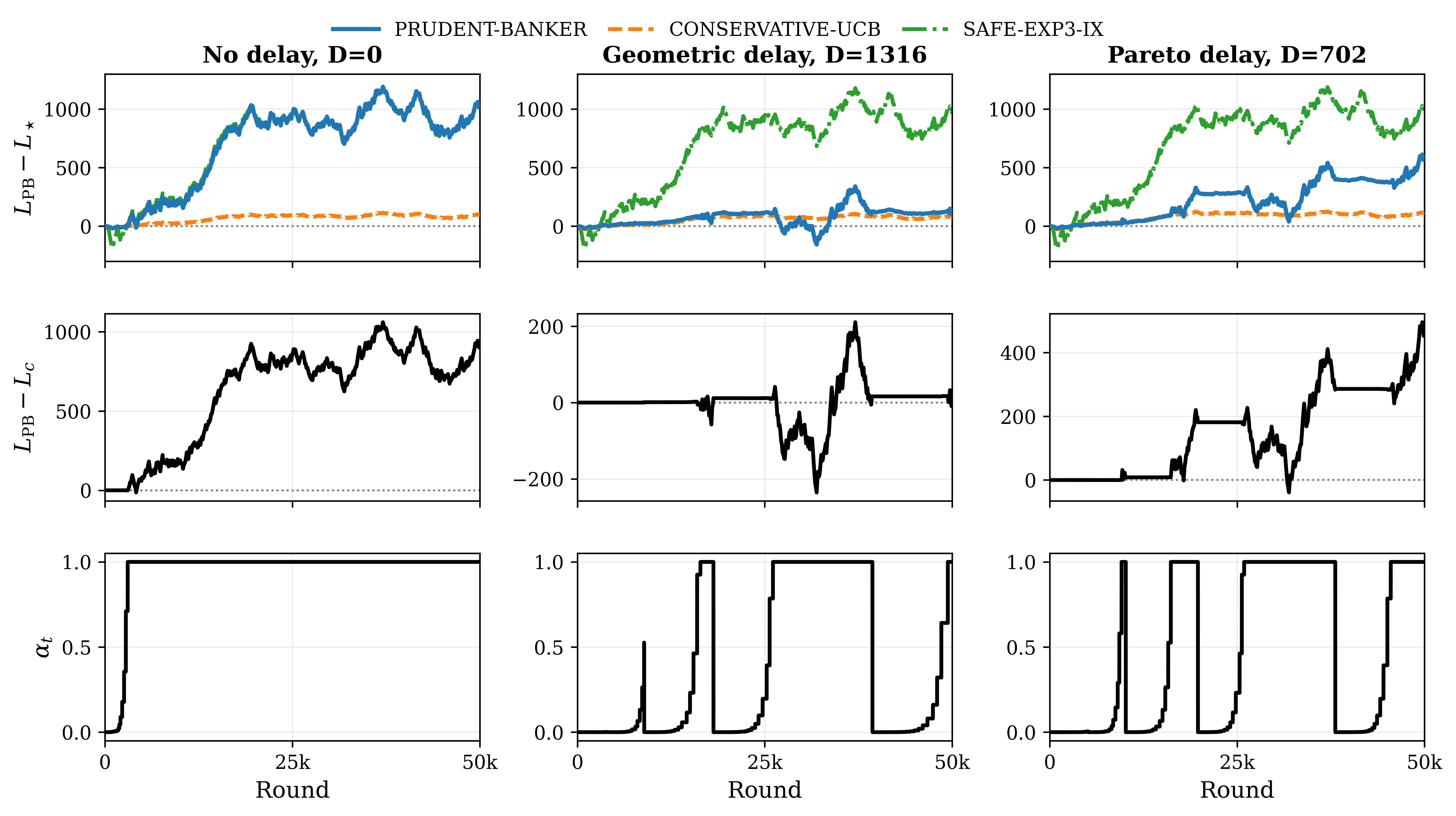}
  \vspace{-0.75em}
  \caption{\small Dynamic performance of \PB{} versus delay-aware safety baselines.
  \emph{Top:} cumulative pseudo-loss minus that of the hindsight best fixed arm
  $i^\star$. \emph{Middle:} \PB{} cumulative loss minus comparator loss
  $L_c$ (negative values mean \PB{} beats $x^c$). \emph{Bottom:} the \PB{}
  mixing coefficient $\alpha_t$. Under delay, \PB{} repeatedly performs
  conservative restarts and only re-enters aggressive phases after the
  delay-calibrated threshold certifies comparator suboptimality.}
  \label{fig:main}
  \vspace{-1em}
\end{figure}

\vspace{-0.5em}
\paragraph{Results.} Figure~\ref{fig:main} confirms the three predictions of our analysis. \emph{(i) Robust safety.} Across all three delay regimes, the middle row shows that the gap between \PB{} and the safe comparator stays bounded (and is repeatedly negative under both delay conditions), matching the $\widetilde{O}(1)$ comparator-regret guarantee of \cref{thm:main}. By contrast, the safety wrappers built around \CUCB{} and \SafeEXP{} only enforce a multiplicative $(1-\alpha_{\mathrm{safe}})$ budget and cannot certify comparator-level safety. \emph{(ii) Delay-calibrated phase transitions.} The bottom row visualizes the operational role of the delay-aware restart threshold: in the no-delay column, $\alpha_t$ rises monotonically to $1$ and stays there, recovering the immediate-feedback behavior of phased aggression; under the geometric and Pareto delays, $\alpha_t$ oscillates as the realized delay budget repeatedly exceeds the current estimate $\widehat D_s$, triggering hard restarts and demanding fresh evidence before re-aggression---exactly the mechanism that \cref{lem:missing,lem:extinction} were designed to enforce.
\emph{(iii) Worst-case competitiveness.} In the top row, \PB{} stays within an $\widetilde O(\sqrt{T}+\sqrt D)$ envelope of $i^\star$ and dominates \SafeEXP{}, whose adversarial exploration is wasteful when the comparator is already concentrated. \CUCB{} appears slightly tighter to $i^\star$ only because its default arm \emph{is} $i^\star$ in this diagnostic setting; this advantage is not online and disappears once the safe anchor is no longer hindsight-optimal.

\vspace{-0.5em}
\paragraph{Takeaway. \hspace{-1em}}
Together with upper and the lower bounds, our experiments support the central message of the paper: {\it baseline safety and delay-robust adversarial learning can
be combined without losing the optimal safety--robustness trade-off.}
Natural next
steps include adaptive or learned baselines and extensions to contextual,
linear, reinforcement-learning, and multi-agent settings.
\bibliographystyle{plainnat}
\bibliography{example_paper}
\newpage
\appendix
\onecolumn
\tableofcontents
\newpage
\section{Related Work}
\subsection{Exploitation vs. Safety Trade-offs.}
Approaches to safe exploitation typically deviate from minimax strategies to target sub-optimal opponents, limiting potential losses to the minimax value~\cite{ganzfried2015safe}. These methods often interpolate between conservative and aggressive strategies using tunable confidence parameters~\cite{damer2017safely, liu2022safe}. However, this introduces a fundamental trade-off: such algorithms either risk $\Omega(T)$ regret against strong adversaries~\cite{damer2017safely, liu2022safe} or lack theoretical guarantees for exploitation~\cite{ganzfried2015safe}. While recent advances establish logarithmic regret for small-scale games, they rely on restrictive assumptions like pairwise distinct payoffs or full information feedback~\cite{maiti2023logarithmic}. It is interesting to understand if we can achieve similar best-of-both-worlds results under adversarial setting with delayed feedback.

\subsection{Comparator Adaptivity in Full Information.}
Safety, defined as ensuring low regret relative to a pre-specified baseline policy $x^c$, is a cornerstone of ``safe AI.'' In the full-information setting, it is established that algorithms can achieve constant regret against a specific comparator strategy while maintaining optimal worst-case regret guarantees. Notable examples include \cite{even-dar-2008,hutter2005adaptive, kapralov2011prediction, koolen2013pareto, sani2014exploiting}. Similarly, parameter-free methods adapt to unknown comparators~\cite{orabona2016coin, cutkosky2018black, orabona2019modern}. These works successfully resolve the ``Best-of-Both-Worlds'' challenge regarding safety and adversarial robustness. \textsc{\small Prudent-Banker} extends the Phased Aggression framework~\cite{even-dar-2008} to the bandit setting. While algorithms above apply to immediate feedback setting well, our algorithm aims to achieve safety comparator guarantee on delayed, partial feedback.
\subsection{Hardness of Safe Bandit Learning.}
Under bandit feedback, achieving constant regret against a deterministic comparator generally implies linear worst-case regret against other actions~\cite{lattimore2015pareto}. This impossibility result necessitates relaxing assumptions on the comparator to make $\tilde{\mathcal{O}}(1)$ safety feasible. Recent work by Muller et al.~\cite{MullerEtAl2025} demonstrates that if the safe baseline is required to be \emph{full-support}, one can indeed achieve constant safety regret alongside near-optimal minimax regret using a phased aggression mechanism. We adopt this full-support assumption to render the safety goal achievable. However, the standard phased aggression mechanism in~\cite{MullerEtAl2025} relies on immediate feedback. \textsc{\small Prudent-Banker} extends the framework of~\cite{MullerEtAl2025} to delayed feedback setting.

\subsection{Conservative Bandits and Safe RL.}
Our framework's definition of safety is distinct from the paradigms often studied in conservative bandits~\cite{wu2016conservative} and conservative reinforcement learning (RL)~\cite{garcelon2020conservative}. In those contexts, algorithms are designed to satisfy a constraint—often defined as securing at least a $(1-\alpha)$-fraction of a baseline's return. While this ensures performance relative to a baseline, it allows for linear regret $\mathcal{O}(\alpha T)$ relative to the optimal policy in the worst case. Furthermore, in constrained settings such as Constrained MDPs~\cite{badanidiyuru2018bandits, efroni2020exploration}, achieving constant regret against a safe baseline is frequently unobtainable under adversarial environemnt ~\cite{liu2021learning}. In contrast, \textsc{\small Prudent-Banker} aims for constant $\tilde{\mathcal{O}}(1)$ regret relative to the baseline while simultaneously minimizing regret against the best strategy in hindsight, even in the presence of arbitrary delays.

\subsection{Adversarial bandits and mirror-descent viewpoints.}
Many techniques have been applied to adversarial multi-armed bandits (MAB) setting. For example, exponential-weights methods 
\textsc{EXP3}~\cite{auer2002nonstochastic}, Online Mirror Descent (OMD) and
Follow-the-Regularized-Leader (FTRL) frameworks~\cite{warmuth1997continuous,gordon1999regret,lattimore2020bandit}.
Some new developments employ different regularizer, including Tsallis-entropy and related OMD variants~\cite{zimmert2020optimal,wei2018more}.
Our work follows the OMD paradigm but targets a more complicated regime where (i) feedback is delayed,
and (ii)regret is constant relative to a designated safe comparator.

\subsection{Online learning with delayed feedback.}
Online learning with delayed feedback has been studied across several settings, including multi-armed bandits (MABs), linear bandits, and more general sequential decision-making problems. For MABs, early work considered stochastic settings with i.i.d.\ delays \citep{joulani2013online}, while adversarial formulations established minimax lower bounds under uniform delays \citep{cesa2016delay}; subsequent results achieved near-optimal regret depending on both the horizon $T$ and total delay $D$ via OMD-based methods with moving negative-entropy regularization \citep{bistritz2019online,thune2019nonstochastic,zimmert2020optimal}. Among which the \texttt{Banker-OMD} framework \citep{huang2023banker} provides a unified delay-handling mechanism that decouples delays from algorithm design and extends delayed bandit guarantees to scale-free loss settings. Delayed feedback has also been investigated in linear bandits, including stochastic models with i.i.d.\ or partially observable delays and adversarial models with uniform delays \citep{zhou2019learning,vernade2020linear,ito2020delay}. Beyond bandits, delays have been studied in stochastic and adversarial Markov Decision Processes, general online optimization with full-information or bandit feedback, and wireless network optimization \citep{mesterharm2005line,agarwal2011distributed,dudik2011efficient,desautels2014parallelizing,lancewicki2022learning,howson2021delayed,jin2022near,dai2022follow,huang2021robust}.

\newpage
\section{Formal Description of \textsc{Prudent-Banker}}
\label{app:sec:alg}

\begin{algorithm}[!htbp]
\caption{\textsc{Prudent-Banker}}
\label{alg:Full-banker-omd-phased}
\begin{algorithmic}[1]

\State \textbf{Input:} number of arms $A$, horizon $T$, comparator $x^c$, comparator margin $\delta\in(0,1/A]$, constants $(C_1,C_2)$, mirror map $(\Psi,x_0)$, $x_0 = \mathbf{1}/A$, step-size subroutine $S$

\State \textbf{Initialize:}
stage $s\gets1$,
$\widehat D_s\gets1$,
$k\gets1$,
$\operatorname{start}\gets1$,
$\hat{x}_1\gets x_0$,
$B_0 \gets 0$

\State Define
\[
\widehat R(\widehat D)=\sqrt{C_1C_2}\!\left(3\sqrt{T}+7\sqrt{2\widehat D\ln(\widehat D+1)}\right),
\qquad
\widehat\xi(\widehat D)=\frac{\sqrt{8\widehat D+1}-1}{\delta}
\]

\State $\widehat R_s\gets \widehat R(\widehat D_s)$,
$\widehat\xi_s\gets \widehat\xi(\widehat D_s)$,
$\alpha\gets \min\{1/\widehat R_s,1\}$

\For{$t=1$ to $T$}
  \State $D_{t}^{(s)} \gets \sum_{r=\operatorname{start}}^{t} d_r$
  \If{$D_{t}^{(s)}>\widehat D_s$}
    \State $s\gets s+1$, $k\gets1$, $\widehat D_s\gets 2^{\lceil \log_2 D_{t}^{(s)}\rceil}$
    \State $\operatorname{start}\gets t+1$, $\widehat R_s\gets \widehat R(\widehat D_s)$, $\widehat\xi_s\gets \widehat\xi(\widehat D_s)$
    \State $\alpha\gets \min\{1/\widehat R_s,1\}$, $\hat{x}_t\gets x_0$
    \State \textbf{continue}
  \EndIf

  \State
  \[
  \sigma_t \gets
  S\!\left(t-\operatorname{start}+1,\{(A_u,\mathrm{status}_u,\hat\ell_u,\sigma_u,v_u)\}_{u<t}\right)
  \]

  \State \textbf{Banker initialization:}
  $v_t\gets \sigma_t$, \ $b_t\gets \sigma_t$

  \ForAll{$u=\operatorname{start},\ldots,t-1$ with $\mathrm{status}_u=\textsc{arrived}$}
    \State $\sigma_{t,u}\gets \min\{v_u,b_t\}$
    \State $v_u\gets v_u-\sigma_{t,u}$, \ $b_t\gets b_t-\sigma_{t,u}$
  \EndFor

  \State $B_t \gets B_{t-1} + b_t$

  \State
  \[
  \hat x_t \gets
  \nabla\bar\Psi^*\!\left(
  \frac{1}{\sigma_t}\sum_{\operatorname{start} \le u \le t-1}\sigma_{t,u}\nabla\Psi(z_u)
  +\frac{b_t}{\sigma_t}\nabla\Psi(x_0)
  \right)
  \]

  \State $x_t\gets \alpha\hat x_t+(1-\alpha)x^c$
  \State Sample $A_t\sim x_t$, incur loss $\ell_t(A_t)$

  \State \textbf{Upon receiving each new feedback } $(u,\ell_u(A_u))$:
  \State \hspace{\algorithmicindent} $\mathrm{status}_u\gets \textsc{arrived}$
  \State \hspace{\algorithmicindent} $\hat\ell_u\gets \dfrac{\ell_u(A_u)}{x_u(A_u)}e_{A_u}$
  \State \hspace{\algorithmicindent} $z_u\gets
    \nabla\bar\Psi^*\!\left(\nabla\Psi(x_u)-\dfrac{1}{\sigma_u}\hat\ell_u\right)$

  \State $g_t\gets \sum_{u=\operatorname{start}}^{t}\mathbf 1\{\mathrm{status}_u=\textsc{arrived}\}\hat\ell_u$

  \If{$\max_{x\in\Delta_A}\langle g_t,x^c-x\rangle>2\widehat R_s+\widehat\xi_s$
      \textbf{and} $\alpha<1$}
    \State $k\gets k+1$, $\operatorname{start}\gets t+1$, $\hat x_{t+1}\gets \mathbf{1}/A$
    \State $\alpha\gets \min\{2^{k-1}/\widehat R_s,1\}$
  \EndIf
\EndFor

\end{algorithmic}
\end{algorithm}
\newpage
\subsection{Further Insights in Prudent-Banker}

From technical perspectives, our contributions are threefold:

\medskip
\noindent\textbf{1.\ Eliminating the safety certification gap.}
We derive a \emph{delay-calibrated restart condition} that accounts for the information deficit induced by delayed feedback.
In standard bandits, verifying that the baseline/comparator is suboptimal reduces to a relatively direct hypothesis test.
We show that under delays, this test must be relaxed by an explicit slack term
\[
\xi_s \;=\; \Theta\!\left(\sqrt{\frac{D_s}{\delta}}\right),
\]
where $D_s$ denotes the cumulative delay mass within stage $s$.
This slack acts as a buffer against \emph{phantom suboptimality} arising from unobserved losses, preventing spurious restarts and thereby preserving an $O(1)$ safety guarantee.

\medskip
\noindent\textbf{2.\ Cost-free adaptation to unknown delays.}
A central challenge is adapting to an unknown total delay $D$ without compromising safety. Standard ``doubling trick'' approaches \cite{auer1995gambling,auer2010ucb} typically treat each restart as a fresh instance. While such restarts incur only mild overhead for conventional regret, they can be catastrophic for safety: naively resetting the mechanism forfeits the \emph{negative regret credit} accumulated against $x^c$, potentially forcing the algorithm to pay an additive penalty proportional to the instance scale (e.g., $\sqrt{T}$ or $\sqrt{D}$) repeatedly.

Indeed, the difficulty of learning under an \emph{unknown} delay budget is well-documented in delayed  bandits~\cite{bistritz2019online} and further refined in subsequent work~\cite{bistritz2022no}, where the typical adaptation cost scales on the order of $\sqrt{TK + D\log K}$. Furthermore, as analyzed by \citet{besson2018doubling}, simple ``oblivious'' geometric restarts generally fail to preserve instance-dependent guarantees because periodic resets destroy the accumulated statistical evidence required for fast rates.  In contrast, our coupled schedule
restarts only when warranted by the observed delay structure and safety diagnostics, insuring only a negligible logarithmic overhead. Consequently, delay adaptation becomes essentially ``free'' at the level of the safety bounds.

\medskip
\noindent\textbf{3.\ Generalizing \texttt{Banker-OMD} to Safe Learning.}
Finally, we answer the challenge of extending safety guarantees beyond standard immediate-feedback environments. While \citet{MullerEtAl2025} established the foundations of the (Safe) online linear minimization framework, the feasibility of such guarantees under partial monitoring remained a question. \\We bridge this gap by providing a unified analysis that integrates \texttt{Banker-OMD} with safety constraints. We demonstrate that by decomposing the regret analysis into ``arrived'' and ``pending'' components,
It is possible the potential-function credits to be used for delay tolerance can be \emph{dually leveraged} to certify safety, effectively broadening the applicability of ``best-of-both-worlds'' mechanisms to the asynchronous reality of deployed systems.

\subsection{Further Insights in Banker-OMD} As we discussed in the main draft, \emph{Prudent-Banker} combines two mechanisms---one to \emph{survive delayed bandit feedback} and one to \emph{stay safe relative to a baseline comparator}. Our delayed bandit feedback mechanism follows banker OMD idea. The intuition is to carefully account of how much past information is allowed to influence each decision.

At the beginning of round $t$, the algorithm selects a target action scale $\sigma_t$, which specifies the total amount of update weight to be used when forming the next decision.
Due to delayed feedback, only a subset of past rounds have revealed their losses by time $t$, and therefore only those rounds can contribute update information. In particular, when feedback from round $s < t$ arrives, we can compute the corresponding mirror-descent update:
\[
z_s \;=\; \nabla\bar{\Psi}^*\!\left(\nabla\Psi(x_s) - \frac{1}{\sigma_s}\hat{\ell}_s\right).
\]

To construct the decision, the algorithm draws update weight from previously arrived rounds in a greedy manner.
Each past round $s<t$ maintains a remaining budget $v_s$, representing how much of its feedback can still influence future decisions.
For every arrived round $s$, the algorithm allocates an amount
\[
\sigma_{t,s}=\min\{v_s,\; b_t\},
\]
where $b_t$ denotes the remaining portion of $\sigma_t$ that has not yet been covered.
The allocated amount $\sigma_{t,s}$ is deducted both from $v_s$ and from $b_t$, and this process continues until either the full budget $\sigma_t$ is reached or no further arrived feedback is available.

If the available feedback is insufficient to cover $\sigma_t$, the remaining portion $b_t$ is filled using a fixed default distribution $x_0$.
The resulting decision is defined as a mirror-map average
\[
\tilde{x}_t
=
\nabla \bar{\Psi}^*\!\left(
\frac{1}{\sigma_t}\sum_{s<t}\sigma_{t,s}\nabla\Psi(z_s)
\;+\;
\frac{b_t}{\sigma_t}\nabla\Psi(x_0)
\right),
\]
which combines all available mirror-descent updates with the default one, in proportions determined exactly by the amount of missing feedback.
This ensures that each decision relies only on information that has actually arrived, while reducing to a default when feedback is delayed.

This construction guarantees that each unit of learning-rate budget is used at most once, restores a telescoping structure in the regret analysis, and ensures that the additional regret incurred due to delays depends only on the cumulative missing budget, which can be bounded in terms of the total delay.

\subsection{Proof Roadmap}
\textbf{Stage Partitioning via Doubling Trick.} Since the total delay $D$ is unknown, the algorithm proceeds in stages $s = 1, 2, \dots, S$. We maintain a delay estimator $\widehat{D}_s$, initialized at $\widehat{D}_1 = 1$. At any time step $t$, if the realized cumulative delay exceeds the current estimator, we update $\widehat{D}_{s+1} = 2\widehat{D}_s$ and initiate a new stage. This doubling scheme ensures that the total number of stages is bounded by $S \leq \lceil \log_2 D \rceil + 1$ (Lemma~\ref{lem:delay-estimator}), allowing the algorithm to adapt to the delay scale without prior knowledge.

\textbf{Phased Aggression Framework.} Within each stage $s$, the time horizon is further partitioned into a sequence of phases $k = 1, 2, \dots, K(s)$. In each phase, the algorithm plays a mixed strategy $x_t = \alpha \hat{x}_t + (1-\alpha) x^c$, where $\hat{x}_t$ is the Banker-OMD strategy and $x^c$ is a safe comparator. The aggression parameter $\alpha$ determines the influence of the comparator and is scheduled as:
\begin{equation}
    \alpha = \frac{2^{k-1}}{\widehat{R}_s}
\end{equation}
where $\widehat{R}_s$ denotes the stage-specific regret upper bound. A phase transition $k \to k+1$ is triggered whenever the cumulative loss of the safe comparator exceeds a threshold, signifying ``poor" performance. Upon transitioning, we increase the weight on the OMD learner (increase $\alpha$), effectively "tuning out" the poorly performing comparator.

\textbf{Regret Analysis and Complexity.} This nested architecture facilitates a dual-layered regret guarantee:
\begin{itemize}[leftmargin=2em]
    \item \textbf{Comparator Regret:} By maintaining a phased mixture, the algorithm ensures that the regret relative to the safe comparator remains $O(1)$ within each stage (Lemma~\ref{lem:extinction}, Lemma~\ref{lem:normal}, Lemma~\ref{lem:missing}).
    \item \textbf{Worst-case Regret:} We prove the regret upper bound for Banker-OMD using Lemma~\ref{lem:delay-mass}. Then decompose the regret into Banker-OMD regret (Lemma~\ref{lem:banker}) and regret of comparator relative to strategy $x$. Combining Lemma~\ref{lem:delay-estimator}, each phase incurs a worst-case regret bounded by $4 \widehat R (D^{k(s)}_{\text{end}_{k(s)}}) + O(\sqrt{D^{k(s)}_{\text{end}_{k(s)}}})$ (Lemma~\ref{lem:normal}). Since $\alpha$ grows exponentially, the number of phases per stage is bounded by $K(s) = O(\log \widehat{R}_s)$. 
    \item \textbf{Total Regret:} Summing the regret over $K(s)$ phases and $S$ stages yields a global regret bound that accounts for both the estimation of $D$ and the adaptive weighting of the comparator.
\end{itemize}

\newpage
\section{Deferred Proofs}
\label{app:proof}
\vspace{-1em}
\begin{PBAppLemma}[Bound on Cumulative Outstanding Delay]{number_of_delay_vs_delay_length}{
    Following lemma relates two different measurements of delay.
    The quantity $\mathfrak{d}_t^{k(s)}$ counts how many feedback items within phase $k(s)$ are still missing at time $t$,
    and $\mathfrak{D}_t^{k(s)}=\sum_{u=\operatorname{start}}^t \mathfrak{d}_u^{k(s)}$ is the \emph{cumulative delayed feedback items} up to $t$.
    The lemma shows $\mathfrak{D}_t^{k(s)}$ is monotone in $t$ and
    that $\mathfrak{D}_{\operatorname{end}_{k(s)}}^{k(s)}$ is bounded by the total delay lengths $D_{\operatorname{end}_{k(s)}}^{k(s)}$.
    Intuitively, each round $r$ with delay $d_r$ can contribute at most $d_r$ time-steps of ``being missing,''
    so the total missing-mass over the horizon cannot exceed the sum of delay lengths.
    Since the regret upper bound of \cref{alg:prudent_banker} relies on $\sigma_t$ constructed by $\mathfrak{D}_t^{k(s)}$, we need this lemma to ensure our regret bound is parametrized by $D_{\operatorname{end}_{k(s)}}^{k(s)}$.
    The proof works by double-counting: $\mathfrak{D}_{\operatorname{end}_{k(s)}}^{k(s)}$ counts the number of pairs $(r,u)$ such that the feedback of $r$ is still missing at time $u$,
    which equals $\sum_{r=\operatorname{start}}^{\text{end}_{k(s)}} \min\{d_r,\,\operatorname{end}_{k(s)}-r\}$ and is therefore at most $\sum_{r=\operatorname{start}_{k(s)}}^{\text{end}_{k(s)}} d_r=D_{\operatorname{end}_{k(s)}}^{k(s)}$.
}
    The cumulative number of delayed feedback items up to time $t$, is at most the cumulative delayed feedback items over phase $k(s)$ inside stage $s$, which in turn is bounded by the sum of the delay lengths. Formally, 
    \[
    \mathfrak{D}_t^{k(s)} \leq \mathfrak{D}_{\operatorname{end}_{k(s)}}^{k(s)} \leq D_{\operatorname{end}_{k(s)}}^{k(s)}.
    \]
\end{PBAppLemma}

\begin{proof}
    Since $\mathfrak{D}_t^{k(s)}=\sum_{u=\operatorname{start}_{k(s)}}^t \mathfrak{d}_u^{k(s)}$ with $\mathfrak{d}_u^{k(s)}=\sum_{r=\operatorname{start}_{k(s)}}^{u-1} \mathbf{1}\left\{u \leq r+d_r\right\} \geq 0$, it follows immediately that $\mathfrak{D}_t^{k(s)}$ is non-decreasing in $t$. Thus, the first inequality $\mathfrak{D}_t^{k(s)} \leq \mathfrak{D}_{\text{end}_{k(s)}}^{k(s)}$ holds for all $t \le \text{end}_{k(s)}$.

    For the second inequality, we first show that $\mathfrak{D}_{\text{end}_{k(s)}}^{k(s)}=\sum_{u=\text{start}_{k(s)}}^{\text{end}_{k(s)}} \min \left\{d_u, \text{end}_{k(s)} - u\right\}$.
    By definition of delay complexity, 
    \begin{equation*}
        \mathfrak{D}_{\text{end}_{k(s)}}^{k(s)} = \sum_{u=\operatorname{start}_{k(s)}}^{\operatorname{end}_{k(s)}} \mathfrak{d}_u^{k(s)}=\sum_{u=\operatorname{start}_{k(s)}}^{\operatorname{end}_{k(s)}} \sum_{r=\operatorname{start}_{k(s)}}^{u-1} \mathbf{1}\left\{u \leq r+d_r\right\}.
    \end{equation*}

    We swap the order of summation (counting pairs of $(r,u)$ such that $r < u \le r+d_r$):
    $
        \mathfrak{D}_{\text{end}_{k(s)}}^{k(s)} = \sum_{r=\operatorname{start}_{k(s)}}^{\operatorname{end}_{k(s)}} \sum_{u=r+1}^{\operatorname{end}_{k(s)}} \mathbf{1}\left\{u \leq r+d_r\right\}.
    $
    Fix $s$. We are counting the number of integers $u$ such that $u \in \{r+1, \ldots, \text{end}_{k(s)}\}$ and $u \leq r+d_r$.
    Thus, $u$ ranges from $r+1$ up to $\min \{\text{end}_{k(s)}, r+d_r\}$. The total number of such integers is:
    \begin{align*}
        \max \left(0, \min \left\{\text{end}_{k(s)}, r+d_r\right\} - (r+1) + 1\right) 
        &= \max \left(0, \min \left\{\text{end}_{k(s)}, r+d_r\right\} - r\right) \\
        &= \min \left\{d_r, \text{end}_{k(s)} -r\right\},
    \end{align*}
    where the last step follows because delays are non-negative ($d_r \ge 0$).
    Plugging this back gives:
    \begin{equation*}
        \mathfrak{D}_{\text{end}_{k(s)}}^{k(s)} = \sum_{r=\operatorname{start}_{k(s)}}^{\text{end}_{k(s)}} \min \left\{d_r, \text{end}_{k(s)}-r\right\}.
    \end{equation*}
    Finally, we bound this sum by $D_{\operatorname{end}_{k(s)}}^{k(s)}$:
    \begin{enumerate}
        \item If all feedback arrives by end of phase $k(s)$ (i.e., $r+d_r \leq \text{end}_{k(s)}$ for all $r$), then $\min \{d_r, \text{end}_{k(s)}-r\} = d_r$, and $\mathfrak{D}_{\text{end}_{k(s)}}^{k(s)} = \sum d_r = D_{\text{end}_{k(s)}}^{k(s)}$.
        \item If some feedback arrives after $\text{end}_{k(s)}$ (i.e., $r+d_r > \text{end}_{k(s)}$ for some $r$), then $\min \{d_r, \text{end}_{k(s)}-r\} < d_r$ for those $r$, and thus $\mathfrak{D}_{\text{end}_{k(s)}}^{k(s)} < D_{\text{end}_{k(s)}}^{k(s)}$.
    \end{enumerate}
    In all cases, $\mathfrak{D}_{\text{end}_{k(s)}}^{k(s)} \le D_{\text{end}_{k(s)}}^{(s)}$.
\end{proof}
\newpage

\begin{PBAppLemma}[Bound on Estimated Total Delay]{bound_on_estimated_total_delay}{
    This lemma establishes that the updated budget $\widehat{D}_{s+1}$ provides a tight upper bound on the realized total delay. Specifically, we show that $\widehat{D}_{s+1} < 2 D_{\operatorname{end}_{k(s)}}^{k(s)}$, ensuring that the estimated capacity does not overestimate the actual total delay by more than a factor of two. The analysis proceeds in two steps:
\begin{enumerate*}[label=(\roman*)]
    \item we first demonstrate the monotonicity of the estimator ($\widehat{D}_{s+1} \ge \widehat{D}_s$);
    \item we then derive the upper bound using the doubling update rule defined in \cref{alg:prudent_banker}.
\end{enumerate*}
This inequality is critical for subsequent lemmas, as it allows us to bound the number of phases and the total delay accumulated within them by relating the budget back to the realized total delay.
}
    For every stage $s$, the total-delay capacity satisfies:
    \[
    \widehat{D}_{s} \leq \widehat{D}_{s+1}<2 D_{\operatorname{end}_{k(s)}}^{k(s)}
    \]
\end{PBAppLemma}

\begin{proof}
    $\\$\vspace{-2em}
\begin{itemize}
\item \textbf{Monotonicity ($\widehat D_s \le \widehat D_{s+1}$):}
The inequality holds by construction. In \cref{alg:prudent_banker}, the budget is updated via $\widehat D_{s+1} \gets \widehat D_{s} \cdot 2^m$ for some integer $m \geq 1$ when a reset occurs, or remains constant otherwise. Thus, $\widehat D_{s+1} \ge \widehat D_s$.

\item \textbf{Tightness ($\widehat D_{s+1} < 2 D_{\operatorname{end}_{k(s)}}^{k(s)}$):}
A stage reset is triggered only when $D_{\operatorname{end}_{k(s)}}^{k(s)} > \widehat D_s$. The algorithm finds the smallest integer $m \in \mathbb{N}_{\ge 1}$ such that the new budget covers the realized regret:
\[
2^m \widehat D_s \ge D_{\operatorname{end}_{k(s)}}^{k(s)}.
\]
The update rule sets $\widehat D_{s+1} = 2^m \widehat D_s$. Due to the minimality of $m$, the exponent $m-1$ would fail to cover the regret, implying:
\begin{equation}
    \label{eq:doubling_bounds}
    2^{m-1} \widehat D_s < D_{\operatorname{end}_{k(s)}}^{k(s)} \le 2^m \widehat D_s = \widehat D_{s+1}.
\end{equation}
To obtain the upper bound, we rearrange the left-hand side of \eqref{eq:doubling_bounds}:
\[
\frac{1}{2} (2^m \widehat D_s) < D_{\operatorname{end}_{k(s)}}^{k(s)} \implies \frac{1}{2} \widehat D_{s+1} < D_{\operatorname{end}_{k(s)}}^{k(s)} \implies \widehat D_{s+1} < 2 D_{\operatorname{end}_{k(s)}}^{k(s)}.
\]
Combining the lower and upper bounds, we conclude that $D_{\operatorname{end}_{k(s)}}^{k(s)} \le \widehat D_{s+1} < 2 D_{\operatorname{end}_{k(s)}}^{k(s)}$.
\end{itemize}
\end{proof}
\newpage

\begin{PBAppLemma}[Regret Guarantee for Banker-OMD]{banker_regret_bound}{
    This lemma establishes the regret guarantee for the Banker-OMD strategy running with action scales $\sigma_t$, achieving a bound of roughly $\tilde{O}(\sqrt{T} + \sqrt{D})$.
    We need this because \cref{alg:prudent_banker} is built by mixing the base Banker-OMD iterate $\hat{x}_t$ with the safe comparator $x^c$.
    The proof follows the standard Banker-OMD decomposition: the regret is bounded by the sum of total \emph{borrowing} $B_{\text{end}_{k(s)}}$ and total \emph{immediate cost} $\sum_t \sigma_t^{-1}$.
    To control $B_{\text{end}_{k(s)}}$, we consider the last round $T_0$ when borrowing occurred; the amount borrowed is the current bill plus bills from rounds with outstanding feedback.
    To control $\sum_t \sigma_t^{-1}$, we use the specific structure of $\sigma_t$ to separate a $1/\sqrt{t}$ term and a delay-dependent term, applying the inequality $\sum_{t=1}^T \frac{x_t}{\sqrt{\sum_{s\le t}x_s}}\le 2\sqrt{\sum_t x_t}$ along with the bound $\mathfrak{D}_{\text{end}_{k(s)}}^{k(s)}\le D_{\operatorname{end}_{k(s)}}^{k(s)}$ from Lemma~\ref{pblem:number_of_delay_vs_delay_length}.
}
    Let $\tilde{t} = t - \text{start}_{k(s)} + 1$. 
    
    Setting $\sigma_{t} = \left(\frac{1}{\sqrt{\tilde{t}}} + \mathfrak{d}_{t}^{k(s)} \sqrt{\frac{\ln \left(\mathfrak{D}_t^{k(s)}+1\right)}{\mathfrak{D}_t^{k(s)}}}\right)^{-1}$, we have:
    \[
        \max _{x \in \Delta_A} \mathbb{E}\left[\sum_{t=\mathrm{start}_{k(s)}}^{\mathrm{end}_{k(s)}} \langle \hat{\ell}_t, \hat{x}_t - x \rangle\right]  \leq (C_1 + 2 C_2) \sqrt{\tilde{\text{end}_{k(s)}}} + (3C_1 + 2 C_2) \sqrt{ D_{\operatorname{end}_{k(s)}}^{k(s)} \ln (D_{\operatorname{end}_{k(s)}}^{k(s)}+1)}.
    \]
    Setting $\sigma_t = \sqrt{\frac{C_2}{C_1}} \left(\frac{1}{\sqrt{\tilde{t}}} + \mathfrak{d}_{t}^{k(s)} \sqrt{\frac{\ln \left(\mathfrak{D}_t^{k(s)}+1\right)}{\mathfrak{D}_t^{k(s)}}}\right)^{-1}$, we have:
    \[
        \max _{x \in \Delta_A} \mathbb{E}\left[\sum_{t=\mathrm{start}_{k(s)}}^{\mathrm{end}_{k(s)}} \langle \hat{\ell}_t, \hat{x}_t - x \rangle\right] \leq \sqrt{C_1 C_2} \left( 3\sqrt{\tilde{\text{end}_{k(s)}}} + 7 \sqrt{2 D_{\operatorname{end}_{k(s)}}^{k(s)} \ln (D_{\operatorname{end}_{k(s)}}^{k(s)}+1)} \right).
    \]
\end{PBAppLemma}

\begin{proof}
    Since $\hat{\ell}_t$ is an unbiased estimator, $\mathbb{E}\left[\sum_{t=\mathrm{start}_{k(s)}}^{\mathrm{end}_{k(s)}} \langle \hat{\ell}_t, \hat{x}_t - x \rangle\right]  = \mathbb{E}\left[\sum_{t=\mathrm{start}_{k(s)}}^{\mathrm{end}_{k(s)}} \langle \ell_t, \hat{x}_t - x \rangle\right]$. Taking the maximum over $x \in \Delta_A$:
    \[
        \mathcal{R}_{k(s)} \coloneqq \max _{x \in \Delta_A}\mathbb{E}\left[\sum_{t=\mathrm{start}_{k(s)}}^{\mathrm{end}_{k(s)}} \langle \hat{\ell}_t, \hat{x}_t - x \rangle\right] = \max _{x \in \Delta_A} \mathbb{E}\left[\sum_{t=\mathrm{start}_{k(s)}}^{\mathrm{end}_{k(s)}} \langle \ell_t, \hat{x}_t - x \rangle\right].
    \]
     Since $(\Psi, x_0)$ is $(C_1, C_2)$-regular, then we can apply Corollary 4.3 of \citet{huang2023banker}:
    \begin{equation}
        \label{eq:regret_decomp}
        \mathcal{R}_{k(s)} \leq C_1 \cdot \mathbb{E}\left[B_{k(s)}\right] + C_2 \cdot \mathbb{E}\left[\sum_{t=\mathrm{start}_{k(s)}}^{\mathrm{end}_{k(s)}} \sigma_t^{-1}\right].
    \end{equation}

    \textbf{Step 1: Bounding the Borrowing Term $B_{k(s)}$.}
    Consider the last round $T_0$ when $B_t$ increases  within phase $k(s)$. Define the phase-local index $\tilde{T}_0 = T_0 - \text{start}_{k(s)} + 1$. Assume there are $m=\mathfrak{d}_{T_0}^{k(s)}$ feedback items missing at the beginning of $T_0$, corresponding to rounds $t_1 < \dots < t_m$ within phase $k(s)$. Define phase-local indices $\tilde{t}_i=t_i- \text{start}_{k(s)} +1 $.
    Using the definition $\sigma_{t}^{-1}=\frac{1}{\sqrt{\tilde{t}}}+\mathfrak{d}_t^{k(s)} \sqrt{\frac{\ln \left(\mathfrak{D}_t^{k(s)}+1\right)}{\mathfrak{D}_t^{k(s)}}}$, we have $\sigma_{t}^{-1} \geq \frac{1}{\sqrt{\tilde{t}}}$ and $\sigma_{t}^{-1} \geq \mathfrak{d}_t^{k(s)} \sqrt{\frac{\ln \left(\mathfrak{D}_t^{k(s)}+1\right)}{\mathfrak{D}_t^{k(s)}}}$.
    By Lemma 4.5 of \citet{huang2023banker}, we have:
    \[
        B_{\text{end}_{k(s)}} = B_{T_0}=\sigma_{T_0}+\sum_{t=\operatorname{start}_{k(s)}}^{T_0-1} \mathbf{1}\left\{t+d_t \geq T_0\right\} \sigma_{t}.
    \]
    Since there are $m=\mathfrak{d}_{T_0}^{k(s)}$ feedback items pending, the sum covers exactly the indices $t_i$. Thus:
    \[
        B_{\text{end}_{k(s)}} = \sigma_{T_0} + \sum_{i=1}^m \sigma_{t_i}.
    \]
    Using the bound $\sigma_t \leq \frac{1}{\mathfrak{d}_t^{k(s)}} \sqrt{\frac{\mathfrak{D}_t^{k(s)}}{\ln (\mathfrak{D}_t^{k(s)}+ 1)}}$ and noting that $x \mapsto \sqrt{\frac{x}{\ln(x+1)}}$ is increasing:
    \begin{align*}
    B_{\operatorname{end}_{k(s)}} 
    & \leq \frac{1}{\mathfrak{d}_{T_0}^{k(s)}} \sqrt{\frac{\mathfrak{D}_{T_0}^{k(s)}}{\ln \left(\mathfrak{D}_{T_0}^{k(s)}+1\right)}}+\sigma_{t_1}+\sum_{i=2}^m \frac{1}{\mathfrak{d}_{t_i}^{k(s)}} \sqrt{\frac{\mathfrak{D}_{t_i}^{k(s)}}{\ln \left(\mathfrak{D}_{t_i}^{k(s)}+1\right)}} \\ 
    & \leq \frac{1}{m} \sqrt{\frac{D_{\operatorname{end}_{k(s)}}^{k(s)}}{\ln \left(D_{\operatorname{end}_{k(s)}}^{k(s)}+1\right)}}+\sqrt{\tilde{t}_1}+\sqrt{\frac{D_{\operatorname{end}_{k(s)}}^{k(s)}}{\ln \left(D_{\operatorname{end}_{k(s)}}^{k(s)}+1\right)}} \sum_{i=2}^m \frac{1}{\mathfrak{d}_{t_i}^{k(s)}} .
    \end{align*}
    Since $\tilde{t}_1 \le \tilde{\text{end}}_{k(s)} $ and at time $t_i$ there are at least $i-1$ other missing feedback items (from $t_1, \dots, t_{i-1}$), we have $\mathfrak{d}_{t_i}^{k(s)} \ge i-1$. Thus:
    \begin{align*}
        B_{\text{end}_{k(s)}} 
        &\leq \sqrt{\tilde{\text{end}}_{k(s)}} + \sqrt{\frac{D_{\operatorname{end}_{k(s)}}^{k(s)}}{\ln (D_{\operatorname{end}_{k(s)}}^{k(s)} + 1)}} \left( \frac{1}{m} + \sum_{i=2}^m \frac{1}{i-1} \right) \\
        &= \sqrt{\tilde{\text{end}}_{k(s)}} + \sqrt{\frac{D_{\operatorname{end}_{k(s)}}^{k(s)}}{\ln (D_{\operatorname{end}_{k(s)}}^{k(s)} + 1)}} \left( \frac{1}{m} + \sum_{j=1}^{m-1} \frac{1}{j} \right).
    \end{align*}
    Using the harmonic series bound $\sum_{j=1}^{k} \frac{1}{j} \le \ln k + \gamma + \frac{1}{2k}$ (where $\gamma \approx 0.577$):
    \begin{align*}
        B_{\text{end}_{k(s)}} &\leq \sqrt{\tilde{\text{end}_{k(s)}}} + \sqrt{\frac{D_{\operatorname{end}_{k(s)}}^{k(s)}}{\ln (D_{\operatorname{end}_{k(s)}}^{k(s)} + 1)}} \left( \frac{1}{m} + \ln(m-1) + \gamma + \frac{1}{2(m-1)} \right) \\
        &\leq \sqrt{\tilde{\text{end}_{k(s)}}} + \sqrt{\frac{D_{\operatorname{end}_{k(s)}}^{k(s)}}{\ln (D_{\operatorname{end}_{k(s)}}^{k(s)} + 1)}} (\ln m + 2).
    \end{align*}
    Using the property that if total delay within phase $k(s)$ is less than $D_{\operatorname{end}_{k(s)}}^{k(s)}$, the maximum concurrency $m$ satisfies roughly $m^2/2 \le D_{\operatorname{end}_{k(s)}}^{k(s)} \implies m \le \sqrt{2D_{\operatorname{end}_{k(s)}}^{k(s)}+1/4} + 1/2$, we substitute back:
    \begin{equation}
        \label{eq:B_T_bound}
        B_{\text{end}_{k(s)}} \leq \sqrt{\tilde{\text{end}_{k(s)}}} + \frac{1}{2} \sqrt{\frac{D_{\operatorname{end}_{k(s)}}^{k(s)}}{\ln (D_{\operatorname{end}_{k(s)}}^{k(s)}+1)}} \ln \left(2D_{\operatorname{end}_{k(s)}}^{k(s)} + \frac{1}{4}\right) + 2\sqrt{\frac{D_{\operatorname{end}_{k(s)}}^{k(s)}}{\ln (D_{\operatorname{end}_{k(s)}}^{k(s)}+1)}}.
    \end{equation}

    \textbf{Step 2: Bounding the Immediate Cost $\sum \sigma_t^{-1}$.}
    \begin{align}
        \label{eq:sigma_sum}
        \sum_{t=\text{start}_{k(s)}}^{\text{end}_{k(s)}} \sigma_t^{-1}
        & = \sum_{t=\text{start}_{k(s)}}^{\text{end}_{k(s)}} \left(\frac{1}{\sqrt{\tilde{t}}}+\mathfrak{d}_t^{k(s)} \sqrt{\frac{\ln \left(\mathfrak{D}_t^{k(s)} + 1\right)}{\mathfrak{D}_t^{k(s)}}}\right) \nonumber \\
        & \leq \sum_{t=\text{start}_{k(s)}}^{\text{end}_{k(s)}} \frac{1}{\sqrt{\tilde{t}}} + \sqrt{\ln (D_{\operatorname{end}_{k(s)}}^{k(s)}+1)} \sum_{t=\text{start}_{k(s)}}^{\text{end}_{k(s)}} \frac{\mathfrak{d}_t^{k(s)}}{\sqrt{\mathfrak{D}_t^{k(s)}}} \nonumber \\
        & \leq 2\sqrt{\tilde{\text{end}_{k(s)}}} + \sqrt{\ln (D_{\operatorname{end}_{k(s)}}^{k(s)}+1)} \cdot 2\sqrt{\mathfrak{D}_{\text{end}_{k(s)}}^{k(s)}} \nonumber \\
        & \leq 2\sqrt{\tilde{\text{end}_{k(s)}}} + \sqrt{2 D_{\operatorname{end}_{k(s)}}^{k(s)} \ln (D_{\operatorname{end}_{k(s)}}^{k(s)}+1)}.
    \end{align}
    Here we used $\sum_{t=1}^T \frac{x_t}{\sqrt{\sum_{s=1}^t x_s}} \leq 2 \sqrt{\sum_{t=1}^T x_t}$ and $\mathfrak{D}_{\operatorname{end}_{k(s)}}^{k(s)} \leq D_{\operatorname{end}_{k(s)}}^{k(s)}$(Lemma~\ref{pblem:number_of_delay_vs_delay_length}).

    \textbf{Step 3: Combining Terms.}
    Plugging \eqref{eq:B_T_bound} and \eqref{eq:sigma_sum} into \eqref{eq:regret_decomp}:
    \[
        \mathcal{R}_{k(s)} \leq (C_1 + 2 C_2) \sqrt{\tilde{\text{end}_{k(s)}}} + (3C_1 + 2 C_2) \sqrt{ D_{\operatorname{end}_{k(s)}}^{k(s)} \ln (D_{\operatorname{end}_{k(s)}}^{k(s)}+1)}.
    \]
    Finally, for the tuned learning rate $\sigma_t \leftarrow \sqrt{\frac{C_2}{C_1}} \sigma_t$, the terms scale as $C_1 \sqrt{\frac{C_2}{C_1}} B_{\text{end}_{k(s)}}$ and $C_2 \sqrt{\frac{C_1}{C_2}} \sum \sigma_t^{-1}$, effectively multiplying the final bound by $\sqrt{C_1 C_2}$ (and adjusting constants).
    Specifically:
    \begin{align*}
        \mathcal{R}_{k(s)} &\leq \sqrt{C_1 C_2} \left[ \left(\sqrt{\tilde{\text{end}_{k(s)}}} + \frac{1}{2}\sqrt{\dots}\ln(\dots) + 2\sqrt{\dots}\right) + \left(2\sqrt{\tilde{\text{end}_{k(s)}}} + \sqrt{2D\ln(D+1)}\right) \right] \\
        &\leq \sqrt{C_1 C_2} \left( 3\sqrt{\tilde{\text{end}_{k(s)}}} + 7 \sqrt{2 D_{\operatorname{end}_{k(s)}}^{(s)} \ln (D_{\operatorname{end}_{k(s)}}^{(s)}+1)} \right).
    \end{align*}
    The constant $7$ arises from bounding the terms involving $\ln(2D_{\operatorname{end}_{k(s)}}^{k(s)})$ and constants for $D_{\operatorname{end}_{k(s)}}^{k(s)} \ge 1$.
\end{proof}

\begin{PBAppLemma}[Bound on Missing Feedback Contribution]{missing_feedback_contribution}{
    Following lemma upper bounds the contribution of \emph{missing} feedback items by time $T_0$:
    it controls $\sum_{u \in \mathcal{M}(T_0)} \langle \hat{\ell}_u, x^c-x\rangle$ where $\mathcal{M}(T_0)$ is the set of rounds whose feedback has not arrived by $T_0$.
    Intuitively, missing feedback is dangerous because it prevents the phase-gap statistic from ``seeing'' some losses,
    so we must separately bound how bad the unseen (missing) terms can be.
    We need this lemma in both phase analyses (normal and extinction): it is the only place where we pay explicitly for delayed/missing information,
    and it is exactly why the restart threshold includes the additive slack $\hat{\xi}_s$.
    The proof has two parts.
    First we bound the \emph{number} of missing items: if feedback from $s$ is still missing at $T_0$ then necessarily $d_s\ge T_0-s$,
    so summing these lower bounds over the missing indices yields a triangular-number constraint
    $D_{\operatorname{end}_{k(s)}}^{k(s)} \ge 1+2+\cdots+m=m(m+1)/2$ and hence $m\lesssim \sqrt{D_{\operatorname{end}_{k(s)}}^{k(s)}}$.
    Second we bound the \emph{per-item} impact of a missing estimator using the mixture structure $x_t=\alpha\hat{x}_t+(1-\alpha)x^c$.
}
    Fix a phase $k(s)$ of \cref{alg:prudent_banker} starting at $\mathrm{start}_{k(s)}$ ($\operatorname{start}_{k(s)}$ in the algorihtm).
    Let $T_0 := \mathrm{start}_{k(s)+1}-2$ (the last time index at which the algorithm checks the phase condition before the phase ends).
    Let
    \[
    \mathcal{M}(T_0)
    \;\coloneqq\;
    \{\, u \in \{\operatorname{start}_{k(s)},\ldots,T_0-1\} : \mathrm{status}_u(T_0)=\textsc{missing}\,\},
    \qquad
    m \;\coloneqq\; |\mathcal{M}(T_0)|
    \]
    be the set and number of rounds within phase $k(s)$ whose feedback has not arrived by time $T_0$ .
    Then for any $x\in\Delta_A$:
    \[
        \sum_{u\in\mathcal{M}(T_0)} \langle \hat{\ell}_u,\, x^c-x\rangle
        \;\le\;
        2\sqrt{2D_{\operatorname{end}_{k(s)}}^{k(s)}+\tfrac14}-1 .
    \]
\end{PBAppLemma}

\begin{proof}
    \textbf{Step 1: Bounding the number of missing feedbacks.}
    For any $r\in\mathcal{M}(T_0)$, the feedback for round $u$ has not arrived by time $T_0$, so its delay satisfies $d_u \ge T_0-u$.
    Summing over all missing items:
    \[
        D_{\operatorname{end}_{k(s)}}^{k(s)}
        \;=\;
        \sum_{t=\text{start}_k(s)}^{\text{end}_{k(s)}} d_t
        \;\ge\;
        \sum_{u\in\mathcal{M}(T_0)} d_u
        \;\ge\;
        \sum_{u\in\mathcal{M}(T_0)} (T_0-u).
    \]

    Let the missing indices be ordered as $u_1 < u_2 < \cdots < u_m.$ Since the latest possible missing round is $T_0-1$ (feedback is checked at end of round), we have:
    \begin{align*}
        T_0 - u_m &\ge 1, \\
        T_0 - u_{m-1} &\ge 2, \\
        &\vdots \\
        T_0 - u_1 &\ge m.
    \end{align*}
    Summing these lower bounds yields a triangular number:
    \[
        \sum_{i=1}^m (T_0-u_i) \;\ge\; \sum_{j=1}^m j \;=\; \frac{m(m+1)}{2}.
    \]
    Combining this with the above lower bound on
    $D_{\operatorname{end}_{k(s)}}^{k(s)}$,
    we obtain
    \[
        \frac{m(m+1)}{2}
        \;\le\;
        D_{\operatorname{end}_{k(s)}}^{k(s)}.
    \]
    Solving the quadratic for $m$:
    \[
        m^2+m-2D_{\operatorname{end}_{k(s)}}^{k(s)} \le 0
        \quad\Longrightarrow\quad
        m
        \;\le\;
        \frac{-1+\sqrt{1+8D_{\operatorname{end}_{k(s)}}^{k(s)}}}{2}
        \;=\;
        \sqrt{2D_{\operatorname{end}_{k(s)}}^{k(s)}+\tfrac14}-\tfrac12.
    \]
    \textbf{Step 2: Bounding the per-round inner product.}
    When the feedback for round $u$ arrives, the algorithm constructs the importance-weighted estimator $\hat{\ell}_u(a) = \frac{\ell_u(a)}{x_u(a)} \mathbf{1}\{A_u=a\}$.
    We analyze the inner product $\langle \hat{\ell}_u, x^c-x\rangle$:
    \[
        \langle \hat{\ell}_u, x^c-x\rangle = \langle \hat{\ell}_u, x^c\rangle - \underbrace{\langle \hat{\ell}_u, x\rangle}_{\ge 0} \le \langle \hat{\ell}_u, x^c\rangle.
    \]
    Evaluating the term $\langle \hat{\ell}_u, x^c\rangle$:
    \[
        \langle \hat{\ell}_u, x^c\rangle = \sum_{a} \frac{\ell_u(a)\mathbf{1}\{A_u=a\}}{x_u(A_u)} x^c(a) = \frac{\ell_u(A_u) x^c(A_u)}{x_u(A_u)}.
    \]
    In any active phase, the played policy is a mixture $x_u = \alpha \hat{x}_u + (1-\alpha) x^c$.
    WLOG suppose that $\hat{R}=2^r$ is a power of 2. Since $\alpha=\frac{2^{k-1}}{\hat{R}}<1$ implies $k<\log _2 \hat{R}+1=r+1$, then $\alpha \leq 2^{r-1} / 2^r=1 / 2$. We have $1-\alpha \ge 1/2$, and thus:
    \[
        x_u(A_u) \ge (1-\alpha)x^c(A_u) \ge \tfrac12 x^c(A_u).
    \]
    Using $\ell_u(A_u) \le 1$, we get:
    \[
        \langle \hat{\ell}_u, x^c\rangle \le \frac{1 \cdot x^c(A_u)}{\tfrac12 x^c(A_u)} = 2.
    \]
    Therefore, for every missing round $u$, $\langle \hat{\ell}_u, x^c-x\rangle \le 2$.

    \textbf{Conclusion.}
    Summing over the set of missing rounds $\mathcal{M}(T_0)$:
    \[
        \sum_{u\in\mathcal{M}(T_0)} \langle \hat{\ell}_u, x^c-x\rangle
        \le \sum_{u\in\mathcal{M}(T_0)} 2
        = 2m
        \le 2\left(\sqrt{2D_{\operatorname{end}_{k(s)}}^{k(s)}+\tfrac14}-\tfrac12\right)
        = 2\sqrt{2D_{\operatorname{end}_{k(s)}}^{k(s)}+\tfrac14}-1.
    \]
\end{proof}

\begin{PBAppLemma}[Regret Bound for Normal Phase]{normal_phase_bound}{
    In a normal phase, the restart condition never triggers:
    the phase-gap statistic does not exceed the threshold $2 \widehat R_s + \widehat \xi_s$ throughout the phase.
    This lemma shows two things:
    1. The regret against any strategy $x$ is bounded by a controlled quantity (roughly $2 \widehat R_s + \widehat \xi_s$ plus a bounded missing-feedback contribution).
    2. The regret against the safe comparator $x^c$ is small (at most $2^{k-1}$).
    We need this because the algorithm either keeps running in a normal phase, or it restarts due to a substantially suboptimal comparator; this lemma handles the former.
    Proof intuition: based on $x_t=\alpha_k \hat{x}_t+(1-\alpha_k)x^c$, we decompose regret into (i) regret of the base Banker-OMD sequence $\hat{x}_t$ (bounded by Lemma~\ref{pblem:banker_regret_bound}) and (ii) a comparator-gap term.
    For the comparator-gap term, we split into ``arrived'' feedback (controlled by the phase-gap definition since restart didn't trigger) and ``missing'' feedback (controlled by Lemma~\ref{pblem:missing_feedback_contribution}).
}
    Consider a phase $k(s)$ of \cref{alg:prudent_banker} that starts at index $\mathrm{start}_{k(s)}$ and ends at $\mathrm{end}_{k(s)} := \mathrm{start}_{k(s)+1}-1$ (if no new phase is triggered, $\mathrm{end}_{k(s)} = T$).
    Let $\alpha_{k(s)} < 1$ denote the value of $\alpha$ during this phase, and recall that for all $t$ in this phase the played distribution is $x_t = \alpha_{k(s)} \hat{x}_t + (1-\alpha_{k(s)})x^c$.
    Let $\widehat\xi_s := \frac{\sqrt{8 \widehat D_s+1}-1}{\delta}$.
    Then for any comparator $x\in\Delta_A$:
    \[
        \mathbb{E}\!\left[\sum_{t=\mathrm{start}_{k(s)}}^{\mathrm{end}_{k(s)}}
        \langle \hat{\ell}_t, x_t - x\rangle\right]
        \;\le\;
        4 \widehat R (D^{k(s)}_{\text{end}_{k(s)}}) + \frac{\sqrt{16 D_{\operatorname{end}_{k(s)}}^{k(s)}+1}-1}{\delta} + 2 \sqrt{2D_{\operatorname{end}_{k(s)}}^{k(s)}+\tfrac14} + 1.
    \]
    Moreover, for the comparator $x^c$:
    \[
        \mathbb{E}\!\left[\sum_{t=\mathrm{start}_{k(s)}}^{\mathrm{end}_{k(s)}}
        \langle \hat{\ell}_t, x_t - x^c\rangle\right] \le 2^{k(s)-1}.
    \]
\end{PBAppLemma}

\begin{proof}
Fix a phase $k(s)$ inside stage $s$. For simplicity of notation, we use $k$ instead.

    Fix a phase $k$. Since $x_t = \alpha_k \hat{x}_t + (1-\alpha_k)x^c$ throughout the phase, we can decompose the regret for any $x\in\Delta_A$:
    \begin{align*}
        \sum_{t=\mathrm{start}_k}^{\mathrm{end}_k} \langle \hat{\ell}_t, x_t-x\rangle
        &= \sum_{t=\mathrm{start}_k}^{\mathrm{end}_k} \left\langle \hat{\ell}_t, \alpha_k(\hat{x}_t-x) + (1-\alpha_k)(x^c-x)\right\rangle\\
        &= \alpha_k \sum_{t=\mathrm{start}_k}^{\mathrm{end}_k} \langle \hat{\ell}_t, \hat{x}_t-x\rangle
        \;+\;
        (1-\alpha_k)\sum_{t=\mathrm{start}_k}^{\mathrm{end}_k} \langle \hat{\ell}_t, x^c-x\rangle .
    \end{align*}

    \textbf{Step 1: Bounding the Base Regret.}
    For all phases in stage $s=1$ and all phases excluding the first one in stage $s \geq 2$, Lemma~\ref{pblem:banker_regret_bound} applied to the Banker-OMD base sequence $\{\hat{x}_t\}$ indicates $\mathbb{E}[\sum_{t=\mathrm{start}_k}^{\mathrm{end}_k} \langle \hat{\ell}_t, \hat{x}_t-x\rangle] \le \widehat R (D_{\operatorname{end}_{k(s)}}^{k(s)})$. For first phases in each stage $s \geq 2$,  $\mathbb{E}[\sum_{t=\mathrm{start}_k}^{\mathrm{end}_k} \langle \hat{\ell}_t, \hat{x}_t-x\rangle] \le \widehat R (D_{\operatorname{end}_{k(s)}}^{k(s)}) + 1$. This is because: $\mathbb{E} [\langle \hat{\ell}_{\text{start}_{1}}, \hat{x}_{\text{start}_{1}} - x \rangle] \leq 1$, then we initialize $\hat{\ell}_{\text{start}_{2}} = 1/A$ and continue the phases. Since $\widehat R (\cdot)$ is monotone increasing and $D_{\operatorname{end}_{k(s)}}^{k(s)} \leq \widehat D_s$ before exiting phase. Thus:
    \begin{equation} \label{eq:base_term}
        \alpha_k\,\mathbb{E}\!\left[\sum_{t=\mathrm{start}_k}^{\mathrm{end}_k} \langle \hat{\ell}_t, \hat{x}_t-x\rangle\right]\le \alpha_k (\widehat R (\widehat D_s) + 1) = \alpha_k (\widehat R_s +1 ).
    \end{equation}

    \textbf{Step 2: Bounding the Comparator Gap (Arrived Feedback).}
    Recall the definition of the cumulative gap statistics:
    \[
        g_t := \sum_{j=\mathrm{start}_k}^{t} \hat{\ell}_j\,\mathbf{1}\{\mathrm{status}_j(t)=\textsc{arrived}\},
        \qquad
        \mathrm{phase\_gap}_t := \max_{u\in\Delta_A}\langle g_t, x^c-u\rangle.
    \]
    Since the phase does \emph{not} restart before time $\mathrm{end}_k$, the restart condition must have failed for all check-points $t \in \{\mathrm{start}_k, \dots, \mathrm{end}_k - 1\}$. Therefore:
    \[
        \mathrm{phase\_gap}_t \le 2\widehat R_s+\widehat\xi_s \qquad \text{for all } t \le \mathrm{end}_k - 1.
    \]
    Evaluating this at the last check-point $t = \mathrm{end}_k - 1$ for our specific comparator $x$:
    \begin{equation} \label{eq:arrived_bound}
        \sum_{j=\mathrm{start}_k}^{\mathrm{end}_k - 1}
        \langle \hat{\ell}_j\mathbf{1}_{\{\mathrm{status}_j(\mathrm{end}_k-1)=\textsc{arrived}\}}, x^c-x\rangle
        \le \mathrm{phase\_gap}_{\mathrm{end}_k-1}
        \le 2\widehat R_s+\widehat\xi_s.
    \end{equation}

    \textbf{Step 3: Bounding Missing and Last-Round Terms.}
    We decompose the total comparator gap sum as:
    \begin{align*}
        \sum_{t=\mathrm{start}_k}^{\mathrm{end}_k} \langle \hat{\ell}_t, x^c-x\rangle
        &= \underbrace{\sum_{t=\mathrm{start}_k}^{\mathrm{end}_k - 1} \langle \hat{\ell}_t\mathbf{1}_{\{\mathrm{arrived}\}}, x^c-x\rangle}_{\text{Bounded by } 2\widehat R_s+\widehat\xi_s \text{ (Eq. \ref{eq:arrived_bound})}}
        + \underbrace{\sum_{t=\mathrm{start}_k}^{\mathrm{end}_k - 1} \langle \hat{\ell}_t\mathbf{1}_{\{\mathrm{missing}\}}, x^c-x\rangle}_{\text{Missing Feedback}}
        + \underbrace{\langle \hat{\ell}_{\mathrm{end}_k}, x^c-x\rangle}_{\text{Last Round}}.
    \end{align*}
    For the \textbf{Missing Feedback} term, we apply Lemma~\ref{pblem:missing_feedback_contribution}:
    \[
        \sum_{t=\mathrm{start}_k}^{\mathrm{end}_k-1} \langle \hat{\ell}_t\mathbf{1}_{\{\mathrm{status}_t=\textsc{missing}\}}, x^c-x\rangle
        \le 2\sqrt{D_{\operatorname{end}_{k(s)}}^{k(s)}}+\tfrac14 - 1.
    \]
    For the \textbf{Last Round} term, note that $\hat{\ell}_u = \frac{\ell_u(A_u)}{x_u(A_u)}e_{A_u}$. Thus:
    \[
        \langle \hat{\ell}_{\mathrm{end}_k}, x^c-x\rangle \le \langle \hat{\ell}_{\mathrm{end}_k}, x^c\rangle
        = \frac{\ell_{\mathrm{end}_k}(A_{\mathrm{end}_k}) x^c(A_{\mathrm{end}_k})}{x_{\mathrm{end}_k}(A_{\mathrm{end}_k})}
        \le \frac{x^c(A_{\mathrm{end}_k})}{(1-\alpha_k)x^c(A_{\mathrm{end}_k})}
        = \frac{1}{1-\alpha_k}.
    \]
    In a normal phase, $\alpha_k \le 1/2$ (since $\alpha$ doubles but is capped at $1$, and $\alpha_k < 1$ implies the previous doubling didn't exceed the threshold), so $\frac{1}{1-\alpha_k} \le 2$. Thus, the last round contributes at most $2$.

    \textbf{Step 4: Conclusion.}
    Combining all terms and using $(1-\alpha_k) \le 1$:
    \begin{align*}
        \sum_{t=\mathrm{start}_k}^{\mathrm{end}_k} \langle \hat{\ell}_t, x_t-x\rangle
        &\le \alpha_k (\widehat R_s + 1) + (1-\alpha_k)\Big((2\widehat R_s+\widehat\xi_s) + (2 \sqrt{2D_{\operatorname{end}_{k(s)}}^{k(s)} + \tfrac14} - 1) + 2\Big)\\
        &\le \alpha_k (\widehat R_s + 1) + (1-\alpha_k)\Big((4\widehat R (D_{\operatorname{end}_{k(s)}}^{k(s)}) + \frac{\sqrt{8 \widehat D_s+1}-1}{\delta}) + (2 \sqrt{2 D_{\operatorname{end}_{k(s)}}^{k(s)} +\tfrac14} - 1) + 2\Big)\\
        &\le 4 \widehat R (D^{k(s)}_{\text{end}_{k(s)}}) + \frac{\sqrt{16 D_{\operatorname{end}_{k(s)}}^{k(s)}+1}-1}{\delta} + 2 \sqrt{2D_{\operatorname{end}_{k(s)}}^{k(s)}+\tfrac14} + 1.
    \end{align*}
    The second inequality comes from $\widehat R_s = \widehat R (\widehat D_s) \leq \widehat R (\widehat D_{s+1}) \leq \widehat R (2 D_{\operatorname{end}_{k(s)}}^{k(s)}) < 2  \widehat R (D_{\operatorname{end}_{k(s)}}^{k(s)})$ by Lemma~\ref{pblem:bound_on_estimated_total_delay}.
   Taking expectation proves the first claim.
    For the second claim, set $x=x^c$. The comparator gap term vanishes entirely, leaving only the base regret term:
    \[
        \mathbb{E}\!\left[\sum_{t=\mathrm{start}_k}^{\mathrm{end}_k} \langle \hat{\ell}_t, x_t-x^c\rangle\right]
        = \alpha_k\,\mathbb{E}\!\left[\sum_{t=\mathrm{start}_k}^{\mathrm{end}_k} \langle \hat{\ell}_t, \hat{x}_t-x^c\rangle\right]
        \le \alpha_k \widehat R_s \leq 2^{k-1}.
    \]
\end{proof}

\newpage
\begin{PBAppLemma}[Negative Regret in Extinction Phase]{extinction_phase_bound}{
    The following lemma captures the upper bound of comparator regret in extinction phase (the phase that will be exited) is \emph{strictly negative}—at most $-2^{k(s)-1}$.
    We need this lemma for the overall $O(1)$ guarantee against the safe comparator $x^c$:
    each time we restart, we certify that the algorithm has made geometric negative progress relative to $x^c$,
    which will later cancel the positive cost of the final non-restarting phase.
    Proof intuition: rewrite $x_t-x^c=\alpha_{k(s)}(\hat{x}_t-x^c)$ and then insert an optimizer $x^*$ of the phase-gap at the exit time,
    which makes the arrived part contribute at most $-(2 \widehat R_s + \widehat \xi_s)$.
    The base Banker-OMD regret contributes at most $\widehat R_s$, and the missing part is bounded by $O (\sqrt{D_{\operatorname{end}_{k(s)}}^{k(s)}})$.
    Parameter $\widehat \xi_s$ is chosen so that it dominates the missing part, leaving comparator regret at most $- \widehat R_s$,
    and multiplying by $\alpha_{k(s)}=2^{k(s)-1}/ \widehat R_s$ gives the claimed $-2^{k(s)-1}$ bound.
}
    Consider a phase $k(s)$ of \cref{alg:prudent_banker} that starts at index $\mathrm{start}_{k(s)}$ and ends at $\mathrm{end}_{k(s)} := \mathrm{start}_{k(s)+1}-1$ (if no new phase is triggered, take $\mathrm{end}_{k(s)} = T$).
    Then:
    \[
        \mathbb{E}\Bigg[\sum_{t=\mathrm{start}_{k(s)}}^{\mathrm{end}_{k(s)}}\langle \hat{\ell}_t,\, x_t-x^c\rangle\Bigg]\;\le\;-2^{k(s)-1}.
    \]
\end{PBAppLemma}

\begin{proof}
For simplicity of notation, use $k$ to replace $k(s)$ in the following.

    \textbf{Step 1: Decomposition.}
    During phase $k$, the mixing parameter is constant $\alpha := \alpha_k = \frac{2^{k-1}}{\widehat R_s} < 1$, and the played distribution is $x_t = \alpha \hat{x}_t + (1-\alpha)x^c$.
    Hence, $x_t-x^c = \alpha(\hat{x}_t-x^c)$, which allows us to scale the regret:
    \begin{equation}\label{eq:mix_reduce}
        \sum_{t=\mathrm{start}_k}^{\mathrm{end}_k}\langle \hat{\ell}_t, x_t-x^c\rangle
        = \alpha \sum_{t=\mathrm{start}_k}^{\mathrm{end}_k}\langle \hat{\ell}_t, \hat{x}_t-x^c\rangle.
    \end{equation}

    \textbf{Step 2: Exit Condition Analysis.}
    At time $\mathrm{end}_k$, the algorithm computes the cumulative gap statistics using only arrived feedback:
    \[
        g_{\mathrm{end}_k} := \sum_{j=\mathrm{start}_k}^{\mathrm{end}_k} \hat{\ell}_j\,\mathbf{1}\{\mathrm{status}_j(\mathrm{end}_k)=\textsc{arrived}\}.
    \]
    Since $\mathrm{end}_k$ is the last round before restarting, the restart condition must hold:
    \[
        \text{phase\_gap}_{\mathrm{end}_k} = \max_{x \in \Delta_A} \langle g_{\mathrm{end}_k},\, x^c - x \rangle > 2 \widehat R_s + \widehat \xi_s.
    \]
    Let $x^* \in \arg\max_{x\in\Delta_A}\langle g_{\mathrm{end}_k}, x^c-x\rangle$. Then:
    \begin{equation}\label{eq:exit_gap}
        \langle g_{\mathrm{end}_k}, x^c-x^*\rangle > 2 \widehat R_s + \widehat \xi_s
        \quad\Longleftrightarrow\quad
        \langle g_{\mathrm{end}_k}, x^*-x^c\rangle < -(2 \widehat R_s + \widehat \xi_s).
    \end{equation}
    Expanding $g_{\mathrm{end}_k}$, this implies:
    \begin{equation}\label{eq:arrived_negative}
        \sum_{t=\mathrm{start}_k}^{\mathrm{end}_k}\big\langle \hat{\ell}_t\,\mathbf{1}\{\mathrm{status}_t(\mathrm{end}_k)=\textsc{arrived}\},\,x^*-x^c\big\rangle < -(2 \widehat R_s + \widehat \xi_s).
    \end{equation} \newpage
    \textbf{Step 3: Term-by-Term Bounding.}
    We decompose the sum in \eqref{eq:mix_reduce} by introducing $x^*$ and splitting arrived/missing terms:
    \begin{align}
        \sum_{t=\mathrm{start}_k}^{\mathrm{end}_k}\langle \hat{\ell}_t, \hat{x}_t-x^c\rangle
        &=\sum_{t=\mathrm{start}_k}^{\mathrm{end}_k}\langle \hat{\ell}_t, \hat{x}_t-x^*\rangle \nonumber \\
        &\quad +\sum_{t=\mathrm{start}_k}^{\mathrm{end}_k}\big\langle \hat{\ell}_t\,\mathbf{1}\{\mathrm{status}_t(\mathrm{end}_k)=\textsc{arrived}\}, x^*-x^c\big\rangle \nonumber \\
        &\quad +\sum_{t=\mathrm{start}_k}^{\mathrm{end}_k}\big\langle \hat{\ell}_t\,\mathbf{1}\{\mathrm{status}_t(\mathrm{end}_k)=\textsc{missing}\}, x^*-x^c\big\rangle.
        \label{eq:three_terms}
    \end{align}
    Taking expectations, we bound each term:

    \textit{(i) Banker-OMD Regret Term:}
    By Lemma~\ref{pblem:banker_regret_bound}, $\mathbb{E}\Big[\sum \langle \hat{\ell}_t, \hat{x}_t-x^*\rangle\Big]\le \widehat R (D_{\operatorname{end}_{k(s)}}^{(s)}) \leq \widehat R (\widehat{D}_s) = \widehat R_s $.

    \textit{(ii) Arrived Feedback Term:}
    From \eqref{eq:arrived_negative}, the expected value is at most $-(2 \widehat R_s + \widehat \xi_s)$.

    \textit{(iii) Missing Feedback Term:}
    Let $m$ be the number of missing items at $\mathrm{end}_k$.
    Since $\ell_t(A_t) \le 1$ and $x^c(a) \ge \delta$, we bound the inner product:
    \[
        \langle \hat{\ell}_t, x^*-x^c\rangle \le \langle \hat{\ell}_t, x^*\rangle
        = \frac{\ell_t(A_t) x^*(A_t)}{x_t(A_t)}
        \le \frac{1}{x_t(A_t)}.
    \]
    Since $x_t(a) \ge (1-\alpha)x^c(a) \ge (1-\alpha)\delta$ and assuming $\alpha \le 1/2$ (extinction implies we haven't reached full trust yet), $x_t(a) \ge \delta/2$.
    Thus, $\langle \hat{\ell}_t, x^*-x^c\rangle \le 2/\delta$. Summing over $m$ missing items:
    \begin{equation}\label{eq:missing_bound}
        \mathbb{E}\Bigg[\sum_{\text{missing}} \langle \hat{\ell}_t, x^*-x^c\rangle\Bigg] \le \frac{2m}{\delta}.
    \end{equation}

    \textbf{Step 4: Conclusion.}
    Plugging the bounds back into \eqref{eq:three_terms} and multiplying by $\alpha$:
    \begin{align*}
        \mathbb{E}\Bigg[\sum_{t=\mathrm{start}_k}^{\mathrm{end}_k}\langle \hat{\ell}_t, x_t-x^c\rangle\Bigg]
        &\le \alpha \left( \widehat R_s - (2 \widehat R_s + \widehat \xi_s) + \frac{2m}{\delta} \right) \\
        & = -\alpha \widehat R_s - \alpha \left( \widehat \xi_s - \frac{2m}{\delta} \right).
    \end{align*}
    Recall we set $\widehat \xi_s = \frac{\sqrt{8 \widehat D_s +1}-1}{\delta}$.  From Lemma~\ref{pblem:missing_feedback_contribution}, the number of missing items $m$ satisfies $m \le \frac{\sqrt{8 D_{\operatorname{end}_{k(s)}}^{k(s)} + 1}-1}{2}$.
    Since $D^{k(s)}_{\text{end}_k(s)} \leq \widehat D_{s}$, we have $\frac{2m}{\delta} \le \widehat \xi_s$, which implies $\widehat \xi_s - \frac{2m}{\delta} \ge 0$. 
    The term drops out (or helps), leaving:
    \[
        \mathbb{E}\Bigg[\sum_{t=\mathrm{start}_k}^{\mathrm{end}_k}\langle \hat{\ell}_t, x_t-x^c\rangle\Bigg]
        \le -\alpha \widehat R_s.
    \]
    Using $\alpha = 2^{k-1}/\widehat R_s$, we get $-\alpha \widehat R_s = -2^{k-1}$.
\end{proof}
\newpage
\newpage

\begin{PBAppTheorem}[Global Regret and Safety Guarantee]{global_regret_theorem}{
    This theorem combines the per-phase analyses to obtain two global statements:
    (1) standard pseudo-regret against any strategy $x\in\Delta_A$ is at most $\tilde{O}(\sqrt{T}+\sqrt{D})$,
    and (2) the regret against the special safe comparator $x^c$ is $O(1)$.
    Intuitively, the algorithm alternates between cautious behavior (mixing with $x^c$) and increasingly aggressive behavior (doubling $\alpha$);
    either it eventually becomes fully aggressive ($\alpha=1$) and behaves like Banker-OMD,
    or it keeps restarting, but each restart certifies $x^c$ behaves poorly.
    We need this theorem to justify both objectives of phased aggression:
    near-optimal regret scaling with delays, and stability control via the bounded deviation from $x^c$.
    Proof intuition: partition time into stages, then further decompose stage into phases. The maximum number of stages is bounded by the number of times $\widehat D_1$ can be doubled before exceeding $D$, which is $O(\log D)$. Therefore, we only need to focus on the stage-level regret.
    For the worst-case regret bound, apply Lemma~\ref{pblem:normal_phase_bound} to every phase with $\alpha<1$ and Lemma~\ref{pblem:banker_regret_bound} to the final phase if $\alpha$ reaches $1$;
    the number of $\alpha<1$ phases is only $O(\log \widehat R_s)$ since $\alpha$ doubles each restart, which produces the $\tilde{O}$ bound.
    For the $x^c$-comparator guarantee, sum the phase-wise bounds:
    all exiting phases contribute at most $-\sum 2^{k(s)-1}$ by Lemma~\ref{pblem:extinction_phase_bound}, while the final phase contributes at most $+2^{K(s)-1}$ by Lemma~\ref{pblem:normal_phase_bound},
    so the total forms a geometric cancellation that leaves only a constant.
}
    Let $\delta \in(0,1 / A]$. Let $\{x_t\}_{t=1}^T$ be the sequence produced by \cref{alg:prudent_banker}, and let
    $\hat{\ell}_t\in\mathbb{R}^A$ be the importance-weighted estimator defined in the algorithm.
    Then:
    \[
        \max_{x\in\Delta_A}\;
        \mathbb{E}\!\left[\sum_{t=1}^{T}\langle \hat{\ell}_t,\; x_t-x\rangle\right]
        \;\le\;
        \tilde{O} \left(\sqrt{C_1 C_2}(\sqrt{T}+\sqrt{D}) \right),
    \]
    and moreover:
    \[
        \mathbb{E}\!\left[\sum_{t=1}^{T}\langle \hat{\ell}_t,\; x_t-x^c\rangle\right]\le O(\log_2 D).
    \]
\end{PBAppTheorem}

\begin{proof}
    Partition the time horizon into stages indexed by $s = 1, 2, \ldots, S$. Further partition each stage $s$ into phases indexed by $k(s) = 1,2,\ldots,K(s)$. Phase $k(s)$ runs from $\mathrm{start}_{k(s)}$ to $\mathrm{end}_{k(s)}$.
    Let $\alpha_{k(s)}$ denote the constant mixing parameter for phase $k(s)$. Recall that $\alpha_{k(s)}=\min\{2^{k(s)-1}/\widehat R_s,1\}$. 

    By construction, the delay budget $\widehat{D}_s$ doubles after each stage reset (specifically $\hat{D}_{s+1} \ge 2\hat{D}_s$). Consequently, the maximum number of stages $S$ is bounded by the number of times $\hat{D}_1$ can be doubled before exceeding $D$.
Let $S_{max}$ be the smallest integer such that $2^{S_{max}-1} \ge D$. It follows that:
\[
S \leq \lceil \log_2 D \rceil + 1 = O(\log D).
\]
    \textbf{Part 1: Standard Regret Bound (Against any $x$).}
    Let $k_\star (s):=\min\{k(s):\alpha_{k(s)}=1\}$ (set $k_\star(s) = +\infty$ if $\alpha_{k(s)}$ never reaches 1).
    Since $\alpha_{k(s)} < 1$ implies $2^{k(s)-1} < \widehat R_s$, the number of phases with $\alpha_{k(s)}<1$ is at most $\lceil \log_2 \widehat R_s\rceil$.
    For every such phase $k$, Lemma~\ref{pblem:normal_phase_bound} guarantees that for any $x\in\Delta_A$:
    \[
        \mathbb{E}\!\left[\sum_{t=\mathrm{start}_{k(s)}}^{\mathrm{end}_{k_\star(s)}}\langle \hat{\ell}_t,\; x_t-x\rangle\right]
        \le
        4 \widehat R (D^{k(s)}_{\text{end}_{k(s)}}) + \frac{\sqrt{16 D_{\operatorname{end}_{k(s)}}^{k(s)}+1}-1}{\delta} + 2 \sqrt{2D_{\operatorname{end}_{k(s)}}^{k(s)}+\tfrac14} + 1.
    \]
    If a final phase occurs with $\alpha_{k_\star(s)}=1$, the algorithm reduces to Banker-OMD for that duration. By Lemma~\ref{pblem:banker_regret_bound}:
    \[
        \mathbb{E}\!\left[\sum_{t=\mathrm{start}_{k_\star(s)}}^{\text{end}_{k(s)}}\langle \hat{\ell}_t,\; x_t-x\rangle\right]
        \le \widehat R (D^{k(s)}_{\text{end}_{k(s)}}).
    \]
    \newpage
    Summing over all phases within a stage:
    \begin{align*}
        \mathbb{E}\!\left[ \sum_{k(s) =1}^{K(s)} \sum_{t=\mathrm{start}_{k(s)}}^{\mathrm{end}_{k_\star(s)}} \langle \hat{\ell}_t,\; x_t-x\rangle\right]
        &\le
        \lceil\log_2 \widehat R_s \rceil \Big(4 \widehat R (D^{k(s)}_{\text{end}_{k(s)}}) + \frac{\sqrt{16 D_{\operatorname{end}_{k(s)}}^{k(s)}+1}-1}{\delta} + 2 \sqrt{2D_{\operatorname{end}_{k(s)}}^{k(s)}+\tfrac14} + 1\Big)\;\\
        & +\;\mathbf{1}\{k_\star<\infty\}\, \widehat R (D^{k(s)}_{\text{end}_{k(s)}})\\
        & \le \lceil 1 + \log_2 \widehat R (D) \rceil 
        \Big(4 \widehat R (D) + \frac{\sqrt{16 D+1}-1}{\delta} + 2 \sqrt{2D+\tfrac14} + 1\Big)\;\\
        & +\;\mathbf{1}\{k_\star<\infty\}\, \widehat R (D)\\
        &=
        \tilde{O}\!\left(\widehat R(D)+\sqrt{D}\right).
    \end{align*}
    Substituting $\widehat R(D) =\tilde{O} \left(\sqrt{C_1 C_2}(\sqrt{T}+\sqrt{D}) \right)$, we obtain the stated bound $\tilde{O} \left(\sqrt{C_1 C_2}(\sqrt{T}+\sqrt{D}) \right)$.

    Then sum over all stages, 
    \begin{align*}
        \mathbb{E}\!\left[\sum_{s=1}^{S} \sum_{k(s) =1}^{K(s)} \sum_{t=\mathrm{start}_{k(s)}}^{\mathrm{end}_{k_\star(s)}} \langle \hat{\ell}_t,\; x_t-x\rangle\right]
        \le (\lceil \log_2 D \rceil + 1) \tilde{O} \left(\sqrt{C_1 C_2}(\sqrt{T}+\sqrt{D}) \right)
    \end{align*}
    \textbf{Part 2: Safety Guarantee (Regret against $x^c$).}
    We analyze the sum $\sum_{k} \text{Regret}_k(x^c)$ by considering two cases.

    \textit{Case A: The algorithm never reaches $\alpha_{k(s)}=1$.} with in a stage $s$.
    The final phase $K(s)$ is a normal phase. By Lemma~\ref{pblem:normal_phase_bound}, its regret against $x^c$ is $\le 2^{K(s)-1}$.
    Every earlier phase $k(s) < K(s)$ ended because the restart condition was triggered. By Lemma~\ref{pblem:extinction_phase_bound}, such phases have strictly negative regret against $x^c$: $\le -2^{k(s)-1}$.
    Summing these terms yields a telescoping geometric series:
    \[
        \mathbb{E}\!\left[\sum_{k(s) =1}^{K(s)} \sum_{t=\mathrm{start}_{k(s)}}^{\mathrm{end}_{k_\star(s)}} \langle \hat{\ell}_t,\; x_t-x^c\rangle\right]
        \le
        2^{K(s)-1} + \sum_{k(s)=1}^{K(s)-1} (-2^{k(s)-1})
        =
        2^{K(s)-1} - (2^{K(s)-1} - 1)
        =
        1.
    \]

    \textit{Case B: The algorithm reaches $\alpha_{k(s)}=1$ in phase $k_\star (s)$.}
    Phases $1,\ldots,k_\star(s)-1$ are exiting phases. By Lemma~\ref{pblem:extinction_phase_bound}, their total regret against $x^c$ is:
    \[
        \sum_{k(s)=1}^{k_\star(s)-1} (-2^{k(s)-1}) = -(2^{k_\star(s)-1}-1).
    \]
    The final phase $k_\star(s)$ uses $\alpha_{k(s)}=1$ (pure Banker-OMD). By Lemma~\ref{pblem:banker_regret_bound}, its regret is at most $\hat{R}_s$.
    The total regret is thus bounded by $\hat{R}_s - (2^{k_\star(s)-1} - 1)$.
    Since $\alpha_{k_\star(s)}=1$ implies $2^{k_\star(s)-1}/\hat{R}_s \ge 1 \iff 2^{k_\star(s)-1} \ge \hat{R}_s$, we have:
    \[
        \hat{R}_s - (2^{k_\star-1} - 1) \le \hat{R}_s - (\hat{R}_s - 1) = 1.
    \]
    In both cases, the total regret incurred in each stage $s$ against $x^c$ is $O(1)$.

    Then sum over all stages, 
    \begin{align*}
        \mathbb{E}\!\left[\sum_{s=1}^{S} \sum_{k(s) =1}^{K(s)} \sum_{t=\mathrm{start}_{k(s)}}^{\mathrm{end}_{k_\star(s)}} \langle \hat{\ell}_t,\; x_t-x^c\rangle\right]
        \le (\lceil \log_2 D \rceil + 1) = O(\log_2 D)
    \end{align*}
\end{proof}

\newpage
\section{Examples of Different Regularizers}
\label{app:regularizer}
\subsection{1/2-Tsallis}

Take \(x_0=\left(\frac{1}{A}, \ldots, \frac{1}{A}\right) .\)
For \(\Psi(x)=-2 \sum_i \sqrt{x_i}\), \(\nabla \Psi_i(x)=-\frac{1}{\sqrt{x_i}} .\)
Hence, \(\nabla \Psi\left(x_0\right)=(-\sqrt{A}, \ldots,-\sqrt{A}) \).
For any $y \in \Delta_{A}$,
\begin{align*}
D_{\Psi}\left(y, x_0\right) & =\Psi(y)-\Psi\left(x_0\right)-\left\langle\nabla \Psi\left(x_0\right), y-x_0\right\rangle \\
& =-2 \sum_{i=1}^A \sqrt{y_i}+2 \sqrt{A}-\left\langle-\sqrt{A} \mathbf{1}, y-x_0\right\rangle .
\end{align*}
Since $y, x_0 \in \Delta_{A}$, we have $\left\langle\mathbf{1}, y-x_0\right\rangle=0$, so
\[D_{\Psi}\left(y, x_0\right)=2\left(\sqrt{A}-\sum_{i=1}^A \sqrt{y_i}\right) .\]
Also $0 \leq y_i \leq 1$, so $\sqrt{y_i} \geq y_{i,}$ and therefore
\[\sum_{i=1}^A \sqrt{y_i} \geq \sum_{i=1}^A y_i=1 .\]
Thus
\[D_{\Psi}\left(y, x_0\right) \leq 2(\sqrt{A}-1). \]

Start from the same second-order expansion used in Lemma B.5 in \citet{huang2023banker}:
\[\sigma_t D_{\Psi}\left(x_t, \tilde{z}_t\right)=\frac{1}{2 \sigma_t}\left\|\hat{\ell}_t\right\|_{\nabla^2 \Psi^*\left(\theta_t\right)}^2,\]
for some $\theta_t$ on the segment joining $\nabla \Psi\left(x_t\right)-\hat{\ell}_t / \sigma_t$ and $\nabla \Psi\left(x_t\right)$.
Now compute the Hessians for 1/2-Tsallis:
\[\nabla^2 \Psi(x)=\operatorname{diag}\left(\frac{1}{2 x_i^{3 / 2}}\right) .\]
For the conjugate,
\[\nabla^2 \Psi^*(\nabla \Psi(x))=\left(\nabla^2 \Psi(x)\right)^{-1}=\operatorname{diag}\left(2 x_i^{3 / 2}\right) .\]
Also, coordinatewise,
\[\nabla \Psi_i\left(x_t\right)-\frac{\hat{\ell}_{t, i}}{\sigma_t} \leq \nabla \Psi_i\left(x_t\right),\]
because $\hat{\ell}_t \geq 0$ for the standard importance estimator with nonnegative losses. For this regularizer, each diagonal entry of $\nabla^2 \Psi^*(\cdot)$ is increasing on the relevant negative domain, so
\[\nabla^2 \Psi^*\left(\theta_t\right) \preceq \nabla^2 \Psi^*\left(\nabla \Psi\left(x_t\right)\right)=\operatorname{diag}\left(2 x_{t, i}^{3 / 2}\right) .\]
Therefore,
\begin{align*}
\sigma_t D_{\Psi}\left(x_t, \tilde{z}_t\right) & \leq \frac{1}{2 \sigma_t} \sum_{i=1}^A \hat{l}_{t, i}^2 \cdot 2 x_{t, i}^{3 / 2}  =\frac{1}{\sigma_t} \sum_{i=1}^A \hat{l}_{t, i}^2 x_{t, i}^{3 / 2}  = \frac{1}{\sigma_t}
\left(\frac{l_{t, A_t}}{x_{t, A_t}}\right)^2 x_{t, A_t}^{3 / 2} \leq  \frac{2}{\sigma_t \delta}  \sqrt{x_{t,A_t}} .
\end{align*}
Notice $l_{t,A_t} \in [0,1]$ and $x_{t,A_t} = \alpha \hat x_{t,A_t} + (1 - \alpha) x^c \geq \delta/2$, then $\frac{1}{x_{t, A_t}} \leq 2/ \delta$.
\begin{align*}
\mathbb{E}\left[\sigma_t D_{\Psi}\left(x_t, \tilde{z}_t\right) \mid \mathcal{F}_{t-1}\right] & \leq \frac{2}{\sigma_t \delta} \sum_{i=1}^A x_{t, i}   \sqrt{x_{t,i}}  =\frac{2}{\sigma_t \delta} \sum_{i=1}^A x_{t, i}^{3/2} \leq \frac{2}{\sigma_t \delta}
\end{align*}
\subsection{negative entropy}

Take
\(\Psi(x)=\sum_{i=1}^A x_i \log x_i\).
Then
\(\nabla \Psi_i(x)=1+\log x_i, \nabla^2 \Psi(x)=\operatorname{diag}\left(1 / x_i\right) .\)
 With uniform $x_0$,
\[D_{\Psi}\left(y, x_0\right)=\sum_{i=1}^A y_i \log \frac{y_i}{1 / A}=\sum_{i=1}^A y_i \log \left(A y_i\right) .\]
This is exactly the KL divergence $D_{\mathrm{KL}}\left(y \| x_0\right)$, and over the simplex its maximum is attained at a vertex, giving
\[D_{\Psi}\left(y, x_0\right) \leq \log A .\]
So we take \(C_1=\log A\).

Start from the same second-order expansion used in Lemma B.5 in \citet{huang2023banker}:
\[\sigma_t D_{\Psi}\left(x_t, \tilde{z}_t\right)=\frac{1}{2 \sigma_t}\left\|\hat{\ell}_t\right\|_{\nabla^2 \Psi^*\left(\theta_t\right)}^2,\]
for some $\theta_t$ on the segment joining $\nabla \Psi\left(x_t\right)-\hat{\ell}_t / \sigma_t$ and $\nabla \Psi\left(x_t\right)$.
For negative entropy, the conjugate is
\(\Psi^*(\theta)=\sum_{i=1}^A e^{\theta_i-1},\)
so
\[\nabla^2 \Psi^*(\theta)=\operatorname{diag}\left(e^{\theta_i-1}\right) .\]
Because the losses are nonnegative in the delayed adversarial MAB setup and the importance estimator has nonnegative coordinates, we have
\[\nabla \Psi\left(x_t\right)-\frac{\hat{\ell}_t}{\sigma_t} \leq \nabla \Psi\left(x_t\right) \quad \text { coordinatewise. }\]
Since $e^{\theta_i-1}$ is increasing in $\theta_i$,
\[\nabla^2 \Psi^*\left(\theta_t\right) \preceq \nabla^2 \Psi^*\left(\nabla \Psi\left(x_t\right)\right)=\nabla^2 \Psi\left(x_t\right)^{-1}=\operatorname{diag}\left(x_{t, i}\right) .\]
Hence
\begin{align*}
\sigma_t D_{\Psi}\left(x_t, \tilde{z}_t\right) 
\leq \frac{1}{2 \sigma_t}\left\|\hat{\ell}_t\right\|_{\operatorname{diag}\left(x_t\right)}^2
 =\frac{1}{2 \sigma_t} \sum_{i=1}^A x_{t, i} \hat{\ell}_{t, i}^2 
= \frac{1}{2 \sigma_t}
\left(\frac{l_{t, A_t}}{x_{t, A_t}}\right)^2 x_{t, A_t}
\leq \frac{1}{\sigma_t \delta}  
\end{align*}
Then 
\begin{align*}
\mathbb{E}\left[\sigma_t D_{\Psi}\left(x_t, \tilde{z}_t\right) \mid \mathcal{F}_{t-1}\right] 
\leq \frac{1}{ \sigma_t \delta} \sum_{i=1}^A  x_{t, i} 
=\frac{1}{\sigma_t \delta} 
\end{align*}
\newpage

\section{Lower Bound}
\label{app:lower-bound}

\begin{PBAppLemma}[Greedy bucket properties]{greedy-buckets}{
    The greedy buckets identify intervals during which the learner cannot receive
    any feedback generated inside the same interval. Thus each bucket behaves like
    one batch. Since the delays are non-increasing, later buckets cannot be longer
    than earlier ones. Moreover, the square of a bucket length controls the total
    delay hidden in the next bucket: each point in the next bucket has delay at most
    the previous bucket length, and the next bucket itself is no longer than the
    previous one.
}
Let \(d_1,\dots,d_T\in \mathbb Z_{\ge 1}\) be non-increasing and satisfy \(d_t \le T+1-t \quad \forall t\in[T].\) Let
\(B_m = \{b_m,\dots,b_{m+1}-1\}, m=1,\dots,M,\) be the greedy bucket decomposition, i.e.,
\(\forall t\in B_m: t+d_t > b_{m+1}-1,\)
and \(\exists t_m\in B_m: t_m+d_{t_m}=b_{m+1}.\)
Write
\(L_m := |B_m| = b_{m+1}-b_m.\)
Then:
\begin{enumerate}
    \item \(L_1 \ge L_2 \ge \cdots \ge L_M\).
    \item For every \(m\in[M-1]\),
    \[
    L_m^2 \ge \sum_{t\in B_{m+1}} d_t.
    \]
\end{enumerate}
\end{PBAppLemma}

\begin{proof}
We borrow the proof idea from \citet{masoudian2022best}.
For every \(m\), since \(b_m\in B_m\), we have
\[
b_m+d_{b_m} > b_{m+1}-1.
\]
Since \(b_m\), \(d_{b_m}\), and \(b_{m+1}\) are integers, it follows that
\[
b_m+d_{b_m} \ge b_{m+1},
\]
and therefore
\[
L_m = b_{m+1}-b_m \le d_{b_m}.
\]

On the other hand, by construction of the greedy buckets, there exists \(t_m\in B_m\)
such that
\[
t_m+d_{t_m}=b_{m+1}.
\]
Hence
\[
L_m = b_{m+1}-b_m = d_{t_m} + (t_m-b_m) \ge d_{t_m}.
\]
Since \(t_m \le b_{m+1}-1\) and the delays are non-increasing,
\[
d_{t_m} \ge d_{b_{m+1}-1}.
\]
Thus
\[
L_m \ge d_{b_{m+1}-1}.
\]

Now
\[
L_{m+1} = b_{m+2}-b_{m+1} \le d_{b_{m+1}}
\le d_{b_{m+1}-1}\le L_m,
\]
which proves
\[
L_1 \ge L_2 \ge \cdots \ge L_M.
\]

For the second claim, let \(t\in B_{m+1}\). Since \(t\ge b_{m+1}\) and the delays are non-increasing,
\[
d_t \le d_{b_{m+1}} \le d_{b_{m+1}-1} \le L_m.
\]
Therefore
\[
\sum_{t\in B_{m+1}} d_t
\le |B_{m+1}|\,L_m
= L_{m+1}L_m
\le L_m^2,
\]
where the last inequality uses \(L_{m+1}\le L_m\).
\end{proof}

\begin{PBAppLemma}[Safe lower bound for the batched bandit game]{batched-safe-lb}{
    This theorem isolates the statistical obstruction created by batching. We build
    two environments that differ only in whether arm \(2\) is slightly worse or
    slightly better than the baseline arms. If the learner plays arm \(2\) often,
    then it violates safety in the environment where arm \(2\) is worse. If it plays
    arm \(2\) rarely, then it suffers large regret in the environment where arm
    \(2\) is better. Because feedback inside each block is hidden until the block
    ends, the learner cannot reliably distinguish these two environments quickly
    enough. The quantity \(V=\sum_m L_m^2\) is the effective variance scale of this
    batched testing problem.
}
Let \(A\ge 2\), let \(n\ge 1\), and let
\(L_1 \ge L_2 \ge \cdots \ge L_n \ge 1\)
be positive integers. Consider the batched bandit game with block lengths
\(L_1,\dots,L_n\): in block \(m\), the learner chooses actions
\(J_{m,1},\dots,J_{m,L_m}\in[A]\)
sequentially, but without observing any loss from the current block; after the
block ends, the learner observes the chosen losses
\(g_{m,1,J_{m,1}},\dots,g_{m,L_m,J_{m,L_m}}.\)
Define \(V := \sum_{m=1}^n L_m^2.\)
Fix \(\delta\in(0,1/A]\) and \(\Lambda\ge 0\), and assume
\(\delta \ge \frac{L_1}{64V}.\)
Then there exists a comparator \(x^c\in\Delta_A\) with
\(x^c(i)\ge \delta \qquad \forall i\in[A]\)
such that for every batched bandit algorithm there exists an oblivious
deterministic block-loss sequence
\[
g_{m,s}\in[0,1]^A,
\qquad
m=1,\dots,n,\ s=1,\dots,L_m,
\]
for which
\[
R^{\mathrm{bat}}(x^c;g)\le \Lambda
\quad\Longrightarrow\quad
\max_{x\in\Delta_A} R^{\mathrm{bat}}(x;g)
\ge
\frac{3}{256}\,\delta^{-1/2}\sqrt{\frac{V}{L_1}}
-
\frac{\sqrt{\Lambda}}{32\sqrt{2}}\,
\delta^{-3/4}\left(\frac{V}{L_1}\right)^{1/4}
-\Lambda.
\]
Here
\[
R^{\mathrm{bat}}(x;g)
:=
\mathbb{E}\!\left[
\sum_{m=1}^n\sum_{s=1}^{L_m}g_{m,s,J_{m,s}}
-
\sum_{m=1}^n\sum_{s=1}^{L_m}\langle x,g_{m,s}\rangle
\right],
\]
where the expectation is over the learner's randomization.
\end{PBAppLemma}

\begin{proof}
Fix any batched bandit algorithm. Define
\[
x^c := (1-(A-1)\delta,\delta,\dots,\delta)\in\Delta_A.
\]
Since \(\delta\le 1/A\), every coordinate of \(x^c\) is at least \(\delta\).

We first construct two oblivious randomized environments, and then extract a
deterministic loss sequence from one of them.

\medskip
\noindent
\textbf{Step 1: randomized hard instance.}
Pick a parameter \(\gamma>0\), to be fixed later, and define
\[
\varepsilon_m := \gamma \frac{L_m}{\sqrt{V}},
\qquad
\Delta_m := \gamma \frac{L_m}{\sqrt{V}},
\qquad
w_m := \frac{L_m^2}{V},
\qquad m=1,\dots,n.
\]
Then \(w_m\ge 0\) and \(\sum_{m=1}^n w_m=1\).

For each block \(m\) and each slot \(s=1,\dots,L_m\), let
\[
Z_{m,s}^+ \sim \mathrm{Bernoulli}\!\left(\frac12+\varepsilon_m\right),
\qquad
Z_{m,s}^- \sim \mathrm{Bernoulli}\!\left(\frac12-\varepsilon_m\right),
\]
independently across all \((m,s)\).

Under \(\mathcal E_+\), define
\[
g_{m,s,1}^{(+)}=\frac12,
\qquad
g_{m,s,2}^{(+)}=Z_{m,s}^+,
\qquad
g_{m,s,i}^{(+)}=\frac12 \quad (i=3,\dots,A),
\]
and under \(\mathcal E_-\), define
\[
g_{m,s,1}^{(-)}=\frac12,
\qquad
g_{m,s,2}^{(-)}=Z_{m,s}^-,
\qquad
g_{m,s,i}^{(-)}=\frac12 \quad (i=3,\dots,A).
\]
Thus every loss lies in \([0,1]^A\).

Let \(N_m\) be the number of times the learner plays arm \(2\) in block \(m\):
\[
N_m := \sum_{s=1}^{L_m}\mathbf{1}\{J_{m,s}=2\}.
\]
Define the weighted number of arm 2 being played in block $m$ as:
\[
W := \sum_{m=1}^n \frac{L_m N_m}{V}.
\]
Since \(0\le N_m\le L_m\) and \(\sum_{m=1}^n L_m^2=V\), we have
\[
0\le W\le 1.
\]

\medskip
\noindent
\textbf{Step 2: safety regret under \(\mathcal E_+\).}
Under \(\mathcal E_+\), arm \(2\) has mean loss
\[
\mathbb E[g_{m,s,2}^{(+)}]=\frac12+\Delta_m,
\]
while all other arms have mean loss \(1/2\). Hence the learner's expected loss
in block \(m\) equals
\[
\frac{L_m}{2}+\Delta_m\,\mathbb E_+[N_m],
\]
and the comparator \(x^c\) incurs expected loss
\[
L_m \left((1-\delta)\frac12+\delta\left(\frac12+\Delta_m\right)\right)
=
\frac{L_m}{2}+\delta L_m\Delta_m.
\]

Therefore
\[
R_+^{\mathrm{bat}}(x^c)
=
\sum_{m=1}^n \Delta_m\bigl(\mathbb E_+[N_m]-\delta L_m\bigr).
\]
Since \(\Delta_m=\gamma L_m/\sqrt{V}\),
\[
R_+^{\mathrm{bat}}(x^c)
=
\gamma\sqrt{V}\left(\mathbb E_+[W]-\delta\right).
\]
Thus, if
\[
R_+^{\mathrm{bat}}(x^c)\le \Lambda,
\]
then
\begin{equation}
\label{eq:batched-good}
\mathbb E_+[W]\le \delta+\frac{\Lambda}{\gamma\sqrt{V}}.
\end{equation}

\medskip
\noindent
\textbf{Step 3: KL bound between \(\mathcal E_+\) and \(\mathcal E_-\).}
Let \(P_+\) and \(P_-\) denote the laws of the full batched transcript under
\(\mathcal E_+\) and \(\mathcal E_-\), respectively. Write \(H_m\) for the
transcript up to the end of block \(m-1\).

By the chain rule for KL divergence,
\[
\mathrm{KL}(P_+\|P_-)
=
\sum_{m=1}^n
\mathbb E_+\!\left[
\mathrm{KL}\bigl(
P_+(J_{m,1:L_m},Y_{m,1:L_m}\mid H_m)
\|
P_-(J_{m,1:L_m},Y_{m,1:L_m}\mid H_m)
\bigr)
\right],
\]
where
\[
J_{m,1:L_m}:=(J_{m,1},\dots,J_{m,L_m}),
\qquad
Y_{m,1:L_m}:=(g_{m,s,J_{m,s}})_{s=1}^{L_m}.
\]

Because the learner observes no loss from the current block before the block
ends, the conditional law of \(J_{m,1:L_m}\) given \(H_m\) is the same under
\(P_+\) and \(P_-\). Hence the block-\(m\) contribution reduces to the KL of
the observed losses conditioned on the chosen action sequence.

Fix \(H_m\) and a realized action sequence \(a_{1:L_m}\). For every slot \(s\)
with \(a_s\neq 2\), the observed loss is deterministically \(1/2\) under both
environments, so the KL contribution is \(0\). For every slot \(s\) with
\(a_s=2\), the observed loss is Bernoulli with parameter
\(1/2+\varepsilon_m\) under \(\mathcal E_+\) and with parameter
\(1/2-\varepsilon_m\) under \(\mathcal E_-\). Since losses are independent
across slots, the conditional KL for block \(m\) is
\[
N_m(a_{1:L_m})\,
\mathrm{kl}\!\left(\frac12+\varepsilon_m,\frac12-\varepsilon_m\right).
\]
Taking expectation gives
\[
\mathrm{KL}(P_+\|P_-)
=
\sum_{m=1}^n
\mathbb E_+[N_m]\,
\mathrm{kl}\!\left(\frac12+\varepsilon_m,\frac12-\varepsilon_m\right).
\]

For \(\varepsilon\in[0,1/4]\),
\[
\mathrm{kl}\!\left(\frac12+\varepsilon,\frac12-\varepsilon\right)
\le 16\varepsilon^2.
\]
We will verify below that \(\varepsilon_m\le 1/4\) for all \(m\). Using
\(\varepsilon_m=\gamma L_m/\sqrt{V}\), we obtain
\[
\mathrm{KL}(P_+\|P_-)
\le
16\gamma^2
\sum_{m=1}^n \mathbb E_+[N_m]\frac{L_m^2}{V}.
\]
Since \(L_m\le L_1\),
\[
\frac{L_m^2}{V}\,\mathbb E_+[N_m]
\le
L_1 \frac{L_m\,\mathbb E_+[N_m]}{V}.
\]
Therefore
\[
\mathrm{KL}(P_+\|P_-)
\le
16\gamma^2 L_1 \mathbb E_+[W].
\]

By Pinsker's inequality,
\[
\mathrm{TV}(P_+,P_-)
\le
\sqrt{\frac12\mathrm{KL}(P_+\|P_-)}
\le
2\sqrt{2}\,\gamma\sqrt{L_1\mathbb E_+[W]}
\le
4\gamma\sqrt{L_1\mathbb E_+[W]}.
\]
Since \(W\in[0,1]\),
\[
\mathbb E_-[W]-\mathbb E_+[W]
\le
\mathrm{TV}(P_+,P_-)
\le
4\gamma\sqrt{L_1\mathbb E_+[W]}.
\]
Hence
\begin{equation}
\label{eq:batched-bad}
\mathbb E_-[W]
\le
\mathbb E_+[W]+4\gamma\sqrt{L_1\mathbb E_+[W]}.
\end{equation}

\medskip
\noindent
\textbf{Step 4: regret against the best arm under \(\mathcal E_-\).}
Under \(\mathcal E_-\), arm \(2\) is the unique best arm in expectation, because
\[
\mathbb E[g_{m,s,2}^{(-)}]=\frac12-\Delta_m,
\]
while every other arm has mean loss \(1/2\). The learner's expected regret
against the pure strategy \(e_2\) is therefore
\[
R_-^{\mathrm{bat}}(e_2)
=
\sum_{m=1}^n \Delta_m\bigl(L_m-\mathbb E_-[N_m]\bigr)
=
\gamma\sqrt{V}\bigl(1-\mathbb E_-[W]\bigr).
\]
Combining this with \eqref{eq:batched-bad} and then with
\eqref{eq:batched-good}, we get: if
\(R_+^{\mathrm{bat}}(x^c)\le \Lambda\), then
\[
R_-^{\mathrm{bat}}(e_2)
\ge
\gamma\sqrt{V}
\left(
1-\delta-\frac{\Lambda}{\gamma\sqrt{V}}
-4\gamma\sqrt{L_1\left(\delta+\frac{\Lambda}{\gamma\sqrt{V}}\right)}
\right).
\]
Using \(\sqrt{x+y}\le \sqrt{x}+\sqrt{y}\),
\[
R_-^{\mathrm{bat}}(e_2)
\ge
\gamma\sqrt{V}\bigl(1-\delta-4\gamma\sqrt{L_1\delta}\bigr)
-\Lambda
-4\sqrt{\Lambda}\,\gamma^{3/2}V^{1/4}\sqrt{L_1}.
\]

Choose
\[
\gamma := \frac{1}{32\sqrt{L_1\delta}}.
\]
Then
\[
\varepsilon_m
=
\frac{L_m}{32\sqrt{L_1\delta}\sqrt{V}}
\le
\frac{\sqrt{L_1}}{32\sqrt{\delta V}}
\le \frac14,
\]
where the last inequality follows from
\[
\delta\ge \frac{L_1}{64V}.
\]
Thus the Bernoulli KL bound used above is valid.

Moreover,
\[
4\gamma\sqrt{L_1\delta}=\frac18,
\qquad
\delta\le \frac1A\le \frac12,
\]
and hence
\[
1-\delta-4\gamma\sqrt{L_1\delta}\ge \frac38.
\]
Therefore
\[
\gamma\sqrt{V}\bigl(1-\delta-4\gamma\sqrt{L_1\delta}\bigr)
\ge
\frac{3}{256}\,\delta^{-1/2}\sqrt{\frac{V}{L_1}}.
\]
Also,
\[
4\sqrt{\Lambda}\,\gamma^{3/2}V^{1/4}\sqrt{L_1}
=
\frac{\sqrt{\Lambda}}{32\sqrt{2}}\,
\delta^{-3/4}\left(\frac{V}{L_1}\right)^{1/4}.
\]
Hence, if \(R_+^{\mathrm{bat}}(x^c)\le \Lambda\), then
\[
R_-^{\mathrm{bat}}(e_2)
\ge
\frac{3}{256}\,\delta^{-1/2}\sqrt{\frac{V}{L_1}}
-
\frac{\sqrt{\Lambda}}{32\sqrt{2}}\,
\delta^{-3/4}\left(\frac{V}{L_1}\right)^{1/4}
-\Lambda.
\]
Since
\[
\max_{x\in\Delta_A}R_-^{\mathrm{bat}}(x)
\ge
R_-^{\mathrm{bat}}(e_2),
\]
we have shown that
\begin{equation}
\label{eq:batched-minus-conclusion}
\max_{x\in\Delta_A}R_-^{\mathrm{bat}}(x)
\ge
\frac{3}{256}\,\delta^{-1/2}\sqrt{\frac{V}{L_1}}
-
\frac{\sqrt{\Lambda}}{32\sqrt{2}}\,
\delta^{-3/4}\left(\frac{V}{L_1}\right)^{1/4}
-\Lambda
\end{equation}
whenever \(R_+^{\mathrm{bat}}(x^c)\le \Lambda\).

\medskip
\noindent
\textbf{Step 5: choose one randomized environment.}
There are two cases.

\emph{Case 1:} \(R_+^{\mathrm{bat}}(x^c)>\Lambda\). Choose
\(\mathcal E=\mathcal E_+\). Then
\[
R_{\mathcal E}^{\mathrm{bat}}(x^c)\le \Lambda
\quad\Longrightarrow\quad
\max_{x\in\Delta_A}R_{\mathcal E}^{\mathrm{bat}}(x)\ge B_\Lambda
\]
holds vacuously, where
\[
B_\Lambda :=
\frac{3}{256}\,\delta^{-1/2}\sqrt{\frac{V}{L_1}}
-
\frac{\sqrt{\Lambda}}{32\sqrt{2}}\,
\delta^{-3/4}\left(\frac{V}{L_1}\right)^{1/4}
-\Lambda.
\]

\emph{Case 2:} \(R_+^{\mathrm{bat}}(x^c)\le \Lambda\). Choose
\(\mathcal E=\mathcal E_-\). Then \eqref{eq:batched-minus-conclusion} gives
\[
\max_{x\in\Delta_A}R_{\mathcal E}^{\mathrm{bat}}(x)\ge B_\Lambda,
\]
so the implication
\[
R_{\mathcal E}^{\mathrm{bat}}(x^c)\le \Lambda
\quad\Longrightarrow\quad
\max_{x\in\Delta_A}R_{\mathcal E}^{\mathrm{bat}}(x)\ge B_\Lambda
\]
holds as well.

Thus, for every batched bandit algorithm there exists an oblivious randomized
environment \(\mathcal E\) for which the desired implication holds.

\medskip
\noindent
\textbf{Step 6: extract a deterministic sequence.}
If we are in Case 1, then
\[
R_{\mathcal E_+}^{\mathrm{bat}}(x^c)>\Lambda.
\]
Since this is the expectation of the regret over the random block-loss sequence
drawn from \(\mathcal E_+\), there exists a deterministic realization for which
\[
R^{\mathrm{bat}}(x^c;g)>\Lambda.
\]
For this realization, the implication
\[
R^{\mathrm{bat}}(x^c;g)\le \Lambda
\quad\Longrightarrow\quad
\max_{x\in\Delta_A}R^{\mathrm{bat}}(x;g)\ge B_\Lambda
\]
holds vacuously.

If we are in Case 2, then
\[
\max_{x\in\Delta_A}R_{\mathcal E_-}^{\mathrm{bat}}(x)\ge B_\Lambda.
\]
Writing \(R^{\mathrm{bat}}(x;g)\) for the regret on a deterministic
realization \(g\), we have
\[
R_{\mathcal E_-}^{\mathrm{bat}}(x)
=
\mathbb E_{g\sim\mathcal E_-}
\!\left[
R^{\mathrm{bat}}(x;g)
\right].
\]
Therefore
\[
\mathbb E_{g\sim\mathcal E_-}
\!\left[
\max_{x\in\Delta_A}R^{\mathrm{bat}}(x;g)
\right]
\ge
\max_{x\in\Delta_A}
\mathbb E_{g\sim\mathcal E_-}
\!\left[
R^{\mathrm{bat}}(x;g)
\right]
\ge
B_\Lambda.
\]
Hence some deterministic realization \(g\) satisfies
\[
\max_{x\in\Delta_A}R^{\mathrm{bat}}(x;g)\ge B_\Lambda.
\]
For this realization, the implication
\[
R^{\mathrm{bat}}(x^c;g)\le \Lambda
\quad\Longrightarrow\quad
\max_{x\in\Delta_A}R^{\mathrm{bat}}(x;g)\ge B_\Lambda
\]
holds.

Thus, in both cases, there exists an oblivious deterministic block-loss
sequence with the claimed property.
\end{proof}

\begin{PBAppLemma}[Uniform simulation of a delayed bandit algorithm by a batched bandit algorithm]{delay-to-batch}{
    This lemma formalizes the reduction from delayed feedback to batching. Once we
    restrict attention to a suffix of greedy buckets, no feedback generated inside a
    bucket can be used before that bucket ends. Therefore a batched algorithm can
    faithfully simulate any delayed algorithm: it pre-simulates the zero prefix,
    then processes each suffix bucket as one batch, revealing the chosen losses only
    after the bucket ends. Since the simulated delayed transcript is identical to
    the original delayed transcript, the regret against every comparator is exactly
    preserved.
}
Assume the setting of Lemma~\ref{pblem:greedy-buckets}, and fix \(j\in[M]\).
For every delayed bandit algorithm \(\mathcal A^{\mathrm{del}}\), there exists a
batched bandit algorithm \(\mathcal A^{\mathrm{bat}}\) with block lengths
\(L_j,\dots,L_M\), depending only on \(\mathcal A^{\mathrm{del}}\), the delay
sequence \(d_{1:T}\), and \(j\), such that the following holds.

For every family of vectors
\[
g_{m,s}\in[0,1]^A,
\qquad
m=j,\dots,M,\quad s=1,\dots,L_m,
\]
define a deterministic delayed-loss sequence \(\ell(g)=(\ell_1(g),\dots,\ell_T(g))\) by
\[
\ell_t(g) :=
\begin{cases}
\mathbf 0, & t < b_j,\\[0.5ex]
g_{m,s}, & t=b_m+s-1 \text{ for some } m\in\{j,\dots,M\},\ s\in[L_m].
\end{cases}
\]
Then for every comparator \(x\in\Delta_A\),
\[
R_T^{\mathrm{del}}(x;\ell(g))
=
R_{j:M}^{\mathrm{bat}}(x;g).
\]
where \[
R_T^{\mathrm{del}}(x;\ell(g))
=
\mathbb E\!\left[
\sum_{t=1}^T
\left(
[\ell_t(g)]_{A_t}
-
\langle x,\ell_t(g)\rangle
\right)
\right].
\]

Similarly,
\[
R_{j:M}^{\mathrm{bat}}(x;g)
:=
\mathbb E\!\left[
\sum_{m=j}^M\sum_{s=1}^{L_m}
\left(
[g_{m,s}]_{J_{m,s}}
-
\langle x,g_{m,s}\rangle
\right)
\right],
\]
where \(A_t\) is the action played by \(\mathcal A^{\mathrm{del}}\) at delayed
round \(t\) and \(J_{m,s}\) is the action played by
\(\mathcal A^{\mathrm{bat}}\) in slot \(s\) of block \(m\).
Here the left-hand side is the regret of \(\mathcal A^{\mathrm{del}}\) on the delayed
instance \(\ell(g)\), and the right-hand side is the regret of
\(\mathcal A^{\mathrm{bat}}\) on the batched instance \(g\).
\end{PBAppLemma}

\begin{proof}
Fix a delayed bandit algorithm \(\mathcal A^{\mathrm{del}}\). We construct a single
batched bandit algorithm \(\mathcal A^{\mathrm{bat}}\); the construction depends only on
\(\mathcal A^{\mathrm{del}}\), \(d_{1:T}\), and \(j\), and not on the eventual choice of \(g\).

Internally, \(\mathcal A^{\mathrm{bat}}\) maintains a faithful copy of
\(\mathcal A^{\mathrm{del}}\), including its random seed. We use the convention that
feedback whose reveal time is \(t\) is delivered immediately before the action at round \(t\).

\medskip
\noindent
\textbf{Step 1: pre-simulate the zero prefix.}
For every round \(t<b_j\), the delayed loss vector is \(\ell_t(g)=\mathbf 0\), so the entire
prefix is known in advance. The batched simulator therefore runs rounds
\(1,\dots,b_j-1\) of \(\mathcal A^{\mathrm{del}}\) internally before the batched game begins.

More concretely, for each simulated round \(t=1,\dots,b_j-1\), it first delivers to the
internal copy the entire collection of previously generated prefix feedback items whose
reveal time equals \(t\). Each such feedback value is \(0\). The internal copy then chooses an
action \(I_t\in[A]\). Since the loss vector at round \(t\) is \(\mathbf 0\), the chosen loss is
also \(0\). If its reveal time \(t+d_t\) lies in \(\{t+1,\dots,b_j-1\}\), then this zero loss
will be delivered later during the same pre-simulation. If \(t+d_t\in\{b_j,\dots,T\}\), the
simulator stores this feedback item, together with its reveal time \(t+d_t\), for delivery
during the batched simulation. (If \(t+d_t=T+1\), it is irrelevant for decisions up to horizon
\(T\) and need not be stored.)

After this pre-simulation, the internal copy is in exactly the same state that
\(\mathcal A^{\mathrm{del}}\) would be in immediately before round \(b_j\) on the delayed
instance \(\ell(g)\).

\medskip
\noindent
\textbf{Step 2: simulate the suffix online, one bucket at a time.}
Now process the suffix buckets \(B_j,\dots,B_M\). For a fixed bucket \(B_m\), the batched
algorithm goes through the slots corresponding to rounds
\[
t=b_m,\dots,b_{m+1}-1.
\]

Before the slot corresponding to round \(t\), the batched simulator delivers to the internal
copy the entire collection of stored feedback items whose reveal time equals \(t\). We claim
that every feedback item that should be revealed at time \(t\) is indeed already known to the
batched simulator.

Let \(u\) be a round whose feedback is revealed at time \(t\), so \(u+d_u=t\). Since
\(d_u\ge 1\), necessarily \(u<t\). There are now three possibilities.

\begin{itemize}
    \item If \(u<b_j\), then \(u\) is a prefix round, and its zero feedback was generated during
    Step 1 and stored there.
    \item If \(u\in B_r\) for some \(r<m\), then \(u\) lies in a strictly earlier suffix bucket.
    The chosen loss from round \(u\) was observed when block \(r\) ended, and was then stored
    together with its reveal time.
    \item \(u\) cannot lie in the current bucket \(B_m\), because for every \(u\in B_m\), the
    greedy-bucket property gives
    \[
    u+d_u > b_{m+1}-1 \ge t.
    \]
    Likewise, \(u\) cannot lie in a future bucket \(B_r\) with \(r>m\), because then
    \(u\ge b_{m+1}>t\), so \(u+d_u\ge u+1>t\).
\end{itemize}

Hence every feedback item due at time \(t\) comes either from the prefix or from a strictly
earlier suffix bucket, and is already available to the batched simulator.

After receiving all feedback items whose availability time equals \(t\), the
internal copy of \(\mathcal A^{\mathrm{del}}\) outputs an action
\(A_t\in[A]\). The batched algorithm plays the same action in the corresponding
slot
\[
    s=t-b_m+1
\]
of block \(m\); that is, it sets
\[
    J_{m,s}:=A_t.
\]

When block \(m\) ends, the batched environment reveals the chosen losses
\[
    [g_{m,1}]_{J_{m,1}},\dots,[g_{m,L_m}]_{J_{m,L_m}}.
\]
For each \(s\in[L_m]\), let
\[
    u=b_m+s-1.
\]
Since \(J_{m,s}=A_u\), the revealed value satisfies
\[
    [g_{m,s}]_{J_{m,s}}
    =
    [g_{m,s}]_{A_u}
    =
    [\ell_u(g)]_{A_u}.
\]
The simulator stores the delayed feedback item
\[
    \left(A_u,[g_{m,s}]_{J_{m,s}}\right)
\]
with availability time
\[
    \tau_u:=u+d_u+1,
\]
provided \(\tau_u\le T\). If \(\tau_u>T\), then this feedback item cannot affect
any action up to horizon \(T\), and hence may be ignored.

This storage is valid because the greedy-bucket property ensures that, for every
\(u\in B_m\),
\[
    u+d_u>b_{m+1}-1.
\]
Since all quantities are integers, this implies
\[
    u+d_u\ge b_{m+1},
\]
and therefore
\[
    \tau_u=u+d_u+1\ge b_{m+1}+1.
\]
Thus no feedback generated inside the current bucket \(B_m\) is available before
any action inside \(B_m\). Once block \(m\) ends, the batched simulator knows all
chosen losses from \(B_m\), and can deliver them to the internal delayed copy at
their correct future availability times.

By induction over rounds \(t=b_j,\dots,T\), immediately before each suffix round
\(t\), the internal copy of \(\mathcal A^{\mathrm{del}}\) has exactly the same
decision-relevant history as \(\mathcal A^{\mathrm{del}}\) would have on the
delayed instance \(\ell(g)\). In particular, for every suffix round
\(t=b_m+s-1\),
\[
    J_{m,s}=A_t.
\]

\medskip
\noindent
\textbf{Step 3: equality of regrets.}
For every suffix round \(t=b_m+s-1\), we have
\[
    \ell_t(g)=g_{m,s}
    \qquad\text{and}\qquad
    J_{m,s}=A_t.
\]
Therefore the learner incurs the same realized loss in the delayed and batched
games:
\[
    [\ell_t(g)]_{A_t}
    =
    [g_{m,s}]_{A_t}
    =
    [g_{m,s}]_{J_{m,s}}.
\]
The comparator incurs the same loss as well:
\[
    \langle x,\ell_t(g)\rangle
    =
    \langle x,g_{m,s}\rangle.
\]
For the prefix rounds \(t<b_j\), we have \(\ell_t(g)=\mathbf 0\), so
\[
    [\ell_t(g)]_{A_t}
    -
    \langle x,\ell_t(g)\rangle
    =
    0.
\]
Hence, pathwise,
\[
\sum_{t=1}^T
\left(
[\ell_t(g)]_{A_t}
-
\langle x,\ell_t(g)\rangle
\right)
=
\sum_{m=j}^M\sum_{s=1}^{L_m}
\left(
[g_{m,s}]_{J_{m,s}}
-
\langle x,g_{m,s}\rangle
\right).
\]
Taking expectation over the common randomization of the coupled delayed and
batched algorithms gives
\[
R_T^{\mathrm{del}}(x;\ell(g))
=
R_{j:M}^{\mathrm{bat}}(x;g).
\]
\end{proof}

\begin{PBAppTheorem}[Delayed safety lower bound]{delayed-safe-lb}{
    The delayed lower bound combines the previous two ingredients. First, the greedy
    buckets convert a suffix of the delayed game into a batched game with block lengths \(L_j,\ldots,L_M\). Second, the batched lower bound applies to this
    simulated game and gives a regret lower bound governed by \(V_j=\sum_{m=j}^M L_m^2\). Finally, the greedy bucket geometry shows that this
    batched complexity dominates the total delay hidden in the suffix: \(V_j\ge D_{\overline{S_j}}\). Hence large delayed feedback forces the same
    safety--optimality tradeoff as large batches.
}
    Let \(d_1,\dots,d_T\in \mathbb Z_{\ge 1}\) be non-increasing delays with \(d_t\le T+1-t \quad \forall t\in[T].\)
    Let \(D:=\sum_{t=1}^T d_t\) be the total delay. Let
    \(B_m=\{b_m,\dots,b_{m+1}-1\}, \quad m=1,\dots,M,\)
    be the greedy bucket decomposition, and define
    \(L_m:=|B_m|.\) Fix \(j\in[M]\) and write
    \[
    V_j := \sum_{m=j}^M L_m^2,
    \qquad
    S_j := \bigcup_{m=1}^j B_m,
    \qquad
    D_{\overline{S_j}} := \sum_{t\in[T]\setminus S_j} d_t.
    \]
    Fix \(\delta\in(0,1/A]\) and assume
    \(\delta \ge \frac{L_j}{64V_j}.\)
    Then there exists a comparator \(x^c\in\Delta_A\) with
    \(x^c(i)\ge \delta \quad \forall i\in[A]\)
    such that for every delayed adversarial bandit algorithm there exists an
    oblivious deterministic delayed-loss sequence
    \(\ell_1,\dots,\ell_T\in[0,1]^A\)
    for which
   {\footnotesize \[
    R_T^{\mathrm{del}}(x^c;\ell)\le \log D
    \quad\Rightarrow\quad
    \max_{x\in\Delta_A}R_T^{\mathrm{del}}(x;\ell)
    \ge
    \frac{3}{256}\,\delta^{-1/2}\sqrt{\frac{V_j}{L_j}}
    -
    \frac{\sqrt{\log D}}{32\sqrt{2}}\,
    \delta^{-3/4}\left(\frac{V_j}{L_j}\right)^{1/4}
    -\log D.
    \]}
    Moreover,
    \(V_j \ge D_{\overline{S_j}},\)
    and for the same loss sequence \(\ell\),
   {\footnotesize \[
    R_T^{\mathrm{del}}(x^c;\ell)\le \log D
    \quad\Rightarrow\quad
    \max_{x\in\Delta_A}R_T^{\mathrm{del}}(x;\ell)
    \ge
    \frac{3}{256}\,\delta^{-1/2}
    \sqrt{\frac{D_{\overline{S_j}}}{L_j}}
    -
    \frac{\sqrt{\log D}}{32\sqrt{2}}\,
    \delta^{-3/4}\left(\frac{V_j}{L_j}\right)^{1/4}
    -\log D.
    \]}
\end{PBAppTheorem}


\begin{proof}
By Lemma~\ref{pblem:greedy-buckets}, the suffix bucket lengths satisfy
\[
L_j \ge L_{j+1} \ge \cdots \ge L_M.
\]
Set
\[
n := M-j+1,
\qquad
\widetilde L_r := L_{j+r-1},
\qquad r=1,\dots,n.
\]
Then
\[
\widetilde L_1 \ge \widetilde L_2 \ge \cdots \ge \widetilde L_n \ge 1,
\qquad
\sum_{r=1}^n \widetilde L_r^2 = V_j.
\]

Apply Lemma~\ref{pblem:batched-safe-lb} to the block-length sequence
\(\widetilde L_1,\dots,\widetilde L_n\) with
\[
\Lambda=\log D.
\]
Since \(d_t\ge 1\) for all \(t\), we have \(D\ge 1\), and hence
\(\Lambda\ge 0\). Also,
\[
\delta \ge \frac{L_j}{64V_j}
=
\frac{\widetilde L_1}{64\sum_{r=1}^n \widetilde L_r^2},
\]
so the batched theorem applies. Hence there exists a comparator
\(x^c\in\Delta_A\) with
\[
x^c(i)\ge \delta \qquad \forall i\in[A]
\]
such that for every batched bandit algorithm on block lengths
\(\widetilde L_1,\dots,\widetilde L_n\), there exists an oblivious deterministic
block-loss sequence
\[
\widetilde g_{r,s}\in[0,1]^A,
\qquad
r=1,\dots,n,\quad s=1,\dots,\widetilde L_r,
\]
for which
\[
R^{\mathrm{bat}}(x^c;\widetilde g)\le \log D
\quad\Longrightarrow\quad
\max_{x\in\Delta_A}R^{\mathrm{bat}}(x;\widetilde g)
\ge
\frac{3}{256}\,\delta^{-1/2}\sqrt{\frac{V_j}{L_j}}
-
\frac{\sqrt{\log D}}{32\sqrt{2}}\,
\delta^{-3/4}\left(\frac{V_j}{L_j}\right)^{1/4}
-\log D.
\]

Now fix any delayed adversarial bandit algorithm \(\mathcal A^{\mathrm{del}}\).
By Lemma~\ref{pblem:delay-to-batch}, there exists a batched bandit algorithm
\(\mathcal A^{\mathrm{bat}}\), depending only on
\(\mathcal A^{\mathrm{del}}\), \(d_{1:T}\), and \(j\), such that for every suffix
block-loss sequence \(g\) and the associated delayed sequence \(\ell(g)\),
\[
R_T^{\mathrm{del}}(x;\ell(g))
=
R_{j:M}^{\mathrm{bat}}(x;g)
\qquad
\forall x\in\Delta_A.
\]

Apply the batched guarantee above to this specific batched algorithm
\(\mathcal A^{\mathrm{bat}}\). Then there exists a deterministic block-loss
sequence
\[
\widetilde g_{r,s}\in[0,1]^A,
\qquad
r=1,\dots,n,\quad s=1,\dots,\widetilde L_r,
\]
such that
\[
R^{\mathrm{bat}}(x^c;\widetilde g)\le \log D
\quad\Longrightarrow\quad
\max_{x\in\Delta_A}R^{\mathrm{bat}}(x;\widetilde g)
\ge
\frac{3}{256}\,\delta^{-1/2}\sqrt{\frac{V_j}{L_j}}
-
\frac{\sqrt{\log D}}{32\sqrt{2}}\,
\delta^{-3/4}\left(\frac{V_j}{L_j}\right)^{1/4}
-\log D.
\]

Re-index this block-loss sequence by setting
\[
g_{m,s}:=\widetilde g_{m-j+1,s},
\qquad
m=j,\dots,M,\quad s=1,\dots,L_m,
\]
and define the delayed-loss sequence \(\ell\) by
\[
\ell_t :=
\begin{cases}
\mathbf 0, & t<b_j,\\[0.5ex]
g_{m,s}, & t=b_m+s-1
\text{ for some } m\in\{j,\dots,M\},\ s\in[L_m].
\end{cases}
\]
Since this is only a relabeling of the blocks,
\[
R^{\mathrm{bat}}(x;\widetilde g)
=
R_{j:M}^{\mathrm{bat}}(x;g)
\qquad
\forall x\in\Delta_A.
\]
Therefore, by Lemma~\ref{pblem:delay-to-batch},
\[
R_T^{\mathrm{del}}(x;\ell)
=
R_{j:M}^{\mathrm{bat}}(x;g)
=
R^{\mathrm{bat}}(x;\widetilde g)
\qquad
\forall x\in\Delta_A.
\]
Thus
\[
R_T^{\mathrm{del}}(x^c;\ell)\le \log D
\quad\Longrightarrow\quad
\max_{x\in\Delta_A}R_T^{\mathrm{del}}(x;\ell)
\ge
\frac{3}{256}\,\delta^{-1/2}\sqrt{\frac{V_j}{L_j}}
-
\frac{\sqrt{\log D}}{32\sqrt{2}}\,
\delta^{-3/4}\left(\frac{V_j}{L_j}\right)^{1/4}
-\log D.
\]

It remains to relate \(V_j\) to \(D_{\overline{S_j}}\). Lemma~\ref{pblem:greedy-buckets}
gives
\[
L_m^2 \ge \sum_{t\in B_{m+1}} d_t
\qquad
(m=j,\dots,M-1).
\]
Summing over \(m=j,\dots,M-1\),
\[
\sum_{m=j}^{M-1}L_m^2
\ge
\sum_{m=j+1}^M \sum_{t\in B_m} d_t
=
D_{\overline{S_j}}.
\]
Since
\[
V_j = \sum_{m=j}^M L_m^2
\ge
\sum_{m=j}^{M-1}L_m^2,
\]
we obtain
\[
V_j \ge D_{\overline{S_j}}.
\]
Therefore
\[
\sqrt{\frac{V_j}{L_j}}
\ge
\sqrt{\frac{D_{\overline{S_j}}}{L_j}},
\]
and the second displayed lower bound follows for the same loss sequence
\(\ell\).
\end{proof}

\begin{PBAppCorollaryBox}{A delay sequence with linear dependence on total delay}{linear-total-delay}
    Fix integers \(q\ge 1\) and \(N\ge 1\), set
    \(T:=(N+1)q,\)
    and define the non-increasing delay sequence
    \(d_t := \min\{q,T+1-t\},
    \quad
    t=1,\dots,T.\)
    The greedy bucket decomposition is
    \(B_m=\{(m-1)q+1,\dots,mq\},
    \quad
    m=1,\dots,N+1,\)
    so that
    \(L_m=q
    \quad
    (m=1,\dots,N+1).\)
    Taking \(j=1\), if \(\delta\in(0,1/A]\) satisfies
    
    \(\delta
    \ge
    \frac{L_1}{64V_1}
    =
    \frac{1}{64\left(\frac{D}{q}+\frac{q-1}{2}\right)},
    \)
    then there exists a comparator \(x^c\in\Delta_A\) with
    \(
    x^c(i)\ge \delta \quad \forall i\in[A]
    \)
    such that for every delayed adversarial bandit algorithm there exists an
    oblivious deterministic delayed-loss sequence
    \[
    \ell_1,\dots,\ell_T\in[0,1]^A
    \]
    for which
    $
    R_T^{\mathrm{del}}(x^c;\ell)\le \log D
    $ implies 
    \[
    \max_{x\in\Delta_A}R_T^{\mathrm{del}}(x;\ell)
    \ge
    \frac{3}{256}\,\delta^{-1/2}\sqrt{D}
    -
    \frac{\sqrt{\log D}}{32\sqrt{2}}\,
    \delta^{-3/4}D^{1/4}
    -\log D \text{ for } q =1
    \]
\end{PBAppCorollaryBox}

\begin{proof}
The sequence \(d_t=\min\{q,T+1-t\}\) is non-increasing and satisfies
\[
d_t\le T+1-t
\qquad
\forall t\in[T].
\]
Moreover, \(d_t=q\) for \(t\le Nq+1\), and on the final \(q-1\) rounds the delays
decrease as \(q-1,q-2,\dots,1\). Hence
\[
D
=
Nq^2+\sum_{r=1}^q r
=
Nq^2+\frac{q(q+1)}{2}.
\]

For \(m=1,\dots,N+1\), define
\[
b_m:=(m-1)q+1.
\]
Then \(B_m=\{b_m,\dots,b_{m+1}-1\}\) has length \(q\). If \(m\le N\), then
\(d_t=q\) throughout \(B_m\), and therefore for every \(t\in B_m\),
\[
t+d_t=t+q>b_{m+1}-1,
\]
with equality \(b_m+d_{b_m}=b_{m+1}\). For the last bucket \(B_{N+1}\), we have
\[
t+d_t=T+1=b_{N+2}
\qquad
\forall t\in B_{N+1}.
\]
Thus these are exactly the greedy buckets.

The identities for \(L_1\), \(V_1\), and \(D_{\overline{S_1}}\) follow directly:
\[
L_1=q,
\qquad
V_1=\sum_{m=1}^{N+1}q^2=(N+1)q^2,
\]
and since \(S_1=B_1\) and \(\sum_{t\in B_1}d_t=q^2\),
\[
D_{\overline{S_1}}=D-q^2.
\]
Dividing by \(L_1=q\) gives
\[
\frac{V_1}{L_1}
=
(N+1)q
=
\frac{D}{q}+\frac{q-1}{2},
\]
and
\[
\frac{D_{\overline{S_1}}}{L_1}
=
\frac{D-q^2}{q}
=
\frac{D}{q}-q.
\]
The regret lower bound is Theorem~\ref{pbthm:delayed-safe-lb} applied with
\(j=1\).
\end{proof}
\newpage
\section{Experimental details}
\label{app:experiments}

This appendix describes the experimental setup used for the dynamic comparison
in Figure~\ref{fig:dynamic-smoothed-best-arm}, which reports fixed one-step, geometric, and Pareto delayed-feedback conditions.. The purpose of the experiment is
to test safety-aware bandit algorithms under delayed feedback. In particular,
the experiment stresses the following failure mode: if feedback is delayed, an
algorithm may incorrectly infer that the safe comparator is suboptimal and may
switch to an aggressive exploration mode before enough evidence has actually
arrived. The proposed algorithm, \PB{}, is designed to avoid this mistimed
transition by using a delay-calibrated restart threshold.

The displayed run uses horizon
\[
    T = 50{,}000
\]
and number of arms
\[
    A = 100.
\]
At each round \(t\), the learner selects an arm \(a_t \in [A]\) and
incurs loss \(\ell_{t,a_t}\), where
\[
    \ell_t = (\ell_{t,1},\ldots,\ell_{t,A}) \in [0,1]^A.
\]
Smaller loss is better. Although the full loss table is generated before the
experiment, each algorithm only observes the loss of the arm it selected, and
that observation may arrive after a delay.

\subsection{Synthetic non-stationary loss process}

The displayed experiment uses the flexible non-stationary loss generator denoted
\texttt{env\_any} in the implementation. Time is divided into
\[
    B = 500
\]
blocks. For each arm \(i \in [A]\) and each block \(b \in [B]\), the environment
draws an arm-block-specific mean and standard deviation,
\[
    x_{i,b} \sim \mathrm{Unif}(0,1),
    \qquad
    \sigma_{i,b} \sim \mathrm{Unif}(0.1,0.2).
\]
The implementation assigns round \(t\) to block
\[
    b(t)
    = 1 +
    \min\left\{
        \left\lfloor
        \frac{t-1}{\lfloor T/B \rfloor + 1}
        \right\rfloor,
        B-1
    \right\}.
\]
Conditional on these parameters, the loss of arm \(i\) at time \(t\) is sampled
from a normal distribution truncated to the unit interval:
\[
    \ell_{t,i}
    \sim
    \mathrm{TN}_{[0,1]}
    \left(
        x_{i,b(t)}, \sigma_{i,b(t)}^2
    \right).
\]
Thus the environment is non-stationary: an arm that is good in one block may be
bad in another block. This makes the setting more demanding than a stationary
stochastic bandit problem, because algorithms must learn while the quality of
arms changes over time.

All algorithms in a given delay column are run on the same realized loss table
and the same realized delay sequence. Therefore, differences between curves are
due to the algorithms rather than to different random environments.

\subsection{Safe comparator}
To isolate safety behavior conditional on a strong safe reference, we choose the reference arm in hindsight. Let $i^{\star} \in \arg \min _{i \in[A]} \sum_{t=1}^T \ell_{t, i}$ be the best fixed arm for the realized loss table. \PB{} receives the full-support comparator $x^c$ defined by $x^c\left(i^{\star}\right)=1-(A-1) \delta$ and $x^c(i)=\delta$ for $i \neq i^{\star}$. In the displayed run, $A=100$ and $\delta=0.001$, so $x^c\left(i^{\star}\right)=0.901$. The conservative baselines require a deterministic safe default arm and a known default reward; we therefore give them the same anchor $\operatorname{arm} i^{\star}$ and $r_0=\frac{1}{T} \sum_{t=1}^T\left(1-\ell_{t, i^*}\right)$. This oracle choice is favorable to methods that frequently play the default arm and should be interpreted as a diagnostic setting, not as an online procedure for constructing a safe comparator.

\subsection{Delay models}

Feedback from round \(s\) arrives at the end of round \(s+d_s\), where \(d_s\)
is the realized delay. Under this end-of-round convention, feedback from round $s$ with $d_s=0$ is observed after the action and loss at round $s$, and is first usable for the decision at round $s+1$. More generally, feedback arriving at the end of round $s+d_s$ is first usable from round $s+d_s+1$. The realized
total delay is
\[
    D = \sum_{t=1}^T d_t.
\]
The experiment compares three delay settings, summarized in
Table~\ref{tab:delay-settings}.

\begin{table}[t]
  \caption{Delay settings in the dynamic experiment. Each delayed column uses
  the same realized delay sequence for all algorithms; the geometric setting
  has the largest realized total delay in the plotted run.}
  \label{tab:delay-settings}
  \centering
  \begin{tabularx}{\linewidth}{@{}lYc@{}}
    \toprule
    Column & Delay model & Realized \(D\) \\
    \midrule
    Fixed one-step
    &
    \(\Pr(d_t=1)=0.03\) and \(\Pr(d_t=0)=0.97\).
    &
    \(1545\)
    \\
    Geometric
    &
    With probability \(0.03\), feedback is delayed and
    \(d_t \sim \mathrm{Geom}(0.4)\) on \(\{1,2,\ldots\}\); otherwise
    \(d_t=0\).
    &
    \(3927\)
    \\
    Pareto
    &
    With probability \(0.03\), feedback is delayed and
    \(d_t=1+\lfloor Z_t\rfloor\), where \(Z_t\) is generated by NumPy's Pareto sampler, $Z_t \sim \operatorname{Lomax}(\alpha=$ 2.5 , scale $=1$ ), i.e. $Z_t \geq 0$; otherwise \(d_t=0\).
    &
    \(2065\)
    \\
    \bottomrule
  \end{tabularx}
\end{table}

The Fixed one-step setting has only short delays: each round is delayed by one step with probability 0.03, and otherwise has no delay. The geometric setting has the
largest realized total delay in the displayed run. The Pareto setting has a
smaller realized total delay than the geometric setting, but it is heavy-tailed:
rare longer delays can create local bursts of missing feedback. This is
important because the effect of delays is not determined only by the scalar
quantity \(D\); the temporal placement of delayed observations can also affect
when an algorithm becomes aggressive or conservative.

\begin{figure}[t]
  \centering
  \includegraphics[width=\linewidth]{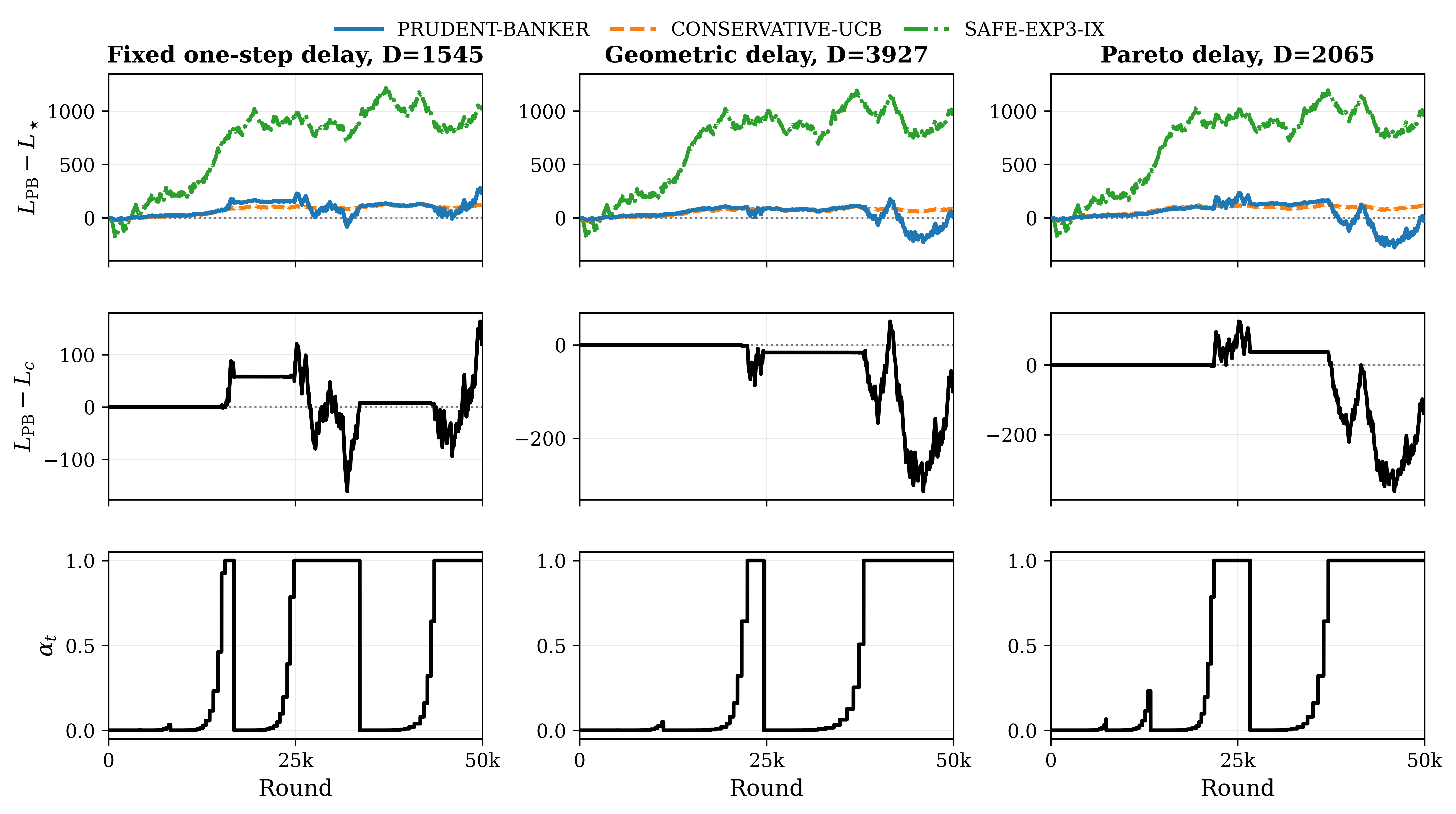}
  \caption{Dynamic performance against the best-fixed-arm comparator.
  The top row plots cumulative loss minus the cumulative loss of the best fixed
  arm in hindsight, the middle row plots \PB{} loss minus comparator loss,
  and the bottom row plots the \PB{} mixing coefficient \(\alpha_t\). The
  key takeaway is that delayed feedback changes the timing of safety decisions:
  under delay, \PB{} repeatedly restarts conservatively and only becomes
  aggressive after revalidating the comparator-suboptimality evidence.}
  \label{fig:dynamic-smoothed-best-arm}
\end{figure}

\subsection{Evaluation metrics}

Let \(p_t^{B} \in \Delta_A\) denote the action distribution used by algorithm
\(B\) at time \(t\). The plotted curves are pseudo-loss curves:
\[
    L_B(t)
    =
    \sum_{s=1}^t
    \langle p_s^B, \ell_s \rangle.
\]
Thus the plotted loss is the expected loss under the algorithm's action
distribution, not merely the realized loss of the sampled arm. This reduces
sampling noise in the visualization while preserving the bandit-feedback
learning dynamics used by the algorithms.

The cumulative loss of the best realized fixed arm is
\[
    L_\star(t)
    =
    \sum_{s=1}^t
    \ell_{s,i^\star}.
\]
The top row of Figure~\ref{fig:dynamic-smoothed-best-arm} plots
\[
    R_B^\star(t)
    =
    L_B(t)-L_\star(t).
\]
This is the cumulative loss of algorithm \(B\) minus the cumulative loss of the
best fixed arm in hindsight. Lower is better. Since \(i^\star\) is selected
using the entire horizon, a curve can be slightly negative at early times if an
algorithm temporarily outperforms the arm that is best over the full horizon.

The middle row plots the \PB{} advantage over comparator:
\[
    L_{\PB{}}(t) - L_c(t),
\]
where
\[
    L_c(t)
    =
    \sum_{s=1}^t
    \langle x_c,\ell_s\rangle.
\]
Positive values mean that \PB{} has more cumulative loss than the safe
comparator. Negative values mean that \PB{} has accumulated less loss than
the comparator.

The bottom row plots the mixing coefficient \(\alpha_t\) used by \PB{}.
Its played distribution is
\[
    x_t
    =
    (1-\alpha_t)x^c + \alpha_t \hat x_t,
\]
where \(x^c\) is the safe comparator and \(\hat x_t\) is the \Banker{}
distribution. Thus \(\alpha_t\) measures aggressiveness:
\[
    \alpha_t \approx 0
    \quad \Longrightarrow \quad
    \text{\PB{} is close to the safe comparator,}
\]
whereas
\[
    \alpha_t = 1
    \quad \Longrightarrow \quad
    \text{\PB{} fully trusts the \Banker{} learner.}
\]
Thus, in the middle row, downward movement indicates improvement relative to the comparator.
\subsection{Algorithms}
\label{app:algorithms}

\subsubsection{\PB{}}

The proposed method combines a delayed-feedback \Banker{} learner with a
safety-aware phased-aggression mechanism. The implementation labels this method
as \PB{} in the figure.

The \Banker{} component uses the negative-entropy regularizer
\[
    \psi(x)
    =
    \sum_{i=1}^A x_i \log x_i.
\]
Therefore, 
\[
C_1=\log A=\log 100,
\qquad
C_2=\frac{1}{\delta}=1000,
\]
and
\[
\sigma_t
=
\sqrt{\frac{C_2}{C_1}}
\left(
\frac{1}{\sqrt{t-\mathrm{start}_{k(s)}+1}}
+
\mathfrak d_t^{k(s)}
\sqrt{
\frac{
\ln\!\left(\mathfrak D_t^{k(s)}+1\right)
}{
\mathfrak D_t^{k(s)}
}
}
\right)^{-1}.
\]
\subsubsection{\CUCB{}}
\CUCB{} is based on Conservative UCB of
\citet{wu2016conservative}, who introduced the conservative bandit framework and
proposed a UCB-style algorithm that maintains a reward constraint relative to a
default arm uniformly over time.
In this experiment, the safe default is the hindsight best
fixed arm \(i^\star\). The baseline is also given the default arm's average
reward
\[
    r_0
    =
    \frac{1}{T}
    \sum_{t=1}^T
    (1-\ell_{t,i^\star}).
\]
This makes \CUCB{} particularly strong in the top-row metric because repeatedly
playing the default arm is already close to optimal with respect to the
best-fixed-arm comparator.

The implementation uses delayed observations. That is, if feedback from round
\(s\) arrives at time \(s+d_s\), then the reward from round \(s\) is not used to
update the empirical mean or confidence interval before that arrival time.
However, \CUCB{} is not delay-calibrated in the same sense as \PB{}: it does
not maintain a delay-budget estimate, it does not use a delay-corrected
comparator-suboptimality threshold, and it does not restart phases when the
realized delay budget exceeds an estimate.

Let $O_{<t}=\left\{s<t: s+d_s<t\right\}$ be the set of feedback items available before choosing at round $t$ . Let $N_i^{\text {obs }}(t)$ be the number of rewards from $\operatorname{arm} i$ in $O_{<t}$, and let $S_i(t)$ be their sum. For an arm with \(N_i^{\mathrm{obs}}(t)>0\), define
\[
    \widehat r_i(t)
    =
    \frac{S_i(t)}{N_i^{\mathrm{obs}}(t)}
\]
and a confidence radius
\[
    c_i(t)
    =
    \sqrt{
        \frac{
            2\log\left(
                \max\left\{
                    3,
                    \frac{2A(t+1)^2}{\delta_{\mathrm{ucb}}}
                \right\}
            \right)
        }{
            N_i^{\mathrm{obs}}(t)
        }
    }.
\]
where $\delta_{\mathrm{UCB}}=\frac{1}{\max \{T, 2\}} = 2 \times 10^{-5}$ in our experiment.

The lower and upper confidence bounds are
\[
    \mathrm{LCB}_i(t)
    =
    \max\{0,\widehat r_i(t)-c_i(t)\},
    \qquad
    \mathrm{UCB}_i(t)
    =
    \min\{1,\widehat r_i(t)+c_i(t)\}.
\]
For an unobserved arm, the implementation uses
\[
    \mathrm{LCB}_i(t)=0,
    \qquad
    \mathrm{UCB}_i(t)=1.
\]
For the default arm, the confidence interval is collapsed to the known value:
\[
    \mathrm{LCB}_{i^\star}(t)
    =
    \mathrm{UCB}_{i^\star}(t)
    =
    r_0.
\]

At time \(t\), the algorithm first selects the optimistic candidate
\[
    j_t
    \in
    \arg\max_{i\in[A]}
    \mathrm{UCB}_i(t).
\]
It then checks whether playing \(j_t\) would keep the pessimistic reward budget
above the conservative safety requirement. Let \(N_i^{\mathrm{play}}(t)\) be
the number of times arm \(i\) has been played before round \(t\). The
implementation plays \(j_t\) only if
\[
    \sum_{i=1}^A
    N_i^{\mathrm{play}}(t)
    \mathrm{LCB}_i(t)
    +
    \mathrm{LCB}_{j_t}(t)
    \ge
    (1-\alpha_{\mathrm{safe}})(t+1)r_0.
\]
Otherwise, it plays the safe default arm \(i^\star\). We pick \(\alpha_{\mathrm{safe}} = 0.1\). 

\subsubsection{\SafeEXP{}}
\SafeEXP{} combines the adversarial
safe-playing strategy of \citet{wu2016conservative} with EXP3-IX as the base
adversarial bandit learner. EXP3-IX is the implicit-exploration variant of EXP3
introduced by \citet{kocak2014efficient}; we use the standard implicit
exploration estimator also analyzed in high-probability form by
\citet{neu2015explore}. In our experiments both baselines are implemented with
delayed observations by applying their updates only when feedback arrives.
The implementation uses
\[
    \eta
    =
    \min
    \left\{
        \frac{1}{2},
        \sqrt{
            \frac{\log A}{AT}
        }
    \right\},
    \qquad
    \gamma
    =
    \frac{\eta}{2}.
\]
Let \(w_{t,i}\) be the EXP3 weight of arm \(i\). When the safety wrapper allows
the base learner to act, the base distribution is
\[
    q_t(i)
    =
    \frac{w_{t,i}}{\sum_{j=1}^A w_{t,j}}.
\]
When feedback from a base-learner round \(s\) arrives, the algorithm forms the
implicit-exploration estimator
\[
    \widehat \ell_{s,a_s}
    =
    \frac{
        \ell_{s,a_s}
    }{
        q_s(a_s)+\gamma
    },
\]
and updates the corresponding log-weight by
\[
    \log w_{s+1,a_s}
    =
    \log w_{s,a_s}
    -
    \eta \widehat \ell_{s,a_s}.
\]
All other coordinates are unchanged by that feedback item.

The safety wrapper maintains a reward budget relative to the known default arm.
At round \(t\), the required budget is
\[
    (1-\alpha_{\mathrm{safe}})r_0(t+1).
\]
If the current budget is large enough, the algorithm allows EXP3-IX to act. If
not, the algorithm plays the safe default arm. As with \CUCB{}, rewards from
non-default arms are credited only when the delayed feedback arrives.

Thus \SafeEXP{} is a delayed-feedback baseline: it updates the adversarial
bandit learner when feedback arrives. However, unlike \PB{}, it does not use a
delay-calibrated comparator-suboptimality test. Consequently, it can continue
to explore broadly even when the comparator is already very strong.

\subsection{Interpretation of Figure~\ref{fig:dynamic-smoothed-best-arm}}
\label{app:dynamic-interpretation}

\subsubsection{Overall patterns}

The top row shows cumulative loss relative to the best fixed arm. Since the
safe comparator puts most weight on hindsight best arm, a conservative
method that often plays the default arm can perform very well in this metric.
This explains why \CUCB{} remains close to the zero line in most columns.

\SafeEXP{} accumulates substantially larger loss relative to the best fixed arm.
This is expected in this diagnostic setting: EXP3-style algorithms maintain
broad exploration for adversarial robustness, and this exploration is costly
when the comparator is already highly concentrated on the best fixed arm.

\PB{} lies between these two extremes. Its behavior is best understood by
looking at the bottom row. When \(\alpha_t\) is small, the algorithm is close to
the safe comparator. When \(\alpha_t\) approaches \(1\), the algorithm becomes
fully aggressive and behaves like the adaptive \Banker{} learner. Therefore, changes
in the top and middle rows are closely tied to the trajectory of \(\alpha_t\).

\subsubsection{Fixed one-step delays}

In the Fixed one-step column, only 3\% of feedback events are delayed, and each active delay has length one. The realized total delay is
\[
    D = 1545.
\]
The bottom row shows repeated cycles in \(\alpha_t\). The coefficient rises as
the algorithm collects evidence, but it drops when the realized delay budget
exceeds the current stage estimate and a restart is triggered.

The top row shows that \PB{} incurs substantially less loss than
\SafeEXP{}, but more loss than \CUCB{}. The middle row is mostly positive, so
\PB{} pays some extra cumulative loss relative to the smoothed best-arm
comparator. However, the key qualitative point is that \PB{} does not
remain permanently aggressive after its first transition. Delayed feedback
causes the algorithm to restart conservatively and revalidate the evidence.

\subsubsection{Geometric delays}

The geometric-delay column has the largest realized total delay:
\[
    D = 3927.
\]
Although only \(3\%\) of rounds activate a delay, the active delay length is
geometric, so the backlog can be substantially larger than in the fixed-delay
setting.

This is the clearest panel for the delay-calibrated behavior of \PB{}. The
bottom row shows longer conservative periods and fewer aggressive intervals. As
the realized delay budget grows, the threshold becomes more conservative, and
the algorithm requires stronger evidence before increasing \(\alpha_t\).

The middle row is more favorable in this column than in the fixed-delay and
Pareto-delay columns. For a substantial interval, it is negative, meaning that
\PB{} temporarily has lower cumulative loss than the smoothed comparator.
Later, after a sharp transition, the gap returns closer to zero. This behavior
is consistent with the intended role of the delay correction: the algorithm
does not overreact to incomplete feedback, but it can still exploit evidence
when the evidence is strong enough after accounting for delay.

\subsubsection{Pareto delays}

The Pareto-delay column has realized total delay
\[
    D = 2065.
\]
This is smaller than the geometric-delay total in the displayed run, but the
delay distribution is heavy-tailed. A few longer delays can hide feedback from
important periods, which can make the comparator-suboptimality test less stable.

The bottom row again shows several restarts of \(\alpha_t\). These restarts
occur because the algorithm discovers that its current delay estimate is too
small and therefore returns to a conservative phase. The top row shows that
\PB{} improves over \SafeEXP{} but remains worse than \CUCB{} relative to
the hindsight best arm. The middle row is mostly negative, indicating
less loss relative to the smoothed comparator.

\subsubsection{Main takeaways}

The experiment supports three qualitative conclusions.

First, The fixed one-step, geometric, and Pareto panels show that the delay-aware restart threshold affects the trajectory of \(\alpha_t\), not only the final cumulative loss

Second, the delay calibration is visible in the trajectory, not only in final
regret. The bottom row shows the operational role of the delay-aware restart
threshold: \PB{} does not treat incomplete feedback as definitive evidence.
Instead, it restarts when the realized delay budget exceeds its estimate and
requires new evidence before becoming aggressive again.

Third, the baselines represent different behaviors. \CUCB{} is extremely
conservative and performs well in the top-row metric because the default arm is
the hindsight best fixed arm in this diagnostic. \SafeEXP{} is adversarially
motivated and delay-updated, but it explores much more broadly and therefore
incurs larger loss relative to the best fixed arm. \PB{} attempts to balance
these objectives: it stays close to the comparator when delay uncertainty is
large, but it can increase aggression when the observed evidence is strong
enough after delay calibration.

\end{document}